\documentclass[sigconf]{acmart}
\setkeys{acmart.cls}{balance=false}
\AtBeginDocument{%
  }
\setcopyright{acmlicensed}

\copyrightyear{2026}
\acmYear{2026}
\setcopyright{cc}
\setcctype{by-nc-nd}
\acmConference[CCS '26]{Proceedings of the 2026 ACM SIGSAC Conference on Computer and Communications Security}{November 15--19, 2026}{The Hague, Netherlands}
\acmBooktitle{Proceedings of the 2026 ACM SIGSAC Conference on Computer and Communications Security (CCS '26), November 15--19, 2026, The Hague, Netherlands}
\acmDOI{10.1145/3830454.3832651}
\acmISBN{979-8-4007-2871-6/2026/11}
\usepackage{amsmath,amsfonts}
\usepackage{algorithm}
\usepackage{algorithmic}
\usepackage{array}
\usepackage{textcomp}
\usepackage{stfloats}
\usepackage{url}
\usepackage{balance} 
\usepackage{verbatim}
\usepackage{graphicx}
\usepackage{url}
\usepackage{xr}
\usepackage{hyperref}
\usepackage{pifont}
\usepackage{amsthm}
\usepackage{subfigure}
\usepackage{graphicx}
\usepackage{xcolor}
\usepackage{colortbl}
\usepackage{caption}
\usepackage{multirow}
\usepackage{enumitem}
\usepackage[normalem]{ulem}
\usepackage{booktabs}
\useunder{\uline}{\ul}{}
\definecolor{celestialblue}{rgb}{0.29, 0.59, 0.82}
\definecolor{cerulean}{rgb}{0.0, 0.48, 0.65}
\definecolor{cadmiumorange}{rgb}{0.93, 0.53, 0.18}
\DeclareMathOperator*{\argmax}{arg\,max}

\newtheorem{theorem}{Theorem}

\usepackage{tikz}
\newcommand*\circled[1]{\tikz[baseline=(char.base)]{
            \node[shape=circle,draw,inner sep=0.8pt] (char) {#1};}}

\usepackage{enumitem}

\begin{document}



\title{AuditVotes: Elevating Provable Defense for GNNs with Efficient Augmentation and  Conditional Smoothing}



\author{Yuni Lai}
\authornotemark[1]
\affiliation{%
  \institution{Xidian University}
  \city{Xi'an}
  \country{China}
}

\affiliation{%
  \institution{The Hong Kong Polytechnic University}
  \city{Hong Kong}
  \country{China}}
\email{laiyuni@xidian.edu.cn}

\author{Yulin Zhu}
\affiliation{%
  \institution{Hong Kong Chu Hai College}
  \city{Hong Kong}
  \country{China}}
\email{ylzhu@chuhai.edu.hk}
\authornote{Equal contribution}

\author{Yixuan Sun}
\affiliation{%
  \institution{The Hong Kong Polytechnic University}
  \city{Hong Kong}
  \country{China}}
\email{selina.sun@connect.polyu.hk}

\author{Yulun Wu}
\affiliation{%
  \institution{National University of Defense Technology}
  \city{Changsha}
  \country{China}}
\email{wuyulun14@nudt.edu.cn}

\author{Bin Xiao}
\affiliation{%
  \institution{The Hong Kong Polytechnic University}
  \city{Hong Kong}
  \country{China}}
\email{b.xiao@polyu.edu.hk}

\author{Gaolei Li}
\affiliation{%
  \institution{Shanghai JiaoTong University}
  \city{Shanghai}
  \country{China}}
\email{gaolei_li@sjtu.edu.cn}

\author{Jianhua Li}
\affiliation{%
  \institution{Shanghai JiaoTong University}
  \city{Shanghai}
  \country{China}}
\email{lijh888@sjtu.edu.cn}

\author{Qi Xie}
\affiliation{%
  \institution{Hubei University}
  \city{Hubei}
  \country{China}}
\email{d20230182@hubu.edu.cn}

\author{Kai Zhou}
\affiliation{%
  \institution{The Hong Kong Polytechnic University}
  \city{Hong Kong}
  \country{China}}
\email{kaizhou@polyu.edu.hk}
\authornote{Corresponding author}

\renewcommand{\shortauthors}{Yuni Lai et al.}

\begin{abstract}

Despite advancements in Graph Neural Networks (GNNs), adaptive attacks continue to challenge their robustness. Certified robustness via randomized smoothing offers provable guarantees but suffers from a severe accuracy–robustness trade-off, limiting its practical use. To bridge this gap, we introduce AuditVotes, the first framework that simultaneously achieves high clean accuracy and strong certified robustness. AuditVotes seamlessly integrates two novel components into the randomized smoothing pipeline: (1) graph rewiring augmentation, which denoises randomized graphs to recover data quality, and (2) conditional smoothing, which filters low-confidence votes to ensure prediction consistency. We establish a novel theoretical result, proving that certified robustness is preserved under arbitrary filtering functions. Designed for inductive learning, our framework generalizes to unseen nodes and applies broadly to other smoothing schemes, including de-randomized smoothing for graphs and Gaussian smoothing for images. Extensive experiments show AuditVotes delivers substantial gains: on Cora-ML under 20-edge attacks, it improves clean accuracy by $437.1\%$ and certified accuracy by $409.3\%$, while maintaining comparable runtime to vanilla smoothing. As a widely applicable and efficient plug-in, AuditVotes offers higher accuracy and stronger guarantees, enabling the practical and certifiably robust GNNs in security-sensitive domains.

\end{abstract}

\ccsdesc[500]{Security and privacy~Formal security models}
\ccsdesc[500]{Security and privacy~Trust frameworks}
\ccsdesc[500]{Security and privacy~Logic and verification}

\keywords{Certified robustness; graph neural networks; provable defense; randomized smoothing; graph augmentation.}

\maketitle

\section{Introduction}
Graph Neural Networks (GNNs)~\cite{kipf2016semi,Hamilton2017inductive,wu2022graph} have emerged as a powerful tool for learning and inference on graph-structured data, finding applications in various domains such as recommendation systems~\cite{fan2019graph,yuan2024contextgnn}, financial fraud detection~\cite{motie2024financial}, and traffic analysis~\cite{dong2023graph}. Despite their success, GNNs are vulnerable to adversarial attacks~\cite{sun2022adversarial,zhai2023state}, which can significantly degrade their prediction accuracy by introducing small, carefully crafted perturbations to the input data
~\cite{zugner2018adversarial,wu2019adversarial,zügner2018adversarial,wang2024efficient}. 
For instance, financial crimes can modify the graph structure by manipulating their transactions to escape fraudster detection~\cite{wu2024safeguarding}. 
This vulnerability has prompted extensive research to develop robust GNN models~\cite{jin2020prognn,liu2021graph,chen2021understanding,zhu2023focusedcleaner}. 
However, a central challenge lies in that the robustness of developed models can be further compromised by more advanced and \textit{adaptive} attacks~\cite{mujkanovic2022defenses,gosch2024adversarial}. Consequently, the robustness uncertainty limits the usage of GNNs in safety-critical applications. 
One promising solution is certified robustness~\cite{li2023sok}, which aims to provide \textit{provable} guarantees that a model's predictions will remain stable under \textit{any} possible adversarial perturbation within a specified range. 

The most representative approach to achieving certified robustness is randomized smoothing ~\cite{li2023sok,cohen2019certified,bojchevski2020efficient,lai2024node,schuchardt2023localized,wang2021certified},
which transforms any \textit{base} classifier into a \textit{smoothed} classifier with robustness guarantees.
Specifically, randomized smoothing utilizes a majority voting mechanism, where each vote is the prediction of the base classifier over a \textit{randomized} graph, produced by adding carefully calibrated noise to the original graph. The class with the most votes is then the final prediction, resulting in a smoothed classifier that can be proven to be robust.
Despite the tremendous advances in certified robustness for GNNs in recent years, several key challenges remain, hindering the practical deployment of certifiably robust GNNs.

\begin{figure}[!t]
    \centering
    \includegraphics[width=1.0\linewidth]{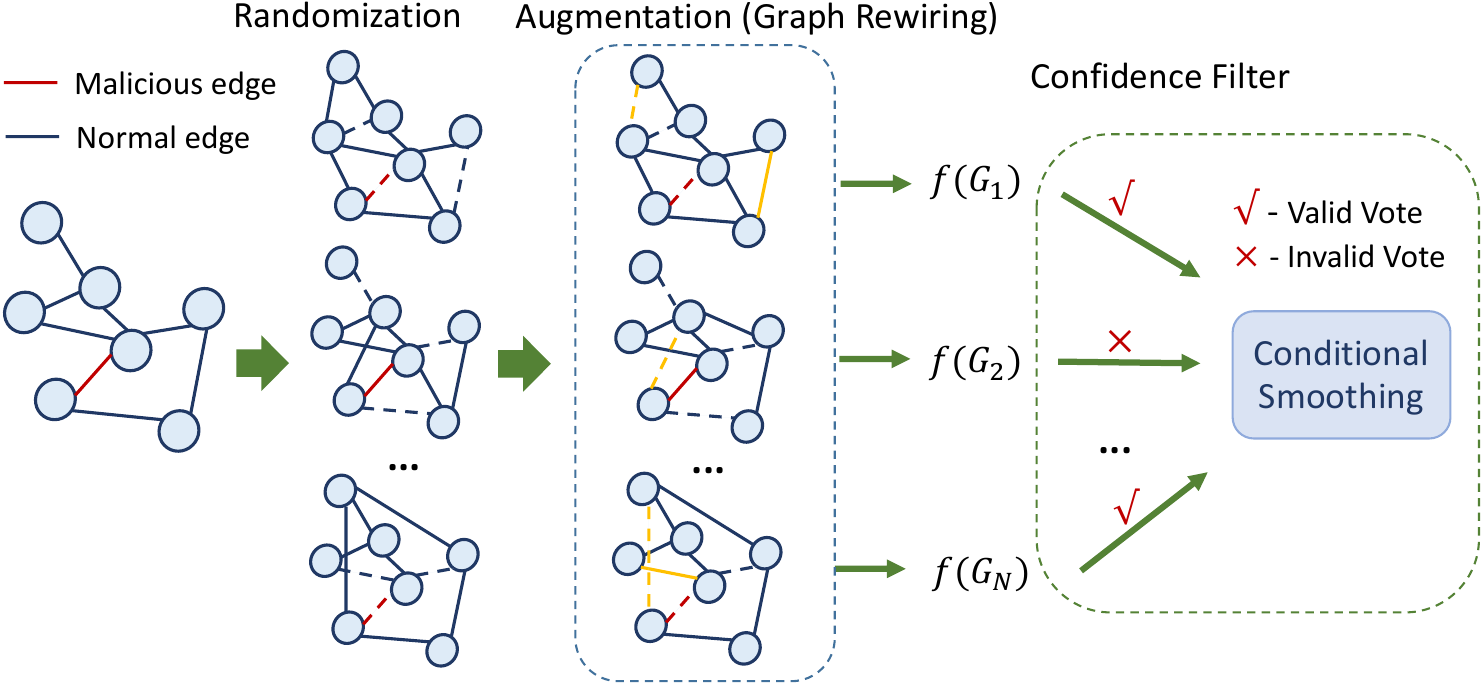}
    \vspace{-15pt}
    \caption{\textbf{AuditVotes} introduces two key components: \underline{au}gmentation and con\underline{dit}ional smoothing. The graph rewiring augmentation de-noises the randomized graph to improve data quality, and the confidence filter removes
    low-quality predictions to enhance prediction consistency.}
    \label{fig:Framework}
    \vspace{-10pt}
\end{figure}

\textbf{Accuracy-Robustness Trade-off}.  
Certified robustness is primarily assessed by certified accuracy, which is the proportion of predictions that are both correct (i.e., clean accuracy) and stable (i.e., robust ratio). Existing works~\cite{bojchevski2020efficient,jia2020certified,wang2021certified,xia2024gnncert,li2025agnncert} have not adequately addressed this inherent trade-off: they achieve a higher robustness ratio often at a cost of introducing larger random noise to the input graph, which can severely degrade data quality and significantly diminish clean accuracy. For example, SparseSmooth~\cite{bojchevski2020efficient} requires applying a higher probability of edge addition $p_{+}$ to ensure a larger certified radius. However, even with a modest $p_{+}$ value of $0.1$, the clean accuracy of the smoothed classifier will drop to less than $0.5$. 
Thus, a key challenge is to effectively reduce the severity of this trade-off to make certifiably robust GNNs practically applicable.

\textbf{Prediction Consistency.} 
It is commonly observed that current smoothed classifiers~\cite{cohen2019certified,bojchevski2020efficient} often exhibit low \textit{prediction consistency}, meaning that the predictions over randomized graphs vary significantly, which can lead to diminished certified performance.
Consequently, several works~\cite{salman2019provably,li2019certified, jeong2020consistency,horvath2021boosting,jeong2023confidence} have focused on improving prediction consistency. However, these approaches often impose a heavy computational burden, resulting in significant overhead for the certification framework. Additionally, most of the proposed techniques are tailored for image classification tasks, with limited exploration in the graph learning domain. 
Therefore, a critical challenge is to develop efficient and graph-specific methods that enhance prediction consistency without introducing excessive computational costs. 

\textbf{Robustness on unseen nodes}.
For the node classification task, existing certifiably robust models focus primarily on 
transductive learning setting ~\cite{bojchevski2020efficient,chen2020smoothing,jia2020certified,wang2021certified,lai2024collective}, where the model leverages the complete graph during training to make predictions on known nodes (already appeared in the graph). 
However, as noted in~\cite{gosch2024adversarial}, this transductive setting falls into a robustness pitfall: the model can remember the clean graph to 
achieve perfect robustness on known nodes while the robustness is unsure on unseen nodes. 
Furthermore, real-world applications often require GNN models to operate in an inductive setting \cite{liu2021indigo,qin2023towards,gao2024graph,motie2024financial}, where they must generalize to accommodate new nodes continuously introduced to the graph. For instance, in financial networks, as new accounts or transactions emerge, models must classify and analyze these entities without prior knowledge of their characteristics. This highlights the necessity of an inductive certification framework to avoid the pitfall of transductive learning and ensure superior robustness for unseen nodes under evasion attacks.

To address the above challenges, we propose \textbf{AuditVotes}, a general certification framework that can achieve both high clean accuracy and certifiably robust performance for GNNs.  Our core strategy is to maintain a larger degree of randomization (for a higher robust ratio) while enhancing the quality of both graph data and votes (to \textit{restore} model accuracy). AuditVotes (Figure~\ref{fig:Framework}) introduces two components seamlessly integrated into the randomized smoothing pipeline: graph rewiring \underline{Au}gmentation and con\underline{dit}ional smoothing. Specifically, augmentation improves input data quality by recovering original structural patterns from noisy graphs, while conditional smoothing enhances output consistency by filtering out low-confidence predictions.

Nevertheless, effectively designing these two components and incorporating them into the smoothing pipeline presents several design requirements: \circled{1} Both augmentation and conditional filter must be computationally efficient to handle thousands of randomized graphs, typical in randomized smoothing. 
\circled{2} The augmentation must adapt to the type and level of injected noise to effectively recover the original graph structure.
\circled{3} Both components should be applicable in inductive learning, generalizing to unseen nodes, which is a practical necessity in evasion attack scenarios~\cite{dai2024comprehensive,gosch2024adversarial}. 
\circled{4} The resulting integrated framework must retain provable robustness guarantees, requiring new theoretical proofs to ensure certification.


In this paper, we instantiate AuditVotes with efficient, noise adaptive, and inductive strategies that meet these demands. For augmentation, we propose three rewiring methods based solely on node features, ensuring they generalize to unseen nodes:
\begin{itemize}[nosep]
    \item[$\bullet$] \textbf{JacAug}, a lightweight and training-free Jaccard‑similarity filter that prunes and adds edges based on feature overlap.
    \item[$\bullet$] \textbf{FAEAug}, a feature auto‑encoder that learns to reconstruct edge patterns from node features.
    \item[$\bullet$] \textbf{SimAug}, a multi‑head similarity model that captures diverse feature relations for rewiring.
\end{itemize}
These augmenters rely only on node features for edge prediction, enabling them to generalize to unseen graphs or nodes.  Moreover, they can be trained and precomputed once before smoothing, avoiding repeated computation. 
For conditional smoothing, we employ a confidence‑based filter (\textbf{Conf}) that excludes low‑confidence predictions from voting. This post‑processing step requires only a scalar comparison per sample, adding negligible overhead while substantially improving prediction consistency. 
We further establish a new robustness certificate for AuditVotes, proving that augmentation and conditional filtering preserve the underlying smoothing guarantees.
Finally, we demonstrate that AuditVotes, as a general framework, can be applied to other smoothing schemes, such as de-randomized smoothing for GNNs~\cite{xia2024gnncert,li2025agnncert,yang2024distributed}, randomized smoothing for image classification~\cite{cohen2019certified,weber2023rab}. 

To validate our approach, we conduct extensive evaluations demonstrating that \textbf{AuditVotes significantly boosts clean accuracy, certified accuracy, and empirical robustness while maintaining high efficiency and wide applicability}.
For example, against edge‑insertion attacks (up to $20$ edges) on the Citeseer dataset, our SimAug improves the clean accuracy from $14.7\%$ to $70.9\%$, and AuditVotes (SimAug+Conf) raises the certified accuracy from $14.7\%$ to $72.6\%$ with the runtime increase of $<1\%$. 
Moreover, we evaluate AuditVotes as an empirical defense under Nettack~\cite{zugner2018adversarial} and IG-attack~\cite{wu2019adversarial}. AuditVotes achieves $83.3\%$ robust accuracy on Cora‑ML with an attack budget of $5$ edges per node.
Furthermore, AuditVotes generalizes effectively to other smoothing schemes: when applied to GNNCert~\cite{xia2024gnncert}, AuditVotes (SimAug) improves clean accuracy by $12.7\%$ and certified accuracy by $13.9\%$ on the Cora-ML dataset. 
With Gaussian smoothing~\cite{cohen2019certified} on the CIFAR-10 dataset, AuditVotes (Conf) raises certified accuracy by  $19.6\%$. To demonstrate the scalability of AuditVotes, we further apply it to AGNNCert~\cite{li2025agnncert} on a large graph dataset with over 2 million nodes, and it only increases the runtime by $<6\%$. 
These results demonstrate that AuditVotes delivers higher accuracy, stronger guarantees, and practical runtime, making it a compelling solution for deploying certifiably robust GNNs in security‑sensitive applications that require high accuracy. In summary, AuditVotes introduces augmentation and conditional smoothing to break the accuracy–robustness trade‑off, operates inductively for real‑world generalization, and applies broadly across smoothing schemes—all with minimal overhead.

\section{Background and Problem Definition}
\label{sec:backg}
In this section, we provide the necessary background by defining notations, describing the node classification task with GNNs,  introducing certifying robustness frameworks applied to GNNs and image classifiers, and defining the threat model. 
\subsection{Inductive Graph Node Classification}
We represent a graph as $\mathcal{G}=(\mathbf{V},\mathbf{E},\mathbf{X})$, where $\mathbf{V}$ and $\mathbf{E}$ denote the set of nodes and edges, respectively,
and $\mathbf{X}$ is the node feature matrix of size $|\mathbf{V}|\times d$.
We use $\mathbf{A}$ to denote the adjacency matrix of the graph $\mathcal{G}$. Each node is associated with a label among $\mathcal{Y}=\{1,2,\cdots, C\}$, and the node classifier, such as a GNN, aims to predict the labels. 
To avoid the robustness pitfall of transductive learning as studied in \cite{gosch2024adversarial}, we consider a fully inductive graph learning that a GNN node classifier $f:\mathbb{G}\rightarrow \mathcal{Y}$ is trained on a training graph $\mathcal{G}_{train}$ and then the model can generalize to unseen nodes in the testing graph $\mathcal{G}_{test}$. Note that the validation nodes and testing nodes are strictly excluded from the training graph (see the experimental setup in Section~\ref{sec:evaluation} for more details).

\subsection{Certified Robustness for GNNs}
The certified robustness model for GNNs can be grouped as a randomized and de-randomized smoothing framework. The former adds random noise to the graph, while the latter partition the graph into several subgraphs. Then, the certificate is established based on ``majority votes" over multiple inputs.
\subsubsection{Randomized Smoothing}
The mainstream approach to realize certified robustness is \textit{randomized smoothing} with representative works ~\cite{bojchevski2020efficient,jia2020certified,wang2021certified,lai2024node,lai2024collective} calibrated for GNNs. 
Specifically, given an input graph $\mathcal{G}$ and \textit{any} base classifier $f(\cdot)$ (such as a GNN), they will add random noise to $\mathcal{G}$, resulting in a collection of randomized graphs denoted by $\phi(\mathcal{G})$, where $\phi(\cdot)$ denotes the randomization process. Then, the base classifier $f$ is used to make predictions over the random graphs, and the final prediction is obtained through majority voting. Equivalently, randomized smoothing can convert any base classifier into a \textit{smoothed classifier} $g(\cdot)$, defined as:
\begin{align}
\label{eqn:smooth_g}
    &g_v(\mathcal{G}):=\argmax_{y\in \mathcal{Y}}p_{v,y}(\mathcal{G}) :=\mathbb{P}(f_v(\phi(\mathcal{G}))=y),
\end{align}
where $\mathbb{P}(f_v(\phi(\mathcal{G}))=y)$ denotes the probability of predicting a node $v$ as class $y$.
Then, it can be proved that the smoothed classifier $g$ is provably robust with respect to a certain perturbation space $\mathcal{B}$, which defines the set of perturbations introduced by the attacker. In this paper, we take SparseSmooth~\cite{bojchevski2020efficient} as an example, and certify against Graph Modification Attacks (GMA) perturbation where the attacker can add at most $r_a$ edges and delete at most $r_d$ edges among existing nodes.
The $\phi(\mathcal{G})$ is defined as randomly adding edges with probability $p_{+}$ and removing edges with probability $p_{-}$. Larger $p_-$ and $p_+$ yield larger certifiable $r_a$ and $r_d$. However, a high level of noise hinders the accuracy of the smoothed model. In this paper, we aim to improve the trade-off between model accuracy and robustness. 

\subsubsection{De-randomized Smoothing} 
De-randomized smoothing~\cite{xia2024gnncert,yang2024distributed,li2025agnncert} divides the graphs into several groups with fixed randomness. For instance, GNNCert~\cite{xia2024gnncert} and AGNNCert-E~\cite{li2025agnncert} partition the edges $\mathbf{E}$ in the graph into $T_s$ groups $\{\mathbf{E}_1, \mathbf{E}_2,\cdots, \mathbf{E}_{T_s}\}$ via a hash function $\mathcal{H}(\cdot)$:
\begin{equation}
    \mathbf{E}_{i}= \{\mathcal{H}(s_u\oplus s_v) \% T_s+1 = i|(u,v)\in \mathbf{E}\},
\end{equation}
where $i=1,2,\cdots,T_s$, $s_u$ denotes the ID of node $u$, and $\oplus$ represents the concatenation of two strings. These edge sets then form $T_s$ subgraphs $\{\mathcal{G}_1, \mathcal{G}_2,\cdots, \mathcal{G}_{T_s}\}$. Then, the final prediction is also obtained by ``majority vote" over $T_s$ subgraphs:
\begin{align}
\label{eqn:gc_gnncert}
    g_v(\mathcal{G})=\argmax_{y\in \mathcal{Y}} N_v(y):=\sum_{i=1}^{T_s} \mathbb{I}\{f_v(\mathcal{G}_i)=y\},
\end{align}
where $f_v(\mathcal{G}_i)$ is a base classifier that take $\mathcal{G}_i$ as input, and the subscript $(\cdot)_v$ represents the classification results for node $v$. $N_v(y)$ denotes the number of counts that the base classifier votes class $y$ to node $v$. Note that GNNCert~\cite{xia2024gnncert} can also partition node features at the same time to defend against feature manipulation. Since we focus on structural attack, we set the group number of feature partition $T_f=1$ (keep all the nodes). 

Assuming that the attacker can insert several nodes to perturb the graph classification, one malicious node only shows in one subgraph. Consequently, the model is certifiably robust to a certain number of edge perturbations (addition or deletion). The group number $T_s$ (analogous to $p_-$) controls the trade-off between model accuracy and robustness. As the $T_s$ increases, the subgraph becomes sparser while the model can tolerate more edge perturbation. Nevertheless, 
the model accuracy is not satisfying because the subgraph contains poor information for each node. Hence, the accuracy-robustness trade-off also exists.




\subsection{Threat Model}
\subsubsection{\textbf{Attacker}} It is known that GNN models are vulnerable to adversarial attacks~\cite{sun2022adversarial,zügner2018adversarial,zugner2018adversarial} where the attacker can modify the input graph $\mathcal{G}$ (e.g., structure attack) to mislead node classification. In this paper, we focus on GMA in which the attacker can modify some of the edges among the existing nodes. Specifically, we define the attacker setting as follows:
\begin{itemize}
    \item \textbf{Attacker's knowledge}: We assume the white-box (worst-case) attacker knows all the graph structure, node features, and node classifier.  
    \item \textbf{Attacker's power}: The attacker can add at most $r_a$ edges and delete at most $r_d$ edges among existing nodes, which we formally describe the perturbation space as: $B_{r_a,r_d}(\mathcal{G}):=\{\mathcal{G}'|\sum_{ij} \mathbb{I}(\mathbf{A}'_{ij}=\mathbf{A}_{ij}-1)\leq r_d, \sum_{ij} \mathbb{I}(\mathbf{A}'_{ij}=\mathbf{A}_{ij}+1)\leq r_a\}$.
    \item \textbf{Attacker's goal}: The attacker aims to perturb the classification result of a target node in the testing phase (after the model training).
\end{itemize}

\subsubsection{\textbf{Defender}} 
%

For any node classifier, the defender's goal is to provide certified robustness that verifies whether the prediction for a node is consistent under a bounded attack power $B_{r_a,r_d}(\mathcal{G})$. Furthermore, an ideal provable defense is supposed to provide:
\begin{itemize}
    \item \textbf{High clean accuracy}: the model accuracy on the clean graph without considering perturbation. 
    \item \textbf{High certified accuracy}: the ratio of nodes that are both correctly classified and certifiably robust. 
\end{itemize}

\section{The AuditVotes Framework}
\label{sec:auditV}


We design AuditVotes as a general framework (Figure~\ref{fig:Framework}) to improve the performance of certified robustness for GNNs. The term \textbf{AuditVotes} derives from its two essential components, \textbf{\underline{Au}}gmentation and Con\textbf{\underline{dit}}ional Smoothing, encapsulating our main idea of auditing or enhancing the quality of votes to bolster the majority-vote-based certification. Here, we introduce the high-level idea of AuditVotes and instantiate it in Section~\ref{sec:Instantiation}. 

\subsection{Graph Rewiring Augmentation}

Randomized smoothing will inject undesirable noise into the graph, which is the fundamental cause of the trade-off between clean accuracy and certified robustness. Existing works often struggle to balance this trade-off effectively: achieving high certified robustness frequently results in a significant sacrifice in clean accuracy.  For example,  
existing research shows that most attacks on GNNs prefer adding edges than deleting edges ~\cite{zügner2018adversarial,wu2019adversarial,jin2021adversarial,sun2022adversarial,li2023revisiting}. 
Consequently, achieving high certified accuracy against \textit{edge addition} is practically significant.
To achieve this, employing a larger $p_{+}$ (i.e., noise level of adding edges) becomes essential to ensure a larger certification radius. 
However, even with a modest $p_{+}$ value of $0.1$, the clean accuracy of the smoothed classifier will drop to less than $0.50$~\cite{bojchevski2020efficient}.


To better balance the trade-off, we propose an augmentation component that operates on the randomized graph $\phi(\mathcal{G})$ before it goes through the base classifier $f$, thus serving as a \textit{pre-processing} step to remove the noise. 
Our strategy is to \textit{keep good randomization parameters} that will lead to superior certification performance, and process the randomized graph to \textit{restore model accuracy}.
Our choice for processing the randomized graphs is \textit{graph rewiring augmentation}~\cite{ding2022data,zhao2022graph}, which is a widely used technique for improving the utility of graphs by rewiring the edges. 


However, effectively fitting augmentation into the randomized smoothing pipeline has several requirements. \textbf{First}, the augmentation process should be efficient, as it needs to process a large number of randomized graphs. 
\textbf{Second}, the augmentation should be adaptive to the type and level of noise added to the graph. For instance, when the add-edge probability $p_+$ is high, the augmentation is supposed to remove more edges in order to recover the original graph. 
\textbf{Third}, the augmenter should be applicable in inductive learning, as the evasion attack on the graph presents in inductive learning that the model has not seen the testing nodes/graphs. 

To apply augmentation, we create a function composition $f(\mathcal{A}(\cdot))$ to be the new base classifier, where $\mathcal{A}(\cdot)$ denotes an augmentation model. With this scheme, the existing black-box certificates such as~\cite{bojchevski2020efficient,jia2020certified,wang2021certified,xia2024gnncert} can be easily applied to our augmented model without change because these certificates are model-agnostic. Notably, it also offers
us the freedom to design the augmenter $\mathcal{A}$. \textbf{In Section \ref{sec:Instantiation}, we realized three simple yet effective augmentation schemes based on node feature similarities to demonstrate the feasibility of this idea}.

\subsection{Conditional Smoothing with Certified Robustness}
\label{sec:condition}

While augmentation enhances the input graph, the base classifier could still make low-quality predictions over the processed graph. Robustness certificates based on randomized smoothing~\cite{bojchevski2020efficient,chen2020smoothing,jia2020certified,wang2021certified} fundamentally depend on prediction consistency, with higher consistency leading to stronger robustness guarantees~\cite{jeong2020consistency,horvath2021boosting}. To address this, we further propose a \textit{conditional smoothing} framework that \textit{post-processes} the predictions of the base classifier before entering the voting procedure. This framework filters out low-quality predictions to enhance prediction consistency.

Formally, we define a \textit{filtering function}, denoted as $h(\cdot)$, which takes only a randomized graph as input and outputs either $0$ or $1$, where $0$ indicates that the prediction is included in the voting process and $1$ excludes it. Then, we can construct a \textbf{conditional smoothed classifier} $g^c(\cdot)$ as follows:
\begin{align}
\label{eqn:consmooth_g}
    &g^c(\mathcal{G}):=\argmax_{y\in \mathcal{Y}}\mathbb{P}(f(\phi(\mathcal{G}))=y|h(\phi(\mathcal{G}))=0),
\end{align}
where $p_{v,y}(\mathcal{G}):=\mathbb{P}(f(\phi(\mathcal{G}))=y|h(\phi(\mathcal{G}))=0)$ represents the probability of predicting an input graph $\mathcal{G}$  as class $y$ \textit{conditioned on the criterion that $h(\phi(\mathcal{G}))=0$}.

\subsubsection{\textbf{Certified Robustness with Arbitrary Filters.}
The key theoretical challenge is that the filter $h$ could distort the smoothing distribution and break existing certificates. Our main theoretical contribution is to prove that conditional smoothing retains certifiable robustness under \textit{arbitrary filtering functions}, provided we adapt the certification procedure appropriately.}

We state the robustness condition formally in the following theorem (proof in Appendix~\ref{AppendixA.1}).
Specifically, given a node $v$, let $y_A$ and $y_B$ denote the top predicted class and the runner-up class, respectively. Let $\underline{p_A}$ and $\overline{p_B}$ denote the lower bound of $p_{v,y_A}$ and the upper bound of $p_{v,y_B}$, respectively. 
\begin{theorem}
\label{theorem-condition}
    Let the conditional smoothed classifier $g^c(\cdot)$ be as defined in Eq.~\eqref{eqn:consmooth_g}, where $h(\cdot)$ is an arbitrary filtering function that takes only a randomized graph as input and outputs either $0$ or $1$.
    We divide the total sample space $\mathbb{G}$ into $I$ disjoint regions $\mathbb{G}=\bigcup_{i=1}^I \mathcal{R}_i$, where $\mathcal{R}_i$ denote the consent likelihood ratio region that $\frac{\mathbb{P}(\phi(\mathcal{G})=Z)}{\mathbb{P}(\phi(\mathcal{G}')=Z)}=c_i$ (a constant), $\forall Z\in\mathcal{R}_i$. 
    Let $r_i=\mathbb{P}(\phi(\mathcal{G})\in\mathcal{R}_i)$, $r'_i=\mathbb{P}(\phi(\mathcal{G}')\in\mathcal{R}_i)$ denote the probability that the random sample fall in the partitioned region $\mathcal{R}_i$. We define $\mu_{r_a,r_b}$ as follows:
    \begin{align}
    \label{opt:randomsmooth}
    \mu_{r_a,r_b}:=\min_{\mathbf{s},\mathbf{t}}\quad &\mathbf{s}^T\mathbf{r}'-\mathbf{t}^T\mathbf{r}', \\
    \text{s.t.} \quad &\mathbf{s}^T\mathbf{r}=\underline{p_A},\mathbf{t}^T\mathbf{r}=\overline{p_B},\mathbf{s}\in[0,1]^{I},\mathbf{t}\in[0,1]^{I}.\nonumber
    \end{align}
    Then for $\forall \mathcal{G}'\in \mathcal{B}_{r_a,r_b}(\mathcal{G})$, we have $g_v(\mathcal{G}')=g_v(\mathcal{G})$, if $\mu_{r_a,r_b}>0$.  
\end{theorem}
Theorem \ref{theorem-condition} lays the theoretical foundation for integrating the filtering function $h(\cdot)$ into the randomized smoothing framework. Notably, it offers us the freedom to design the function $h$ to improve the quality of votes. In practice, it enhances the certified accuracy of smoothed classifiers when $h(\cdot)$ can efficiently improve prediction consistency while filtering out only a small number of samples (There is a trade-off, because the sample size affects the Clopper-Pearson bound in Appendix~\ref{Sec:AppendixB.3}, and a large sample size is related to better certified performance). In Section~\ref{sec:filterfunc}, we implement a simple yet effective filtering function based on the confidence of the prediction made by the classifier.

\section{Instantiating AuditVotes: Efficient and Inductive Implementation}
 \label{sec:Instantiation}

In this section, we instantiate the augmentation and conditional smoothing in AuditVotes with efficient and effective strategies. More implementation details and algorithms are in Appendix~\ref{Sec:implement_detail}.
\subsection{Efficient, Noise-Adaptive, and Inductive Augmentations}

We design graph rewiring augmentation (de-noise) methods that are computationally efficient, suitable for the inductive setting, and adaptive to noise levels. Randomized smoothing in the graph domain typically involves two types of noise:
\begin{itemize}
\item \textit{Edge addition} (with probability $p_+$), which introduces spurious edges that can increase graph density and heterophily, making it harder for the model to distinguish meaningful patterns.
\item \textit{Edge deletion} (with probability $p_-$), which might remove important edges, disrupting the original graph structure and reducing connectivity.
\end{itemize}
To address the side effects introduced by the randomization, we propose three augmentation strategies--JacAug, FAEAug, and SimAug, to de-noise the noisy graphs (Figure~~\ref{fig:aug_framework}). These strategies are:
\begin{itemize}
    \item \textbf{Efficient:} The augmentation pipeline minimizes redundant computations by precomputing reusable components (e.g., edge intensity matrix) and applying lightweight operations during inference. This design ensures scalability to large sampling sizes in randomized smoothing.
    \item \textbf{Noise-Adaptive:} The augmentation dynamically adjusts to noise parameters $p_+$ and $p_-$ by estimating the number of edge additions and deletions introduced by randomization, pruning, and recovering edges accordingly.
    \item \textbf{Inductive:} The augmentation methods learn transferable patterns from node features, ensuring effective generalization to unseen testing graphs.
\end{itemize}

\subsubsection{\textbf{Lightweight Jaccard‑Based Augmentation (JacAug)}} Inspired by the simple but effective defense approach proposed for GNNs, the GCNJaccard~\cite{wu2019adversarial}, we propose a simple graph rewiring augmentation. We set a threshold $\tau$ to prune the existing edge with $J_{u,v}\leq\tau$, and we add the edge if the Jaccard similarity exceeds a threshold $\xi$. Specifically, the Jaccard similarity (binary features) between node $u$ and $v$ is represented as follows:
\begin{equation}
   J_{u,v}= \frac{\mathbf{x}_u^T \mathbf{x}_v}{\mathbf{x}_u^T\mathbf{1}+\mathbf{x}_v^T\mathbf{1}-\mathbf{x}_u^T \mathbf{x}_v},
\end{equation}
where the $\mathbf{x}_u$ and $\mathbf{x}_v$ are the node features, and $\mathbf{1}$ denotes the all one vector, and $T$ denotes the matrix transpose.




\begin{figure}[!t]
    \centering
    \includegraphics[width=0.9\linewidth]{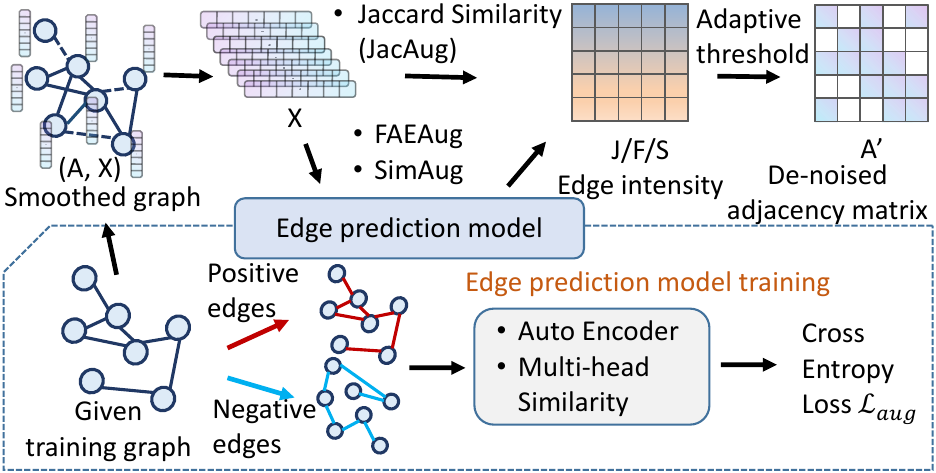}
    \vspace{-5pt}
    \caption{We propose three graph rewiring augmentations (JacAug, FAEAug, and SimAug). JacAug is a training-free augmenter, and the others are learnable augmenters that are trained to recover the original graph via edge prediction.}
    \label{fig:aug_framework}
    \vspace{-8pt}
\end{figure}

\subsubsection{\textbf{Feature‑Autoencoder Augmentation (FAEAug)}} Auto-encoder is widely used for graph augmentation~\cite{zhao2021data}. Nevertheless, when the smoothing applies a high level of noise that randomly adds edges ($p_+$), GNN-based models are ineffective because the graph becomes too dense and highly heterophilic. To tackle the challenge, we propose a simple graph augmentation that only relies on node features:
\begin{align}
\mathbf{F}=\sigma(\mathbf{Z}\mathbf{Z}^T), \quad \mathbf{Z}= \mathbf{W}_1(\text{ReLU}(\mathbf{W}_2(\mathbf{X}))),
\end{align}
where $\mathbf{W}_1$ and $\mathbf{W}_2$ are trainable parameters, and $\sigma$ is a sigmoid activation function. This augmentation model is suitable for the inductive learning setting. The parameters $\mathbf{W}_1$ and $\mathbf{W}_2$ trained on the training graph can be generalized to testing graphs with new nodes. Specifically, given a training subgraph, we can sample a set of positive edges $E_{pos}$ and negative edges $E_{neg}$. Then, we employ binary cross-entropy loss for the graph augmentation (edge prediction) model training:
\begin{align}
\label{eqn:loss_fae}
    \mathcal{L}_{aug}=&-\frac{1}{|E_{pos}\cup E_{neg}|}[\sum_{(u,v)\in E_{pos}}log(F_{u,v})
    \nonumber\\
    &+\sum_{(u,v)\in E_{neg}}log(1-F_{u,v})],
\end{align}
where $F_{u,v}$ is an element in $\mathbf{F}$ corresponding to the edge $(u,v)$, representing the probability of edge existence. 

\subsubsection{\textbf{Multi‑Head Similarity Augmentation (SimAug)}} Inspired by~\cite{zhang2020gnnguard,jiang2019semi,chen2020iterative}, we implement an effective graph augmentation model based on a multi-head similarity function:
\begin{align}
&S_{u,v}^q= cos(\mathbf{w}_q \circ \mathbf{x}_u,\mathbf{w}_q \circ \mathbf{x}_v), \, S_{u,v}=\frac{1}{m}\sum_{q=1}^m s_{u,v}^{q},
\end{align}
where $\mathbf{w}_q$ is a trainable parameter that weighted the node feature dimension, and $m$ is the number of heads. Compared to the Jaccard similarity, this multi-head similarity function is supposed to capture a more comprehensive pattern because different weights correspond to a different similarity function. 
This augmentation model is also suitable for inductive settings and follows a similar training procedure with the loss function in Eq.~\eqref{eqn:loss_fae} (substituting $F_{u,v}$ by $S_{u,v}$).


\subsubsection{\textbf{Efficient augmentation}}
To reduce computation workload and avoid repeated similarity or edge intensity computation during smoothing, we only need to pre-calculate $J_{u,v}/F_{u,v}/S_{u,v}$ for all potential edges and store it in an edge intensity matrix $\mathbf{I}:=\mathbf{J}\text{ or }\mathbf{F}\text{ or }\mathbf{S}$. During inference, graph rewiring is performed efficiently using the following formula:
\begin{equation}
\label{eqn:augment_filter}
    \mathbf{A}'= \mathbf{A} \circ (\mathbf{I}>\tau) +  (\mathbf{A}=0)\circ(\mathbf{I}>\xi),
\end{equation}
where $\circ$ denotes element-wise matrix multiplication, and $(\mathbf{I} > \tau)$ is a binary matrix where each element is 1 if $I_{u,v} > \tau$ and 0 otherwise. The term $\mathbf{A} \circ (\mathbf{I} > \tau)$ retains edges with $I_{u,v} > \tau$. Similarly, $(\mathbf{A} = 0)$ is the complement of the adjacency matrix, where each element is 1 if $A_{u,v} = 0$, and the term $(\mathbf{A} = 0) \circ (\mathbf{I} > \xi)$ adds edges with $I_{u,v} > \xi$.

\subsubsection{\textbf{Noise-adaptive thresholds}}
\label{sec:threshold}
The augmentation strategies require two thresholds, $\tau$ and $\xi$. We propose a simple and effective method to select noise-adaptive thresholds automatically. Specifically, given the noise parameters $p_+$ and $p_-$, we can estimate the number of edges that are randomly added or deleted in the smoothing randomization. Firstly, we calculate the edge sparsity in the training graph: $e_{ratio}:=\sum\mathbf{A}_{train}/|\mathbf{V}_{train}|^2$, where $\mathbf{A}_{train}$ is the adjacency matrix and $|\mathbf{V}_{train}|$ is the size of training graph. For the testing graph, the expected number of edges before perturbation is $E':=e_{ratio} \times |\mathbf{V}_{test}|^2$. 
The number of edges to be added and removed is then estimated as:
\begin{align}
    ADD &= \lfloor E' \times p_-\rfloor, \\
    DEL &= \lfloor (|\mathbf{V}_{test}|^2 - E') \times p_+\rfloor,
\end{align}
where $ADD$ corresponds to edges deleted during randomization, while $DEL$ represents spurious edges added. Threshold $\xi$ is set based on the top-$ADD$ values (in descending order), and threshold $\tau$ is set based on the top-$DEL$ values (in ascending order). This approach ensures that the thresholds adapt efficiently to the noise level.


\begin{figure*}[!ht]
\centering
\subfigure[Citeseer]{\includegraphics[width=0.160\textwidth,height=2.5cm]{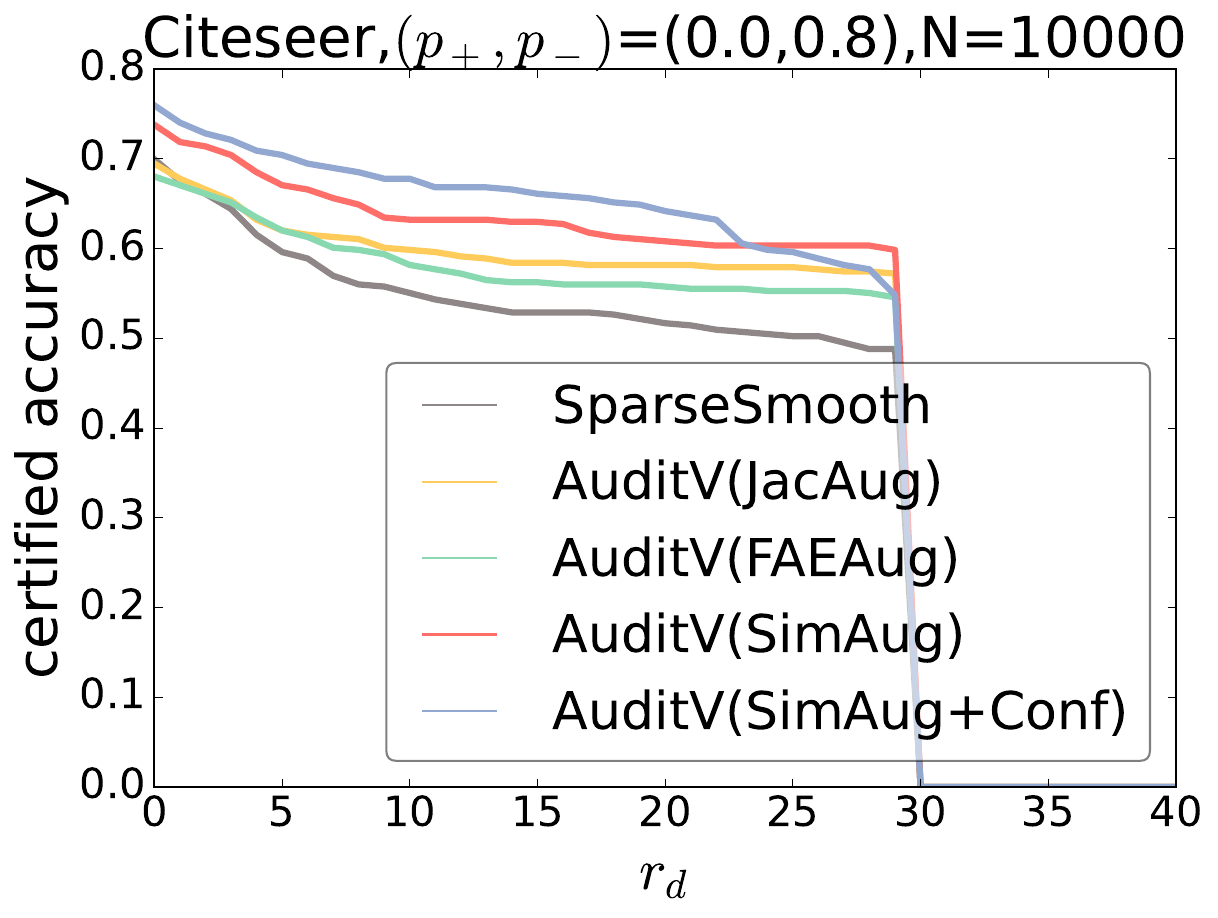}}
\subfigure[Citeseer]{\includegraphics[width=0.160\textwidth,height=2.5cm]{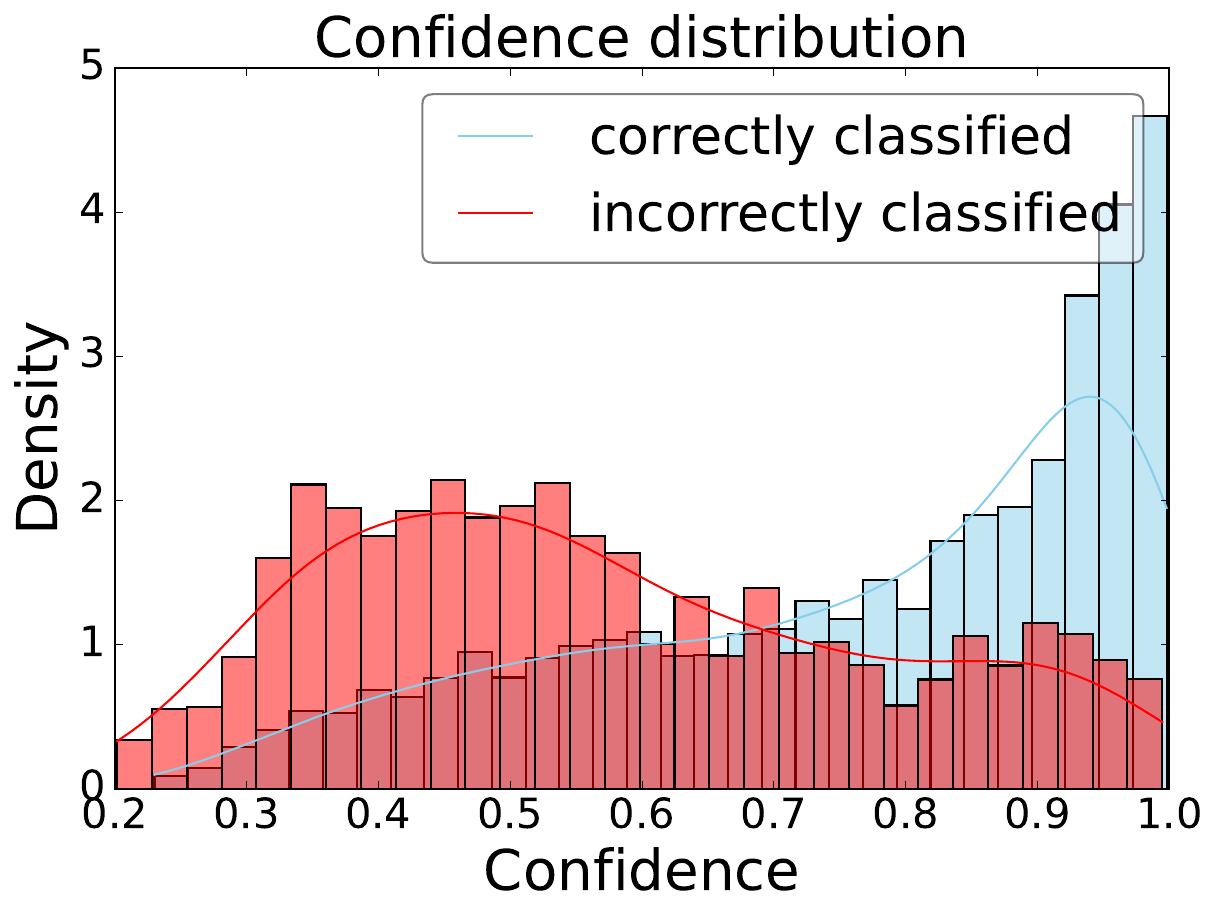}}
\subfigure[Citeseer]{\includegraphics[width=0.160\textwidth,height=2.5cm]{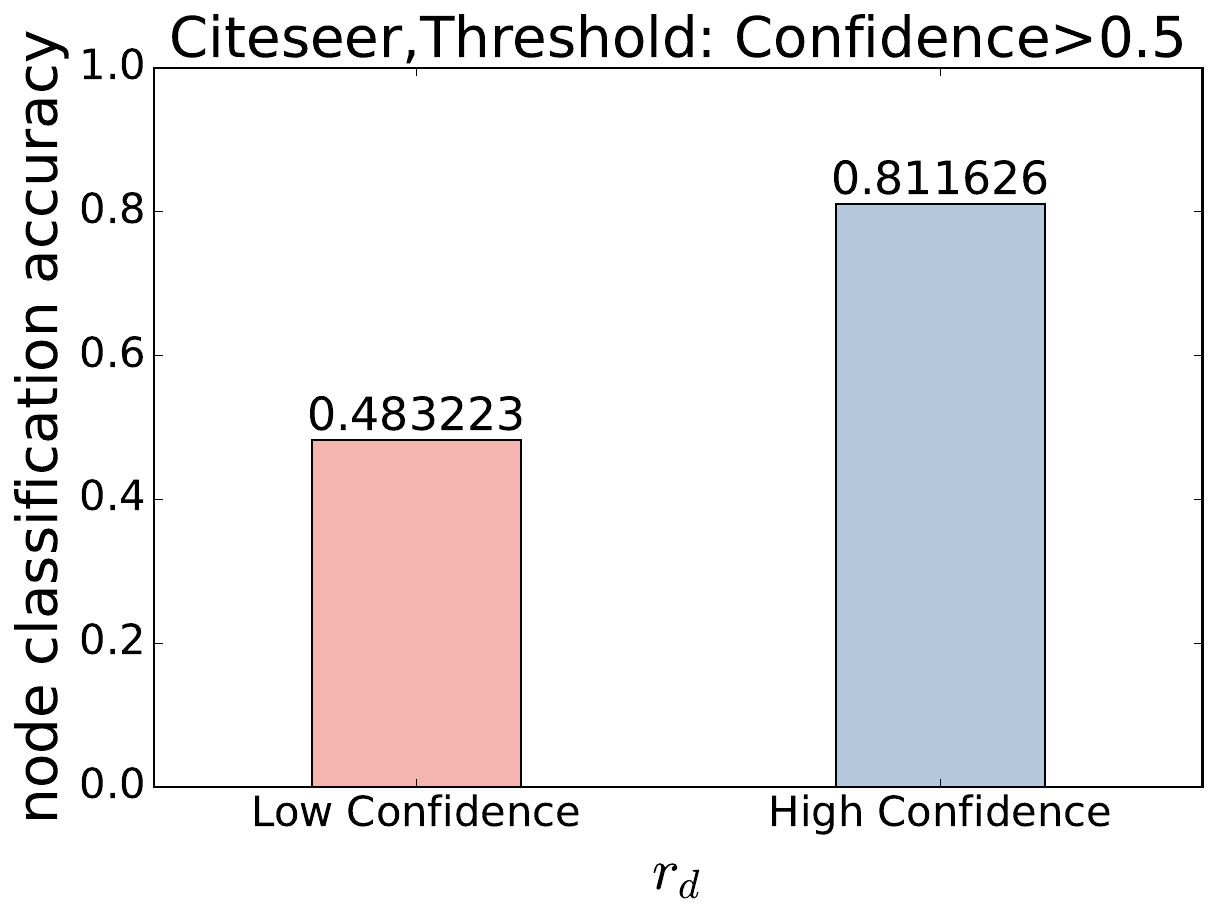}}
\subfigure[Cora-ML]{\includegraphics[width=0.160\textwidth,height=2.5cm]{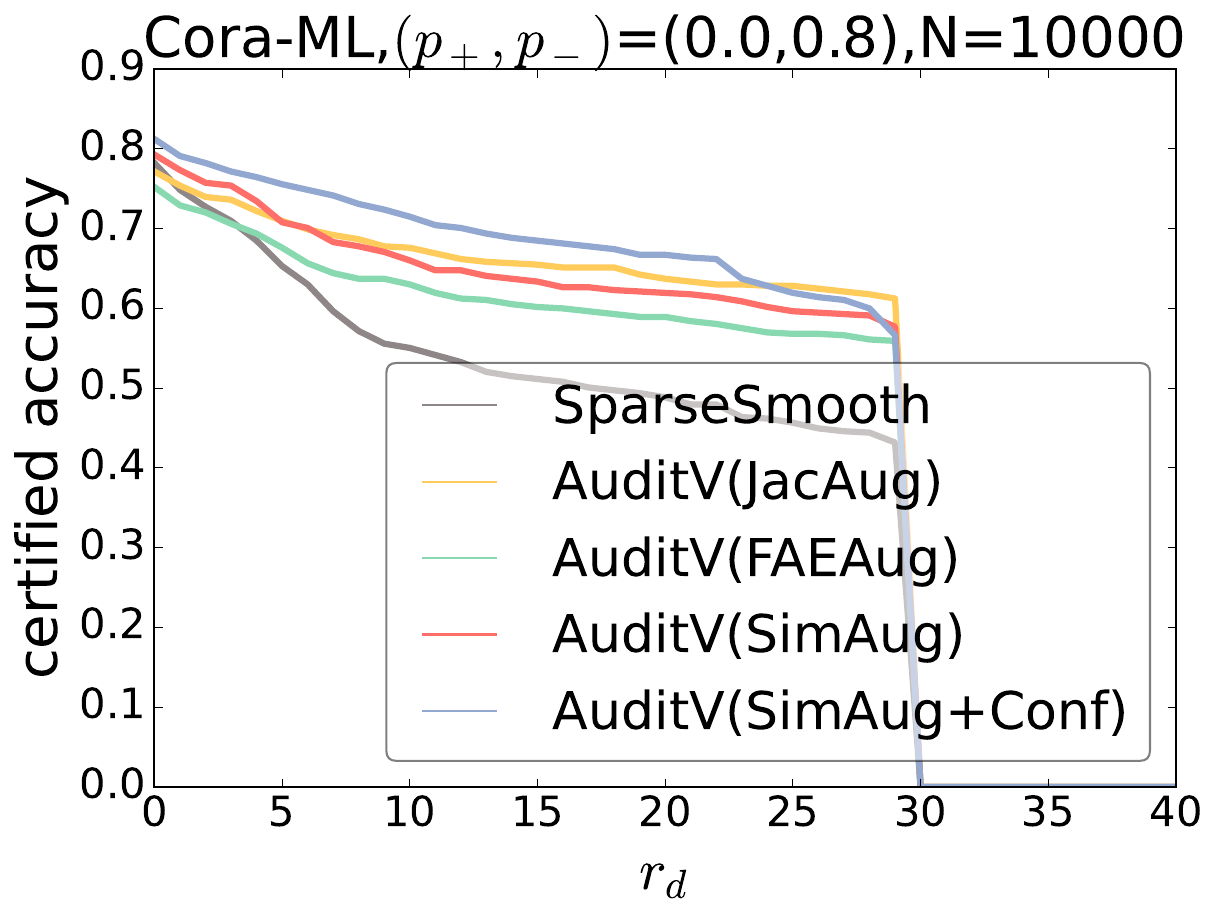}}
\subfigure[Cora-ML]{\includegraphics[width=0.160\textwidth,height=2.5cm]{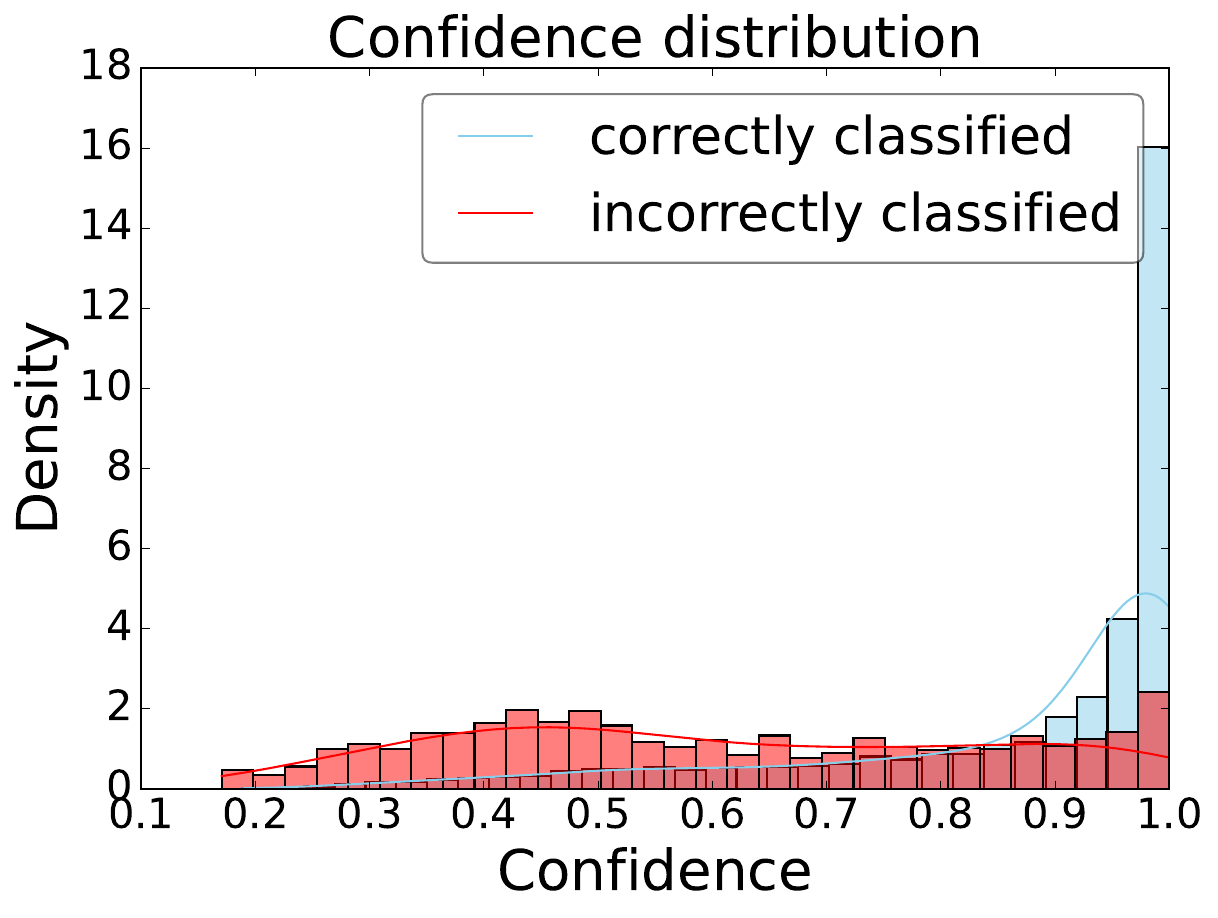}}
\subfigure[Cora-ML]{\includegraphics[width=0.160\textwidth,height=2.5cm]{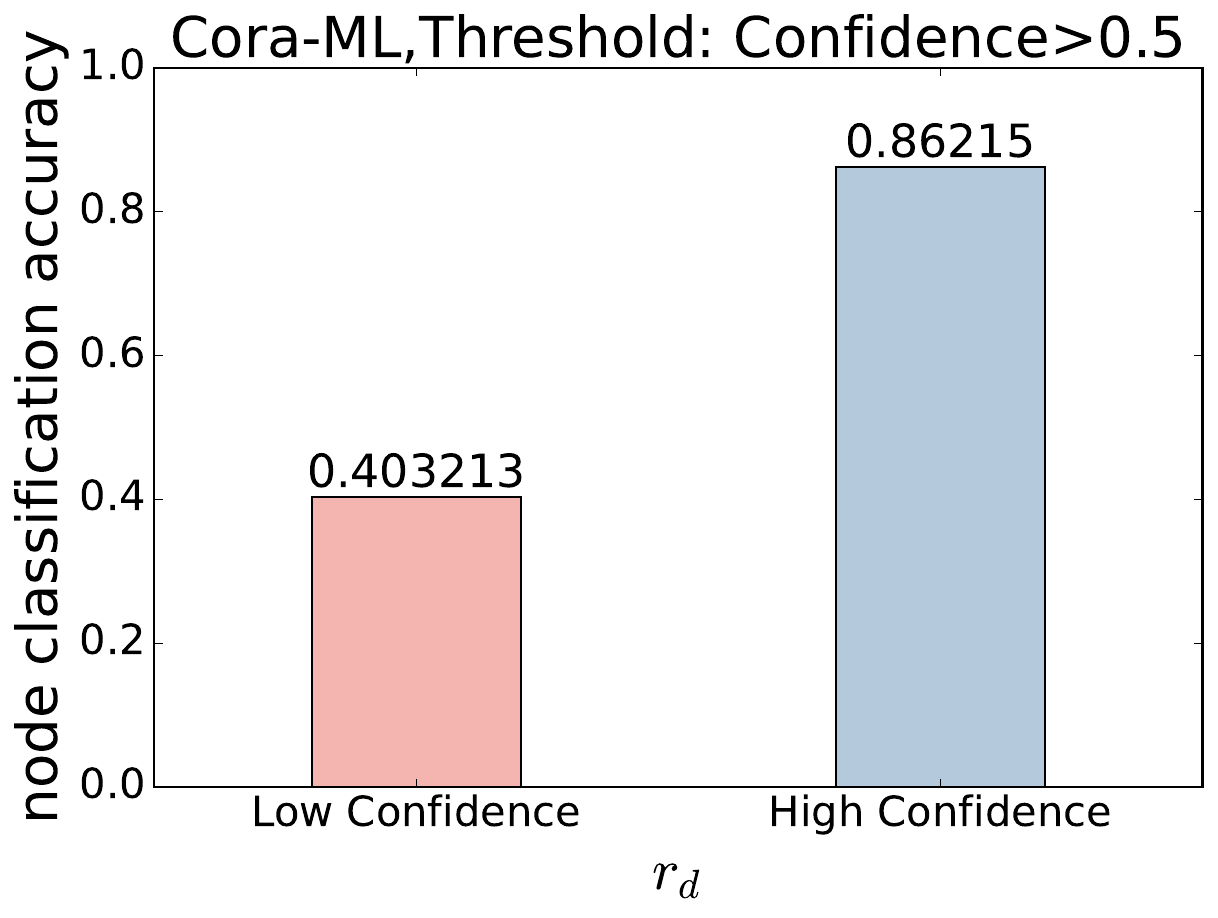}}
\subfigure[Citeseer]{\includegraphics[width=0.160\textwidth,height=2.5cm]{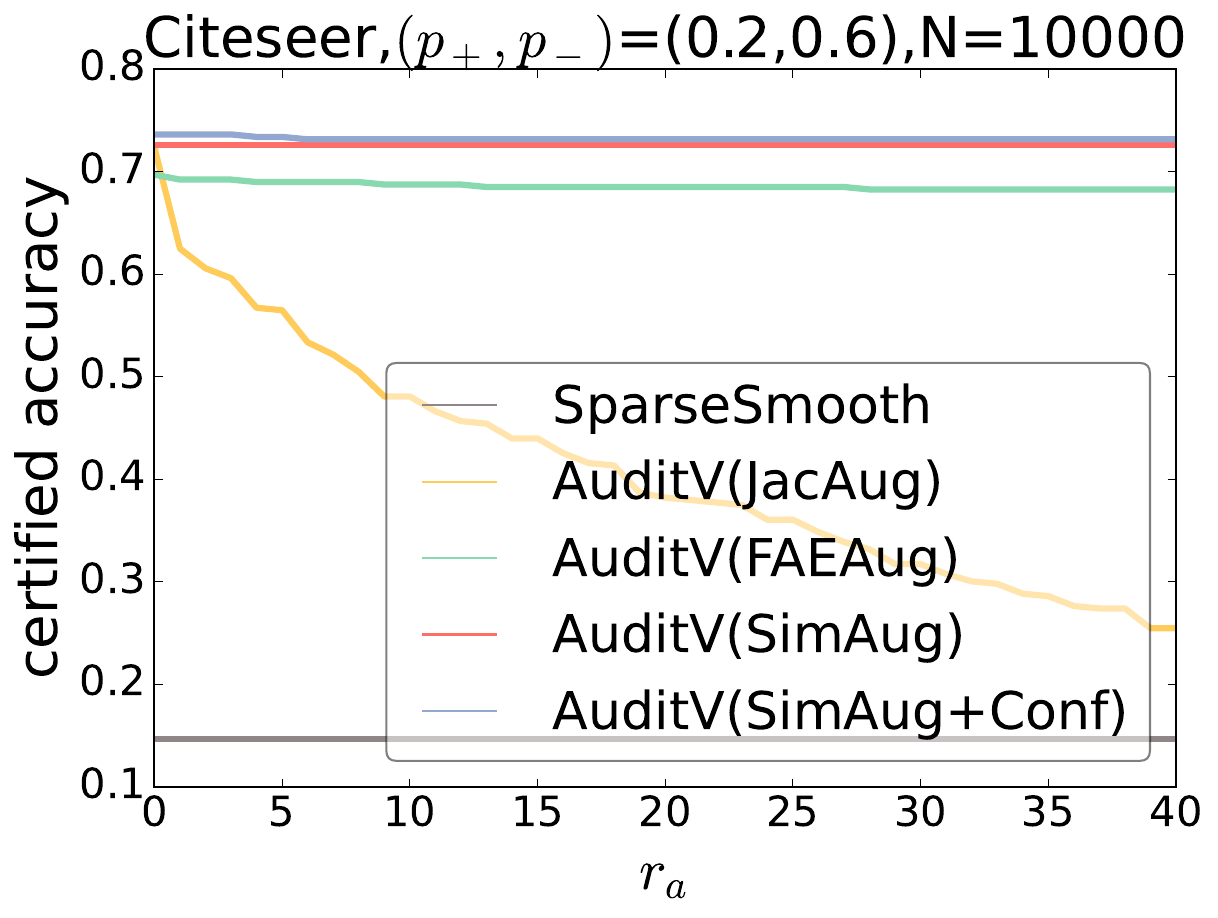}}
\subfigure[Citeseer]{\includegraphics[width=0.160\textwidth,height=2.5cm]{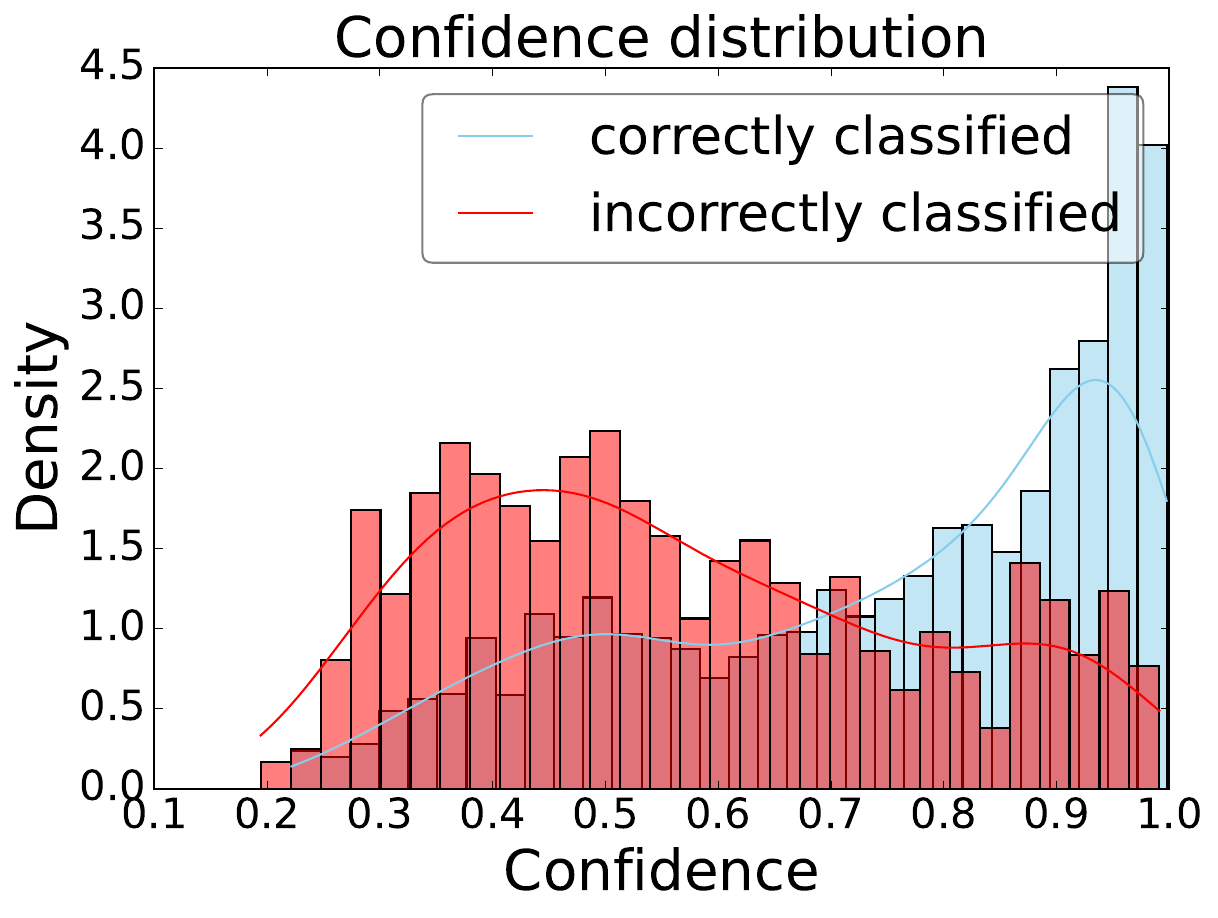}}
\subfigure[Citeseer]{\includegraphics[width=0.160\textwidth,height=2.5cm]{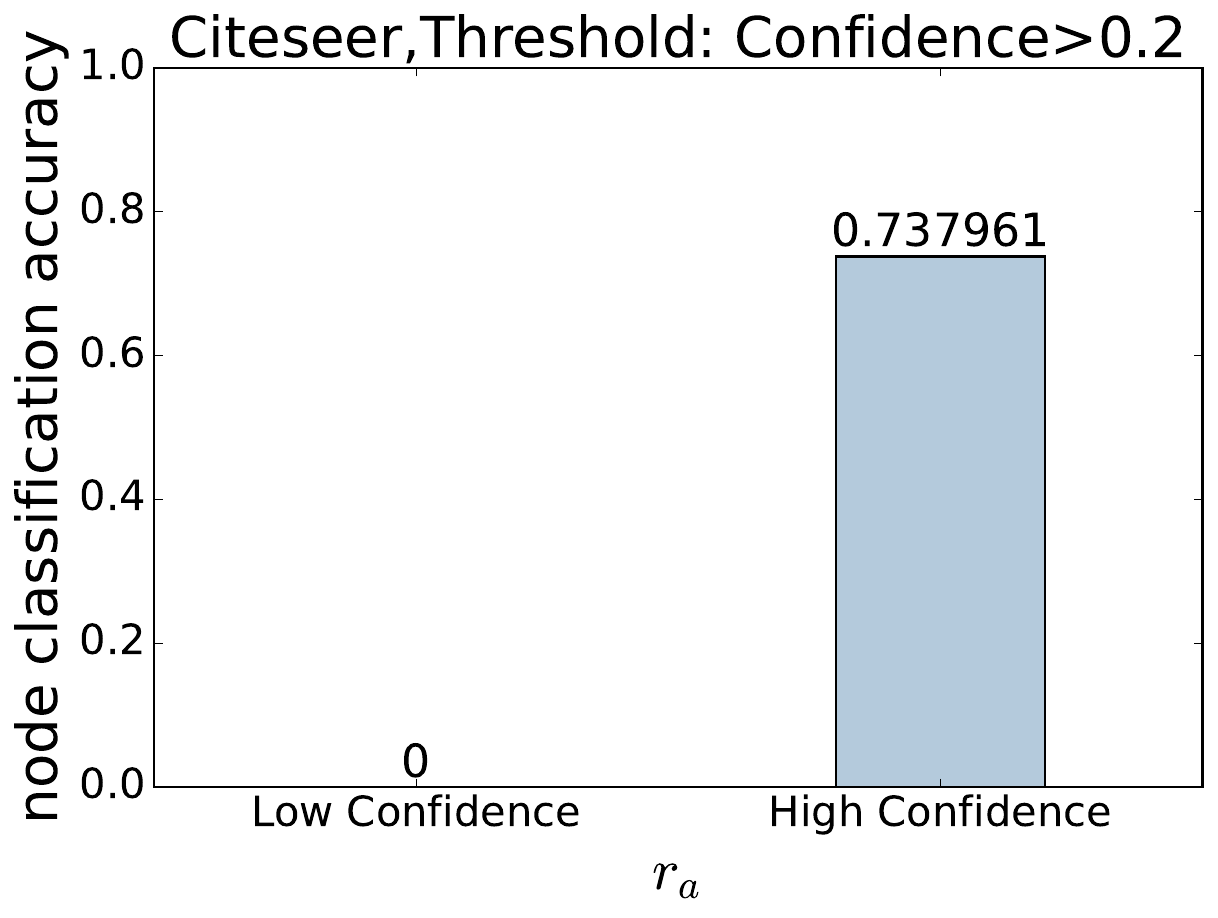}}
\subfigure[Cora-ML]{\includegraphics[width=0.160\textwidth,height=2.5cm]{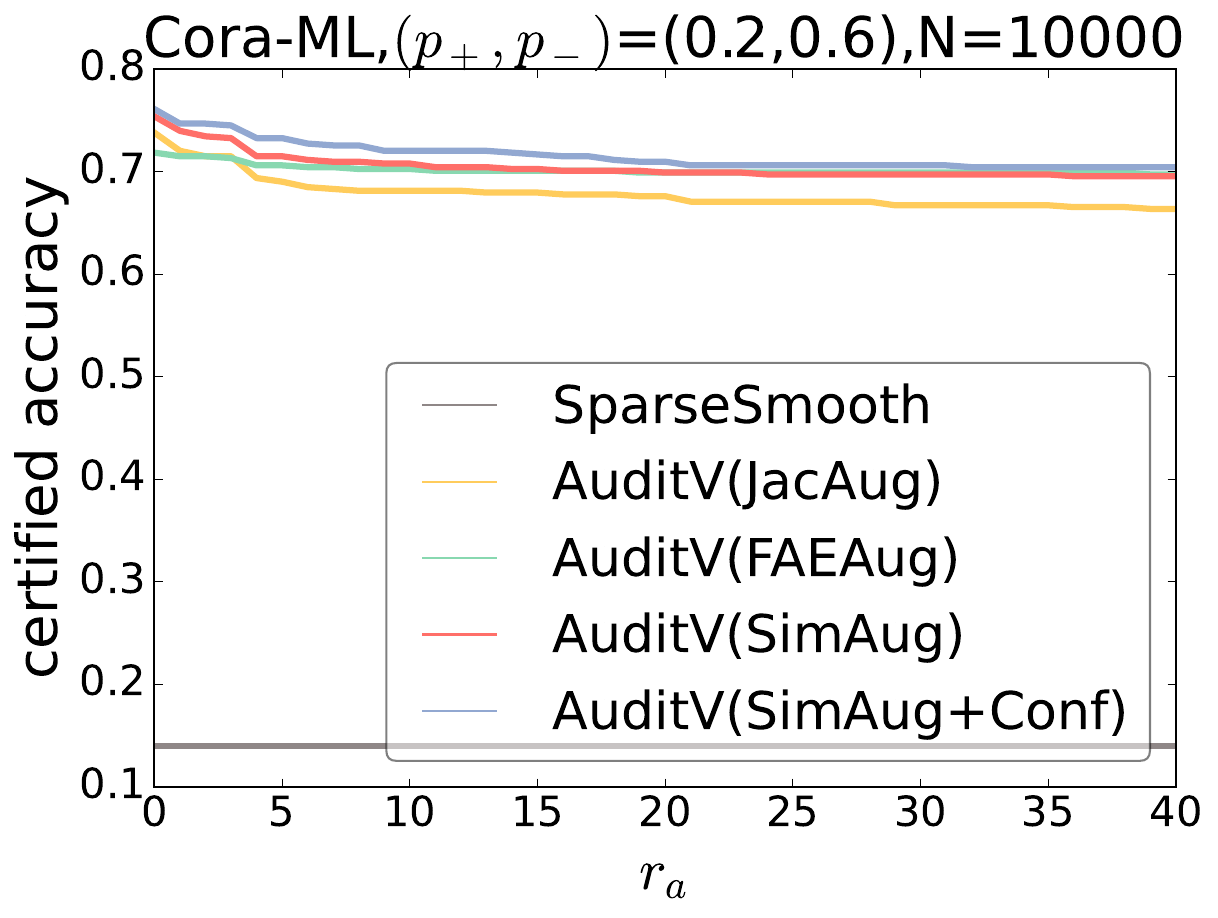}}
\subfigure[Cora-ML]{\includegraphics[width=0.160\textwidth,height=2.5cm]{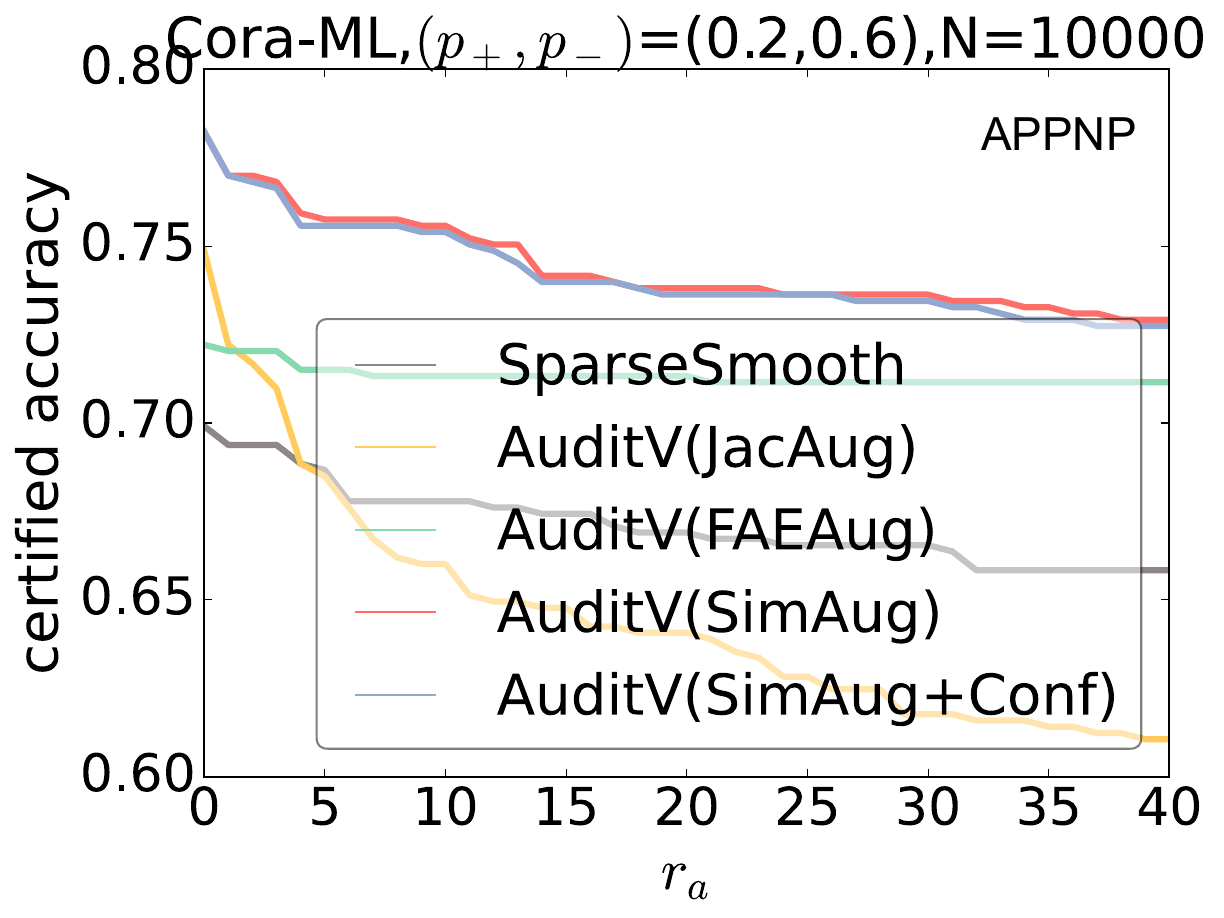}}
\subfigure[Cora-ML]{\includegraphics[width=0.160\textwidth,height=2.5cm]{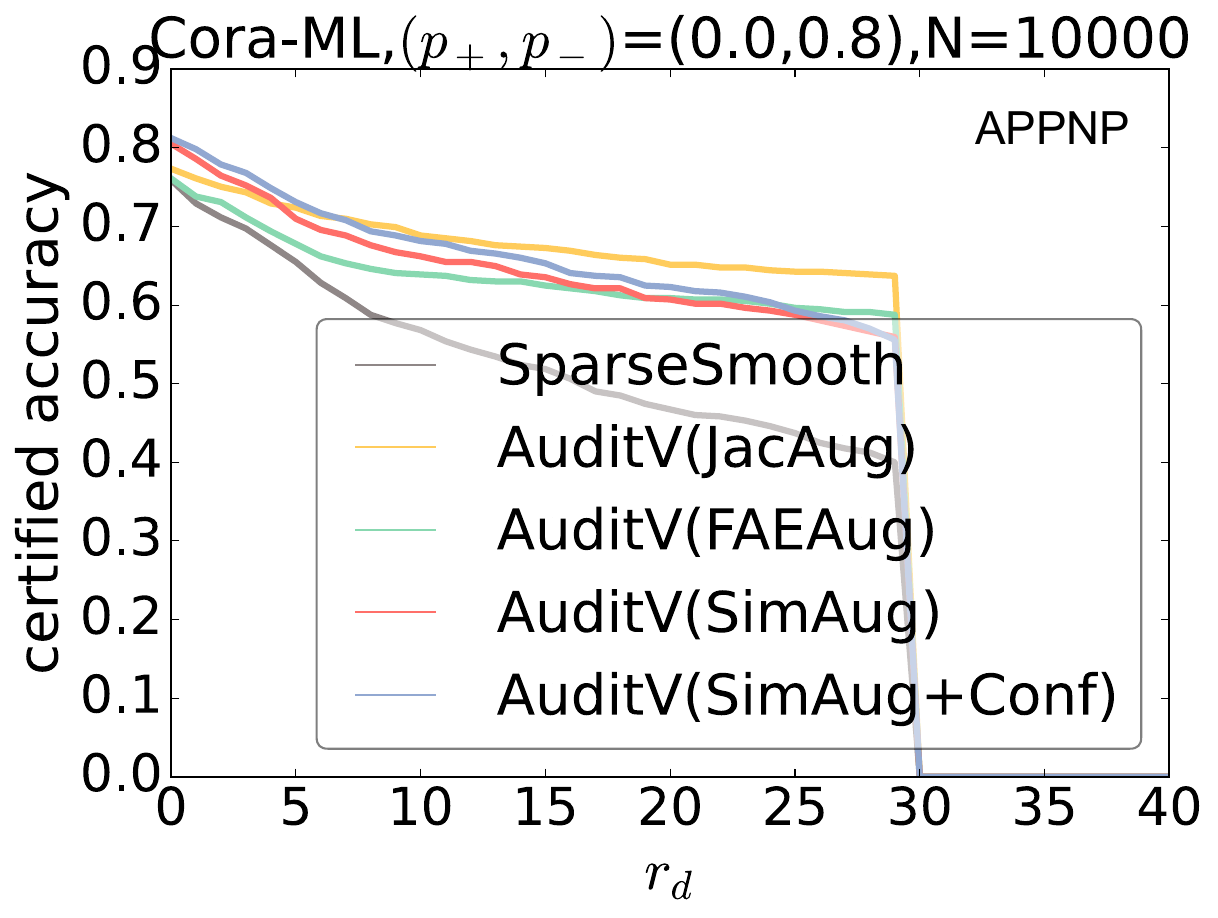}}
\vspace{-8pt}
\caption{Certified accuracy and confidence score distribution. The correctly classified nodes tend to have high confidence scores. By filtering the low-confidence predictions, conditional smoothing improves the certified accuracy.}
\label{fig:analyze_and_certify}
\end{figure*}

\subsection{Conditional Smoothing via Confidence Filtering}
\label{sec:filterfunc}


Building on the theoretical guarantee established in Sec.~\ref{sec:condition} that an arbitrary filtering function that takes only a randomized graph preserves certified robustness, we now present a simple yet highly effective instantiation of conditional smoothing using the prediction confidence of the base classifier. This approach improves certified accuracy with negligible computational overhead, effectively serving as a “free lunch” in the robustness–accuracy trade-off.

We observe that predictions made with high confidence are significantly more likely to be correct. Figure~\ref{fig:analyze_and_certify} (b,e,h) visualizes this phenomenon: correctly classified nodes concentrate in high-confidence regions, whereas errors predominantly occur in low-confidence regions. Leveraging this insight, we filter out low-confidence votes to increase the consistency of the smoothed classifier.


Formally, for each randomized graph $\phi(\mathcal{G})$
 the base classifier $f$ outputs a predicted class along with a confidence score $conf(\phi(\mathcal{G}))$. We define the filtering function as:
\begin{equation}
\label{eqn:h}
h(\phi(\mathcal{G}))=\left\{
\begin{aligned}
0 & \quad \text{if} \, \,conf(\phi(\mathcal{G}))>\theta,\\
1 & \quad \text{otherwise.} \\
\end{aligned}
\right.
\end{equation}
where $\theta$ is a tunable threshold. Only samples with confidence exceeding $\theta$ contribute to the final majority vote. Despite its simplicity, this confidence-based filtering delivers substantial gains in certified accuracy across all datasets, directly because it amplifies the signal from high-consistency predictions.

Notably, this method introduces almost no computational or training overhead. The confidence scores are already produced by the base classifier, and the filter requires only a scalar comparison per sample. This makes conditional smoothing via confidence filtering efficient, simple to deploy, and readily compatible with any base classifier, further broadening the practical applicability of AuditVotes.
We also discuss other potential filtering functions in Appendix~\ref{sec:other_filter}.

\section{Extending AuditVotes to Other Smoothing Schemes}
\label{sec:applicability}
The modular design of AuditVotes enables seamless adaptation to other certified smoothing frameworks beyond sparse randomized smoothing. In this section, we demonstrate its versatility by applying AuditVotes to (1) de‑randomized smoothing for graphs (GNNCert~\cite{xia2024gnncert}, AGNNCert~\cite{li2025agnncert}) and (2) Gaussian randomized smoothing ~\cite{cohen2019certified} for image classification. These extensions show that AuditVotes is not a specialized solution, but a general framework for enhancing both robustness and accuracy across domains and certification paradigms.


\subsection{Application to De‑Randomized Graph Smoothing }

De‑randomized smoothing offers a deterministic alternative to randomized methods for certifying GNNs. Representative frameworks like (A)GNNCert~\cite{xia2024gnncert,li2025agnncert} partition a graph’s edges into $T_s$ disjoint groups via a hash function, producing $T_s$ subgraphs. Predictions from a base classifier on each subgraph are aggregated via majority voting to produce a robust final prediction. However, this edge‑partitioning inherently sparsifies each subgraph, discarding structural information and often reducing both clean and certified accuracy. This problem also exists for node‑partitioning in \cite{li2025agnncert}.
 
The graph augmentation scheme in AuditVotes mitigates the loss of graph information caused by this division. Specifically, after dividing the graph into subgraphs and before classification, we apply augmentation to the subgraphs. Similar to the randomized smoothing framework, the guarantees provided by de-randomized smoothing are compatible with arbitrary base classifiers. Therefore, we can compose a new and composite base classifier as $f(\mathcal{A}(\cdot))$, where $\mathcal{A}(\cdot)$ represents graph augmentation. Then, the (A)GNNCert certificate remains directly applicable to the augmented model without modification.

\paragraph{Implementation with noise‑adaptive thresholds.} Unlike randomized smoothing, GNNCert does not randomly add edges. Instead, the edge partitioning resembles edge deletion in SparseSmooth~\cite{bojchevski2020efficient} with noise level controlled by $T_s$. Thus, we set the edge-pruning threshold $\tau = 0$ and focus on determining the edge-addition threshold $\xi$. Following the method in Section~\ref{sec:threshold}, we estimate the total number of edges in the original graph as $E' = e_{ratio} \times |\mathbf{V}_{test}|^2$. Given $T_s$ subgraph divisions, the expected number of edges to add is $ADD = \lfloor E' \times (1 - (1 / T_s))\rfloor$. Finally, we set the edge-addition threshold $\xi$ based on the top-$ADD$ values.

\subsection{Generalization to Image Classification via Gaussian Smoothing}

AuditVotes is not limited to graph-structured data; its conditional smoothing component naturally extends to continuous input domains such as images. We demonstrate this by applying AuditVotes to Gaussian randomized smoothing~\cite{cohen2019certified}, the most common certification approach for image classification.


\subsubsection{Gaussian smoothing for robust image classification} Given an image classifier $f(\cdot)$ that predicts the label of an image among $\mathcal{Y}=\{1,2,\cdots, C\}$, Gaussian smoothing~\cite{cohen2019certified} defines a smoothed classifier by adding isotropic Gaussian noise $\epsilon\sim \mathcal{N}(0,\sigma^2 I)$ to the input image:
\begin{align}
    g(x)=\argmax_{y\in \mathcal{Y}} \mathbb{P}(f(x+\epsilon)=y).
\end{align}
It was proved that the smoothed model is certifiably robust within $l_2$-norm perturbation $\{x':||x'-x||_2\leq R\}$ with radius $R=\frac{\sigma}{2}(\Phi^{-1}(\underline{p_A})-\Phi^{-1}(\overline{p_B}))$, where $\Phi^{-1}$ is the inverse of Gaussian cumulative distribution function, $\underline{p_A}$ is the lower bound of the top-class probability, and $\overline{p_B}$ is the upper bound of the runner-up class probability. 

\subsubsection{Applying conditional smoothing} We integrate the same confidence based filtering strategy from Section \ref{sec:filterfunc}. Specifically, we define the conditional smoothed classifier  as:
\begin{align}
\label{eqn:consmooth_image}
    g^c(x)=\argmax_{y\in \mathcal{Y}} \mathbb{P}(f(x+\epsilon)=y|h(x+\epsilon)=0),
\end{align}
where $\epsilon\sim \mathcal{N}(0,\sigma^2 I)$, and $h(\cdot)$ is the filtering function that takes only a randomized image as input and outputs either $0$ or $1$, where $0$ indicates that the prediction is included in the voting process and $1$ excludes it.
Notably, this framework involves Gaussian noise in continuous space while SparseSmooth~\cite{bojchevski2020efficient} employs discrete and binary noise distribution. 
Consequently, adapting the certification guarantee to this continuous, filtered setting requires a non-trivial extension of the original Gaussian certificate.
We resolve this by applying the Neyman-Pearson Lemma~\cite{Neyman1992} to avoid the direct estimation of $h(\cdot)$. Given an image $x$, let $y_A$ and $y_B$ denote the top predicted class and the runner-up class, respectively. Let $\underline{p_A}$ and $\overline{p_B}$ denote the lower bound of $\mathbb{P}(f(x+\epsilon)=y_A|h(x+\epsilon)=0)$ and the upper bound of $\mathbb{P}(f(x+\epsilon)=y_B|h(x+\epsilon)=0)$, respectively. Next, we establish the condition for the conditional smoothed classifier to be certifiably robust, which is a non-trivial adaptation from the existing certification method~\cite{cohen2019certified}. 
\begin{theorem}
\label{theorem-condition_image}
    Let the conditional smoothed classifier $g^c(\cdot)$ be as defined in Eq.~\eqref{eqn:consmooth_image}. $\forall x'\in \{x': ||x'-x||_2\leq R\}$, we have $g(x')=g(x)$, where the robust radius is defined as:   
    \begin{align}
    R=\frac{\sigma}{2}(\Phi^{-1}(\underline{p_A})-\Phi^{-1}(\overline{p_B})).
    \end{align}
\end{theorem}
The proof of Theorem~\ref{theorem-condition_image} is provided in Appendix~\ref{AppendixA.2}. Theorem~\ref{theorem-condition_image} ensures that conditional smoothing provably preserves with any $h(\cdot)$ that takes only a randomized image as input. In practice, it enhances the certified accuracy of Gaussian-smoothed image classifiers when $h(\cdot)$ efficiently improves prediction consistency while filtering out only a small number of samples.

\section{Experimental Evaluations}
\label{sec:evaluation}
In this section, we comprehensively evaluate our proposed certifiably robust framework AuditVotes from five aspects: clean accuracy, certified accuracy, empirical accuracy, applicability, and efficiency.
In summary, we present our experimental results by answering the following research questions:

\begin{itemize}
    \item Q1 [\textbf{Certified Accuracy}]:  How effectively can AuditVotes improve certified accuracy? Why it works?
    \item Q2 [\textbf{Trade-off}]:  How effective is AuditVotes in mitigating the trade-off between certified accuracy (robustness) and clean accuracy?
    \item Q3 [\textbf{Empirical Accuracy}]:  How does AuditVotes perform in defending against empirical adversarial attacks? 
    \item Q4 [\textbf{Applicability}]:  How applicable are AuditVotes to other smoothing frameworks? 
    \item Q5 [\textbf{Efficiency}]:  How efficient is AuditVotes compared to vanilla randomized smoothing and other advanced models? Can AuditVotes scale to a large graph? 
\end{itemize}

\subsection{Experimental Setup} 
We describe the evaluation environment, including datasets, models, baselines, parameters, and attack settings.
\subsubsection{\textbf{Datasets}} 
We evaluate AuditVotes on several benchmark datasets, including graph datasets (Citeseer, Cora-ML, PubMed, Amazon2M~\cite{chiang2019cluster}), and image datasets (MNIST, CIFAR-10). For graph datasets, the statistics are shown in Table~\ref{tab:dataset1}. We adopt an inductive semi-supervised graph learning similar to~\cite{gosch2024adversarial}. Specifically, we sample $50$ nodes per class for labeled training nodes ($\mathbf{V}_{ltrain}$) and validation nodes ($\mathbf{V}_{val}$). Then, in each class, we sample $20\%$ of nodes to form testing nodes ($\mathbf{V}_{test}$) in the Citeseer and Cora-ML datasets. For the PubMed dataset, we sample $60\%$ of nodes as testing nodes. The remaining nodes are used as unlabeled training nodes ($\mathbf{V}_{utrain}$). The training graph $\mathcal{G}_{train}$ consists of the labeled training nodes and unlabeled training nodes ($\mathbf{V}_{ltrain}+\mathbf{V}_{utrain}$), the validation graph involved the nodes in $\mathcal{G}_{train}$ and validation nodes ($\mathbf{V}_{ltrain}+\mathbf{V}_{utrain}+\mathbf{V}_{val}$), and the testing graph $\mathcal{G}_{test}$ involves all the nodes ($\mathbf{V}_{ltrain}+\mathbf{V}_{utrain}+\mathbf{V}_{val}+\mathbf{V}_{test}$). For the Amazon2M dataset, we follow exactly the same setting as ~\cite{li2025agnncert}. 
\begin{table}[!h]
\centering
\caption{Statistics of graph datasets.}
\vspace{-8pt}
\begin{tabular}{lrrrrr}
\hline
Datasets       & Nodes & Edges & Classes & Dimension \\
\hline
Cora-ML      & 2,810   & 7,981   & 7        & 2,879      \\
Citeseer  & 2,110   & 3,757   & 6        & 3,703      \\
PubMed    & 19,717   & 44,338   & 3        & 500        \\
Amazon2M    & 2,449,029   & 61,859,140   & 47        & 100        \\
\hline
\end{tabular}
\label{tab:dataset1}
\end{table}

For image datasets (MNIST and CIFAR-10), we adopt the dataset configurations as detailed in Table~\ref{tab:dataset2}.
\begin{table}[!ht]
\centering
\caption{Statistics of image datasets.}
\vspace{-8pt}
\begin{tabular}{lrrrrr}
\hline
Datasets       & Training set & Testing set & Classes & Dimension \\
\hline
MNIST        & 60k   & 10k   & 10      & $28\times28\times1$      \\
CIFAR-10      & 50k   & 10k  &10      &  $32\times32\times3$    \\
\hline
\end{tabular}
\label{tab:dataset2}
\vspace{-8pt}
\end{table}

\subsubsection{\textbf{Models and Baselines}}

\paragraph{\textbf{Node Classification Task}}For node classification, we use SparseSmooth~\cite{bojchevski2020efficient} and GNNCert~\cite{xia2024gnncert} (or AGNNCert~\cite{li2025agnncert}) as the baselines and compare with various AuditVotes configurations applied to them. The evaluated models include:

\begin{itemize}
    \item SparseSmooth~\cite{bojchevski2020efficient}: Sparsity-aware randomized smoothing model for graph.
    \item GNNCert~\cite{xia2024gnncert} (or AGNNCert~\cite{li2025agnncert}): De-randomized smoothing model (based on graph division) for graph. AGNNCert~\cite{li2025agnncert} extends GNNCert~\cite{li2025agnncert} to defend against more kinds of attacks, such as node deletion/insertion. 
    \item AuditVotes: Our AuditVotes serve as a general plug-in, enhancing an existing smoothing scheme (SparseSmooth~\cite{bojchevski2020efficient} or (A)GNNCert~\cite{xia2024gnncert,li2025agnncert}) with specified components: augmentation and conditional filtering. For example, AuditVotes (SimAug+Conf) denotes using AuditVotes to enhance the existing method with \texttt{SimAug} augmentation and \texttt{Conf} filter. Unless otherwise specified, we apply AuditVotes to \cite{bojchevski2020efficient}. We also present a cross-comparison between \cite{bojchevski2020efficient}+AuditVotes and \cite{xia2024gnncert}+AuditVotes in Table~\ref{tab:cross_ra} and ~\ref{tab:cross_rd} (Appendix~\ref{sec:appenx_cross}).
\end{itemize}
By default, we use Graph Convolutional Network (GCN)~\cite{kipf2016semi} as the base classifier. Following~\cite{bojchevski2020efficient}, \textit{we also evaluate APPNP~\cite{gasteiger2019appnp} as the base classifier} (presented in Figure~\ref{fig:analyze_and_certify} (k,l), and Appendix~\ref{appendix:more_results}). 

\paragraph{\textbf{Image Classification Task}} For the image classification task, we apply our proposed conditional smoothing (\textbf{Conf}) to Gaussian~\cite{cohen2019certified} and compare to four baselines: Gaussian~\cite{cohen2019certified}, Stability~\cite{li2019certified}, CAT-RS~\cite{jeong2023confidence}, and 
Diffusion~\cite{carlini2023certified} (Detailed in Appendix~\ref{appendix:baselines}). 

\paragraph{\textbf{Adversarial Attack Defense}} Moreover, the smoothed models can also serve as empirical robust models defending against actual adversarial attacks. We also evaluate the effectiveness of our proposed AuditVotes as an empirical robust model under Nettack~\cite{zugner2018adversarial} and IG-attack~\cite{wu2019adversarial} (widely used structure attacks for graph data). We compare our models with regular (non-smoothed) robust GNNs: GCN~\cite{kipf2016semi}, GAT~\cite{veličković2018graph}, MedianGCN~\cite{chen2021understanding}, and AirGNN~\cite{liu2021graph} (Detailed in Appendix~\ref{appendix:baselines}).

\subsubsection{\textbf{Model settings and hyper-parameters}} 

For the node classifications, we employ a 2-layer GCN with the hidden layer dimension of size $128$. We use learning rate $lr=0.001$, regularization coefficient $\lambda=1\times 10^{-3}$ for training. Early stop with patience $100$ epochs and maximum $1000$ epochs is employed to control the training epochs. 
For the CIFAR-10 dataset, we employ ResNet-110~\cite{he2016deep} model with learning rate $lr=0.01$, regularization coefficient $\lambda=1\times 10^{-4}$, batch size $256$, and $150$ training epochs. 
For the MNIST dataset, we employ LeNet~\cite{lecun1998gradient} model with learning rate $lr=0.01$, Adam optimizer, regularization coefficient $\lambda=1\times 10^{-4}$, batch size $256$, and $150$ training epochs.    
For all the randomized smoothing models, we set $N=10,000$ and $\alpha=0.001$ for Monte Carlo probability approximation in~\cite{bojchevski2020efficient,cohen2019certified}. In node classification, we search the noise levels among $p_+=\{0.0,0.1\}$, $p_-=\{0.4, 0.6,0.7,0.8\}$ for certifying edge deletion, and $p_+=\{0.2,0.4\}$, $p_-=\{0.0, 0.2,0.4,0.6\}$ for certifying edge addition. For GNNCert~\cite{xia2024gnncert}, we employ the MD5 hash function to divide the edges with $T_s=15$ for Citeseer and $T_s=20$ for Cora-ML because the Citeseer contains fewer edges. For AGNNCert~\cite{li2025agnncert}, we set $T=100$ for the Amazon2M dataset, and we follow the same model setting as in \cite{li2025agnncert}. 
We set the noise level $\sigma=0.25$ for image classification tasks. 
For augmentation training, we sample $90\%$ of existing edges in the training subgraph as the positive edges, and sample non-edges in a quantity $10$ times that of the positive edges as negative edges. Then, we train the augmentation models with a learning rate $lr=0.001$ using Adam optimizer for $250$ epochs.
In the conditional smoothing model, we set the confidence threshold as $\theta=0.5$ for certifying edge deletion, $\theta=0.2$ for certifying edge addition, and $\theta=0.9$ for certifying image classification.
In the graph structure attacks (Nettack~\cite{zugner2018adversarial} and IG-attack~\cite{wu2019adversarial}), we select $30$ target nodes following~\cite{zugner2018adversarial} and evaluate the accuracy among these target nodes. We set the attacker budget as $\{1,2,3,4,5\}$ edges per target node. 
For all the baselines, we employed the recommended parameters in their papers.
\textbf{More implementation details are presented in Appendix~\ref{Sec:implement_detail}}.



\subsection{Improving the Certified Accuracy (Q1)}
In this section, we answer question Q1 by demonstrating the more advanced certified accuracy contributes to the AuditVotes. We evaluate its performance in defending against edge deletion perturbation (Table~\ref{tab:certify_rd} and  Figure~\ref{fig:analyze_and_certify} (a,d)) and edge addition perturbation (Table~\ref{tab:certify_ra} and Figure~\ref{fig:analyze_and_certify} (g,j)). 
To analyze the contribution of individual components in AuditVotes, we conduct ablation studies by applying augmentation and the confidence filter separately. Finally, we present the overall impact of the full AuditVotes framework. 

\begin{table}[!ht]
\centering
\caption{Certified accuracy ($p_+=0.0$, $p_-=0.8$). 
We apply AuditVotes to SparseSmooth, and \texttt{J/F/S} stands for the augmentation scheme JacAug/FAEAug/SimAug.}
\vspace{-8pt}
\setlength{\tabcolsep}{1.0pt}
\begin{tabular}{ccccccc}
\toprule[0.9pt]
\multirow{2}{*}{Datasets} & \multirow{2}{*}{\begin{tabular}[c]{@{}c@{}}Model + \\ augmentor\end{tabular}} & \multirow{2}{*}{\begin{tabular}[c]{@{}c@{}}Conditional \\ smoothing\end{tabular}} & \multicolumn{4}{c}{Certified accuracy   ($r_d$)} \\ \cline{4-7} 
 &  &  & Clean & 5 & 10 & 20 \\ \toprule[0.9pt]
\multirow{8}{*}{Citeseer} & \multirow{2}{*}{SparseSmooth\cite{bojchevski2020efficient}} & \cellcolor{gray!10}None & \cellcolor{gray!10}0.700 & \cellcolor{gray!10}0.596 & \cellcolor{gray!10}0.550 & \cellcolor{gray!10}0.517 \\
 &  & +Conf & \textbf{0.748} & {\ul 0.683} & {\ul 0.656} & {\ul 0.613} \\ \cline{2-7} 
 & \multirow{2}{*}{+AuditVotes (J)} & None & 0.695 & 0.620 & 0.599 & 0.582 \\
 &  & +Conf & 0.663 & 0.630 & 0.623 & 0.589 \\ \cline{2-7} 
 & \multirow{2}{*}{+AuditVotes (F)} & None & 0.680 & 0.620 & 0.582 & 0.558 \\
 &  & +Conf & 0.675 & 0.644 & 0.630 & 0.589 \\ \cline{2-7} 
 & \multirow{2}{*}{+AuditVotes (S)} & None & {\ul 0.738} & 0.671 & 0.632 & 0.608 \\
 &  & +Conf & 0.733 & \textbf{0.702} & \textbf{0.683} & \textbf{0.659} \\ \hline
\multirow{8}{*}{Cora-ML} & \multirow{2}{*}{SparseSmooth\cite{bojchevski2020efficient}} & \cellcolor{gray!10}None & \cellcolor{gray!10}0.781 & \cellcolor{gray!10}0.651 & \cellcolor{gray!10}0.549 & \cellcolor{gray!10}0.488 \\
 &  & +Conf & {\ul 0.804} & 0.745 & 0.665 & 0.573 \\ \cline{2-7} 
 & \multirow{2}{*}{+AuditVotes (J)} & None & 0.772 & 0.710 & 0.676 & 0.637 \\
 &  & +Conf & 0.781 & \textbf{0.749} & \textbf{0.735} & \textbf{0.690} \\ \cline{2-7} 
 & \multirow{2}{*}{+AuditVotes (F)} & None & 0.758 & 0.671 & 0.623 & 0.593 \\
 &  & +Conf & 0.770 & 0.726 & 0.687 & 0.653 \\ \cline{2-7} 
 & \multirow{2}{*}{+AuditVotes (S)} & None & 0.793 & 0.704 & 0.660 & 0.614 \\
 &  & +Conf & \textbf{0.811} & {\ul 0.745} & {\ul 0.712} & {\ul 0.667} \\ \hline
\multirow{8}{*}{PubMed} & \multirow{2}{*}{SparseSmooth\cite{bojchevski2020efficient}} & \cellcolor{gray!10}None & \cellcolor{gray!10}0.788 & \cellcolor{gray!10}0.712 & \cellcolor{gray!10}0.656 & \cellcolor{gray!10}0.587 \\
 &  & +Conf & 0.794 & 0.737 & 0.686 & 0.622 \\ \cline{2-7} 
 & \multirow{2}{*}{+AuditVotes (J)} & None & 0.789 & 0.723 & 0.680 & 0.626 \\
 &  & +Conf & {\ul 0.795} & 0.746 & {\ul 0.711} & 0.657 \\ \cline{2-7} 
 & \multirow{2}{*}{+AuditVotes (F)} & None & 0.790 & 0.739 & 0.700 & 0.655 \\
 &  & +Conf & 0.789 & 0.745 & 0.709 & \textbf{0.669} \\ \cline{2-7} 
 & \multirow{2}{*}{+AuditVotes (S)} & None & \textbf{0.800} & \textbf{0.786} & \textbf{0.761} & {\ul 0.662} \\
 &  & +Conf & \textbf{0.800} & {\ul 0.783} & \textbf{0.761} & \textbf{0.669} \\ \bottomrule[0.9pt]
\end{tabular}
\label{tab:certify_rd}
\end{table}

\begin{table}[!ht]
\centering
\caption{Certified accuracy ($p_+=0.2$, $p_-=0.6$).}
\vspace{-8pt}
\setlength{\tabcolsep}{0.5pt}
\begin{tabular}{ccccccc}
\toprule[0.9pt]
\multirow{2}{*}{Datasets} & \multirow{2}{*}{\begin{tabular}[c]{@{}c@{}}Model +\\  augmentor\end{tabular}} & \multirow{2}{*}{\begin{tabular}[c]{@{}c@{}}Conditional \\ smoothing\end{tabular}} & \multicolumn{4}{c}{Certified accuracy   ($r_a$)} \\ \cline{4-7}
 &  &  & \multicolumn{1}{c}{Clean} & \multicolumn{1}{c}{5} & \multicolumn{1}{c}{10} & \multicolumn{1}{c}{20} \\ \toprule[0.9pt]
\multirow{8}{*}{Citeseer} & \multirow{2}{*}{SparseSmooth\cite{bojchevski2020efficient}} & \cellcolor{gray!10}None & \cellcolor{gray!10}0.147 & \cellcolor{gray!10}0.147 & \cellcolor{gray!10}0.147 & \cellcolor{gray!10}0.147 \\
 &  & +Conf & 0.147 & 0.147 & 0.147 & 0.147 \\ \cline{2-7} 
 & \multirow{2}{*}{+AuditVotes (J)} & None & 0.704 & 0.596 & 0.522 & 0.486 \\
 &  & +Conf & 0.724 & 0.560 & 0.486 & 0.389 \\ \cline{2-7} 
 & \multirow{2}{*}{+AuditVotes (F)} & None & 0.690 & 0.683 & 0.680 & 0.675 \\
 &  & +Conf & 0.697 & 0.690 & 0.688 & 0.685 \\ \cline{2-7} 
 & \multirow{2}{*}{+AuditVotes (S)} & None & {\ul 0.709} & {\ul 0.709} & {\ul 0.709} & {\ul 0.709} \\
 &  & +Conf & \textbf{0.726} & \textbf{0.726} & \textbf{0.726} & \textbf{0.726} \\ \hline
\multirow{8}{*}{Cora-ML} & \multirow{2}{*}{SparseSmooth\cite{bojchevski2020efficient}} & \cellcolor{gray!10}None & \cellcolor{gray!10}0.140 & \cellcolor{gray!10}0.140 & \cellcolor{gray!10}0.140 & \cellcolor{gray!10}0.140 \\
 &  & +Conf & 0.140 & 0.140 & 0.140 & 0.140 \\ \cline{2-7} 
 & \multirow{2}{*}{+AuditVotes (J)} & None & 0.738 & 0.690 & 0.681 & 0.674 \\
 &  & +Conf & 0.738 & 0.688 & 0.680 & 0.674 \\ \cline{2-7} 
 & \multirow{2}{*}{+AuditVotes (F)} & None & 0.719 & 0.706 & 0.703 & 0.701 \\
 &  & +Conf & 0.720 & 0.706 & 0.703 & 0.699 \\ \cline{2-7} 
 & \multirow{2}{*}{+AuditVotes (S)} & None & {\ul 0.750} & \textbf{0.727} & {\ul 0.719} & {\ul 0.712} \\
 &  & +Conf & \textbf{0.752} & {\ul 0.726} & \textbf{0.720} & \textbf{0.713} \\ \hline
\multirow{8}{*}{PubMed} & \multirow{2}{*}{SparseSmooth\cite{bojchevski2020efficient}} & \cellcolor{gray!10}None & \cellcolor{gray!10}OOM & \cellcolor{gray!10}OOM & \cellcolor{gray!10}OOM & \cellcolor{gray!10}OOM \\
 &  & +Conf & \multicolumn{1}{r}{OOM} & \multicolumn{1}{r}{OOM} & \multicolumn{1}{r}{OOM} & \multicolumn{1}{r}{OOM} \\ \cline{2-7} 
 & \multirow{2}{*}{+AuditVotes (J)} & None & 0.766 & 0.745 & 0.739 & 0.732 \\
 &  & +Conf & 0.766 & 0.745 & 0.739 & 0.732 \\ \cline{2-7} 
 & \multirow{2}{*}{+AuditVotes (F)} & None & {\ul 0.775} & {\ul 0.767} & {\ul 0.765} & {\ul 0.764} \\
 &  & +Conf & {\ul 0.775} & {\ul 0.767} & {\ul 0.765} & {\ul 0.764} \\ \cline{2-7} 
 & \multirow{2}{*}{+AuditVotes (S)} & None & \textbf{0.800} & \textbf{0.800} & \textbf{0.800} & \textbf{0.800} \\
 &  & +Conf & \textbf{0.800} & \textbf{0.800} & \textbf{0.800} & \textbf{0.800} \\ \bottomrule[0.9pt]
\end{tabular}
\label{tab:certify_ra}
\end{table}

\begin{table}[!h]
\centering
\caption{Statistics of smoothed graphs ($p_+=0.2$, $p_-=0.6$) without/with augmentation.}
\vspace{-8pt}
\setlength{\tabcolsep}{1.5pt}
\begin{tabular}{llrr}
\toprule[0.9pt]
\multicolumn{1}{c}{Datasets} & \multicolumn{1}{c}{Graphs} & \multicolumn{1}{c}{\begin{tabular}[c]{@{}c@{}}Reconstruction AUC\end{tabular}} & \multicolumn{1}{c}{\begin{tabular}[c]{@{}c@{}} Homophily\end{tabular}} \\ \toprule[0.9pt]
\multirow{5}{*}{Citeseer} & Original graph & 1.000 & 0.803 \\
 & Smoothed graph & None & 0.191 \\
 & +JacAug & 0.899 & 0.821 \\
 & +FAEAug & {\ul 0.916} & {\ul 0.872} \\
 & +SimAug & \textbf{0.920} & \textbf{0.895} \\ \hline
\multirow{5}{*}{Cora-ML} & Original graph & 1.000 & 0.848 \\
 & Smoothed graph & None & 0.172 \\
 & +JacAug & 0.817 & 0.792 \\
 & +FAEAug & \textbf{0.921} & {\ul 0.866} \\
 & +SimAug & {\ul 0.836} & \textbf{0.924} \\ \toprule[0.9pt]
\end{tabular}
\label{tab:Stats_aug}
\vspace{-8pt}
\end{table}


\subsubsection{\textbf{Impact of Augmentations}}
Firstly, we begin by evaluating the impact of augmentations on certified accuracy by comparing models without augmentations (SparseSmooth) and those with augmentations (JacAug, FAEAug, and SimAug). 
When certifying edge deletion ($r_d$), all three augmentations improve certified accuracy across all datasets (Table~\ref{tab:certify_rd}). When the attacker deletes 20 edges ($r_d=20$) on the Citeseer dataset, the certified accuracy improves from $51.7\%$ (SparseSmooth) to $58.2\%$ (+JacAug), $55.8\%$ (+FAEAug), and $60.8\%$ (+SimAug). Similar improvements are observed across the Cora-ML and PubMed datasets. 

Moreover, more significant improvements can be observed when certifying $r_a$ with a high noise level of $p_+$ (Table~\ref{tab:certify_ra} and Figure~\ref{fig:analyze_and_certify}).
Notably, the AuditVotes (SimAug) achieves at least $70.0\%$ certified accuracy when $r_a=20$. More specifically, SimAug improves the certified accuracy by $382.3\%$ and $408.6\%$ in the Citeseer and Cora-ML datasets. 
Note that this parameter is essential for certifying against realistic attacks that tend to add edges (Figure~\ref{fig:edge_change}). 
The certified accuracy with vanilla GCN drops below $15.0\%$ due to its low model accuracy. The noise $p_+$ breaks the pattern of the graph, such as the homophily, and the GCN corrupts. Furthermore, the computation memory requirement increases sharply due to the graph density increasing and causes out-of-memory (OOM), especially in the PubMed dataset. In contrast, our augmentations adaptively adjust the graph's sparsity (Figure~\ref{fig:e_sparsity}), significantly mitigating these issues. As shown in Table~\ref{tab:Stats_aug}, our augmentations restore the smoothed graph to the original graph with high reconstruction AUC over $0.80$. Moreover, the augmentations not only recover the homophily but also enhance the homophily. These results underscore the effectiveness of our proposed augmentations in pre-processing the smoothed graphs.

\subsubsection{\textbf{Impact of Confidence Filter}}
Next, we investigate the impact of the conditional smoothing approach by applying the confidence filter (Conf) to SparseSmooth, AuditVotes (JacAug), and AuditVotes (FAEAug) models. By leveraging our confidence filter, we achieve a notable enhancement in certified accuracy, especially in the case of certifying edge deletion. For all four models and in all three datasets, we observe a positive effect of the confidence filter in raising the certified accuracy (Table~\ref{tab:certify_rd}). On the Citeseer dataset, the Conf improves the certified accuracy of SparseSmooth from $51.7\%$ to $61.3\%$. In Figure~\ref{fig:pA_distribution}, we visualize the distribution of $p_A$ in the vanilla smoothing (without Conf) and conditional smoothing with confidence filter (Conf). Higher $p_A$ indicates a higher prediction consistency. We observe that the confidence filter improves the prediction consistency significantly, and this is the main reason for the improvement in certified accuracy. 

For edge addition robustness ($r_a$), the confidence filter also yields slight improvements (Table~\ref{tab:certify_ra}). This is because the prediction consistency is already high under this noise level, and it is hard to be further improved.
Importantly, as shown in Section~\ref{sec:runtime}, these benefits come with minimal additional computational overhead, making the confidence filter an efficient and practical enhancement.

\begin{figure}[!ht]
\centering
\subfigure[Citeseer]{\includegraphics[width=0.235\textwidth,height=2.95cm]{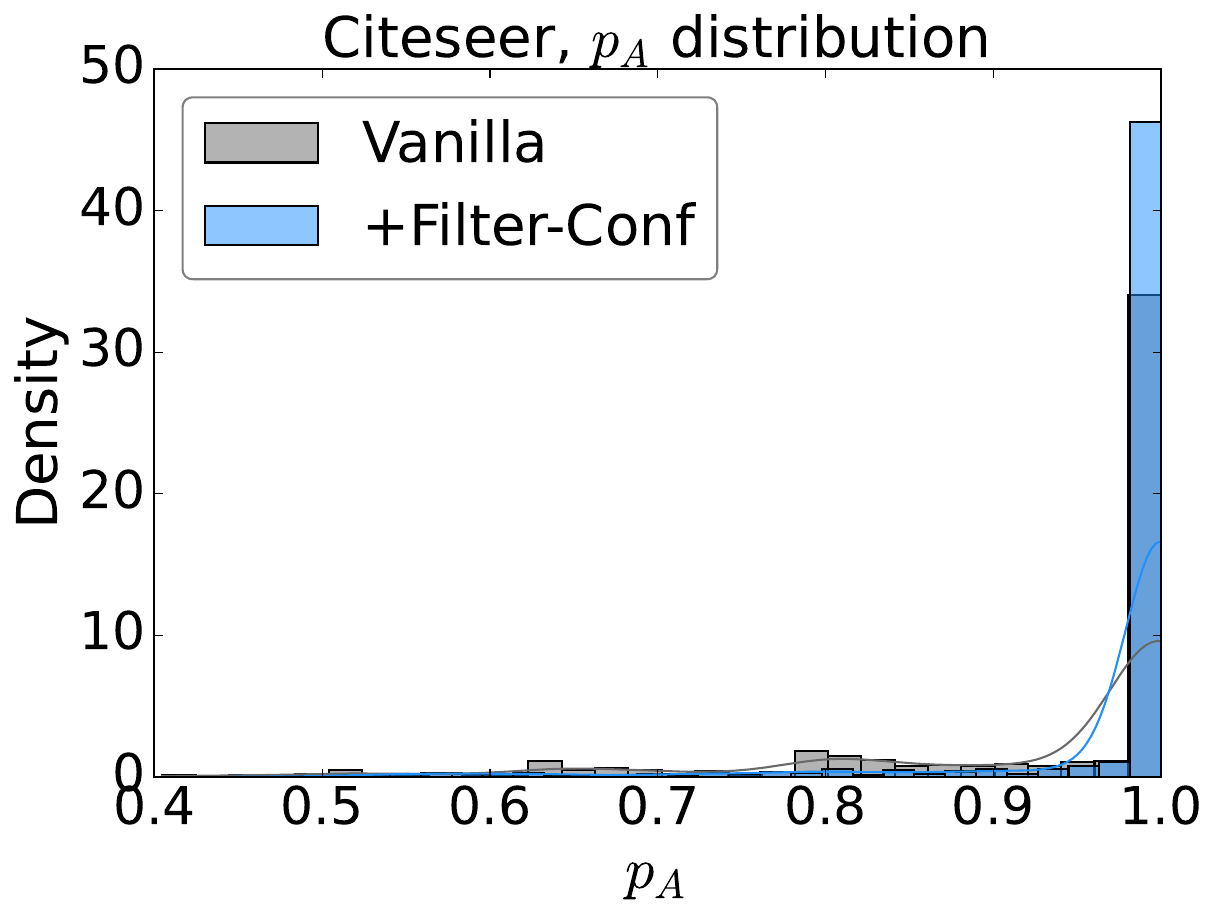}}
\subfigure[Cora-ML]{\includegraphics[width=0.235\textwidth,height=2.95cm]{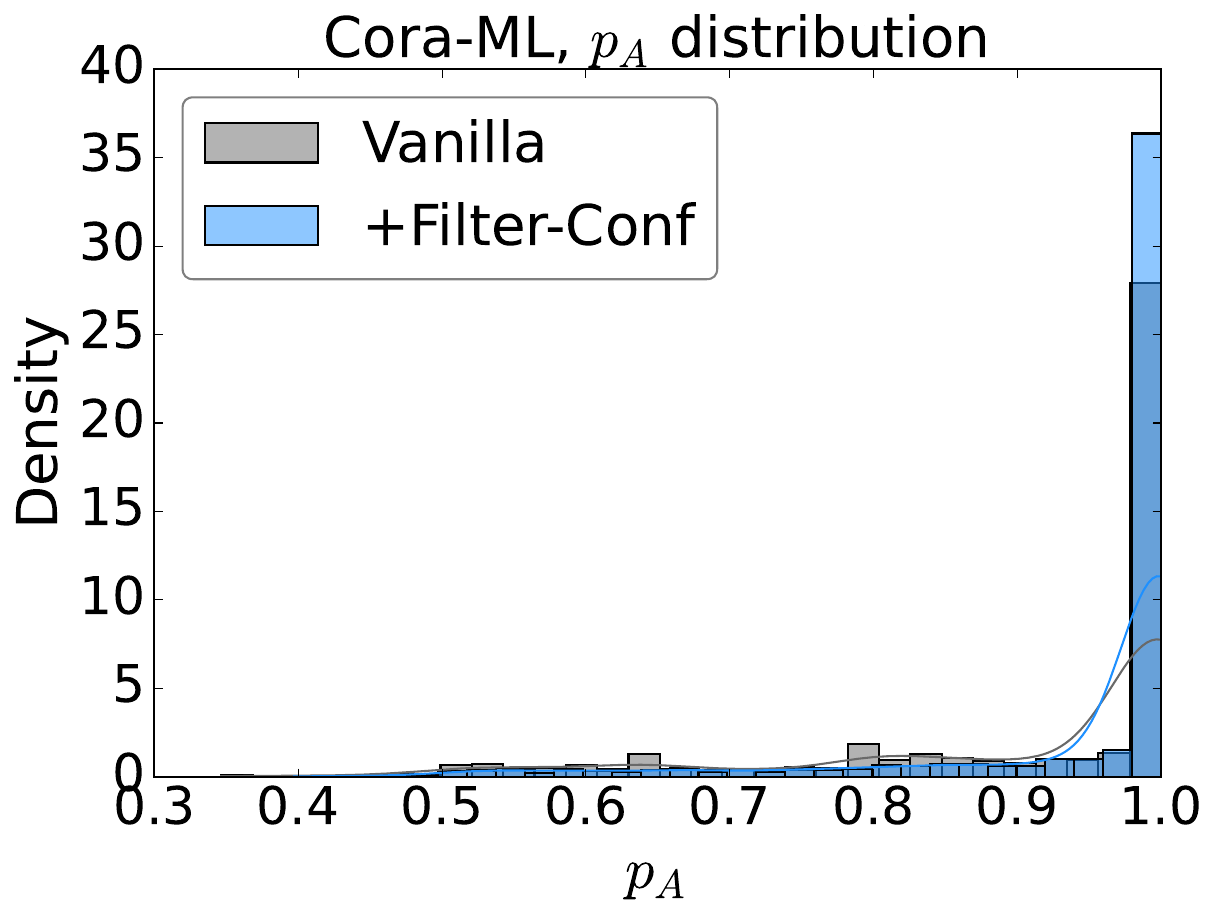}}
\vspace{-8pt}
\caption{Prediction consistency w/wo confidence filter.}
\label{fig:pA_distribution}
\vspace{-8pt}
\end{figure}

\subsubsection{\textbf{Overall Impact of AuditVotes}}
When combining both augmentation and the confidence filter, the full AuditVotes framework achieves the most significant improvement in certified accuracy. 
For edge deletion robustness ($r_d=20$), AuditVotes with SimAug and Conf achieves certified accuracy of $65.9\%$ (Citeseer), $66.7\%$ (Cora-ML), and $66.9\%$ (PubMed), representing improvements of $27.5\%$, $36.7\%$, and $14.0\%$, respectively, compared to SparseSmooth (Table~\ref{tab:certify_rd}).
Moreover, the advantage of AuditVotes is even more significant edge addition robustness (Table~\ref{tab:certify_ra}). When $r_a=20$, AuditVotes (SimAug+Conf) achieves $72.6\%$ (Citeseer), $71.3\%$ (Cora-ML), and $80.0\%$ (PubMed), while the baseline is not working well (with accuracy less than $15.0\%$ or OOM errors). Compared to SparseSmooth, AuditVotes improves certified accuracy by $393.9\%$ and $409.3\%$ on the Citeseer and Cora-ML datasets, respectively. These results highlight the effectiveness of AuditVotes, which integrates augmentation and the confidence filter to achieve the most substantial improvements in certified accuracy.



\subsection{Improving the Trade-off (Q2)}
\begin{figure}[!htb]
\centering
\subfigure[Cora-ML ($r_a$)]{\includegraphics[width=0.185\textwidth,height=2.65cm]{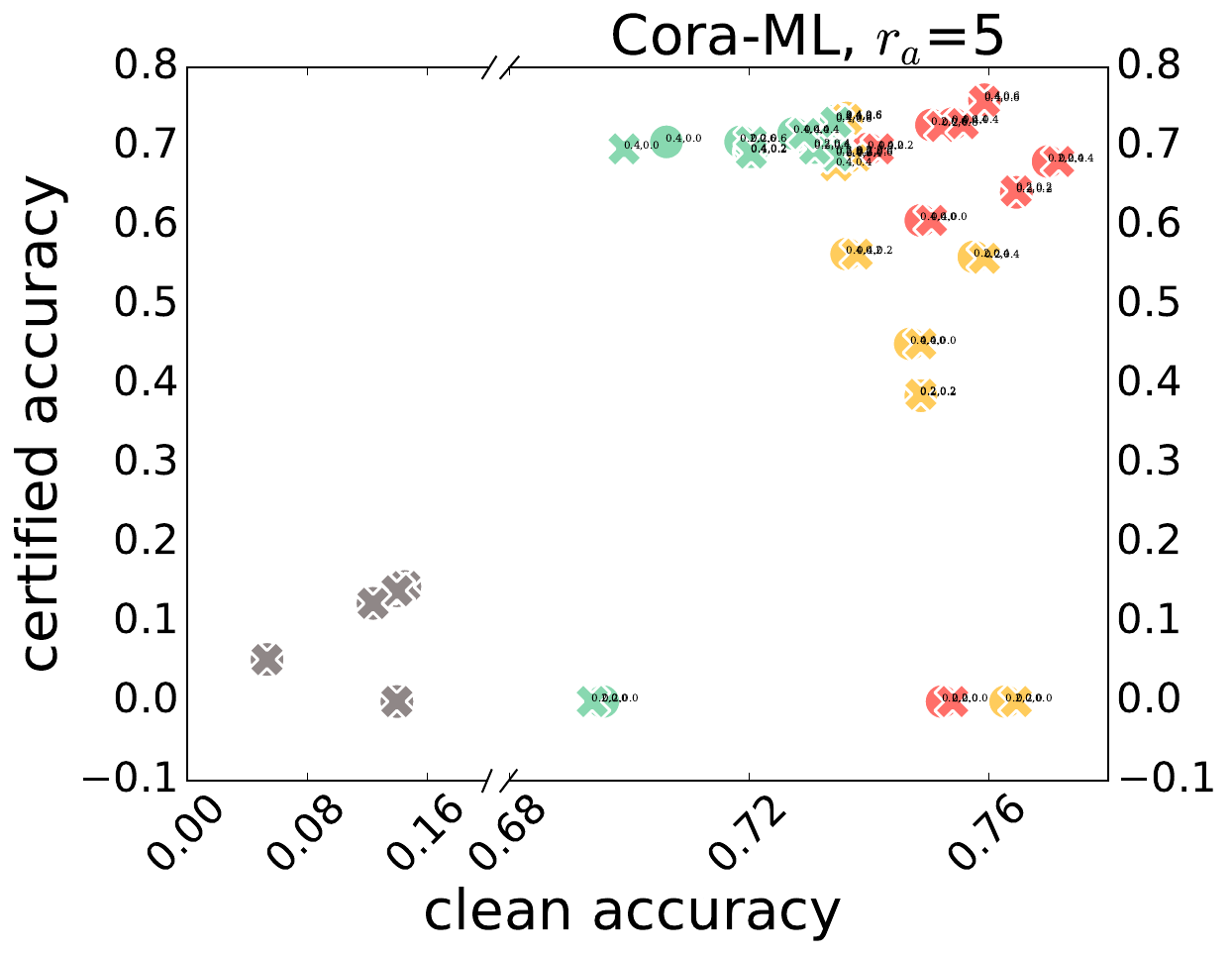}}
\subfigure[Citeseer ($r_a$)]{\includegraphics[width=0.285\textwidth,height=2.65cm]{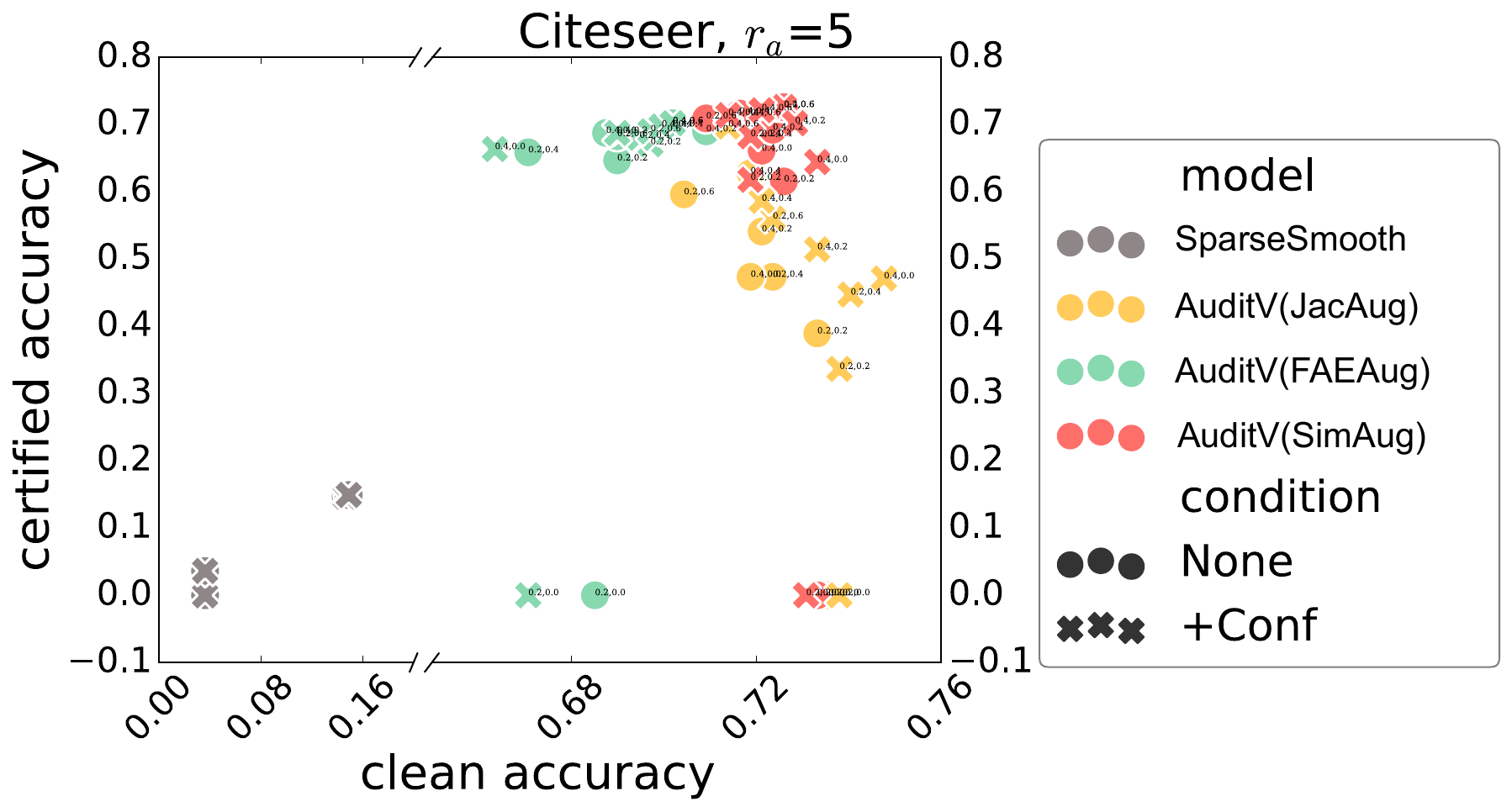}}
\caption{Clean accuracy and certified accuracy trade-off.}
\label{fig:Trade-off_ra}
\end{figure}

Higher noise levels ($p_+$ and $p_-$) in the smoothing distribution generally result in higher certified accuracy, assuming other conditions remain constant. However, increased noise often reduces clean accuracy, thereby diminishing the model's utility. 
Next, we address Q2 by demonstrating how AuditVotes effectively enhances the accuracy-robustness trade-off.

We evaluate clean accuracy and certified accuracy across various noise levels of $p_+$ (Figure~\ref{fig:Trade-off_ra}) and $p_-$ (Figure~\ref{fig:Trade-off_rd} in Appendix). A data point indicates a better trade-off when it is closer to the upper-right corner of the plot, representing both high clean accuracy and high certified accuracy. At identical noise levels, AuditVotes consistently achieves superior clean accuracy and certified accuracy.

When certifying edge addition robustness (Figure~\ref{fig:Trade-off_ra}), our model AuditVotes-SimAug+Conf (red cross) demonstrates the best trade-off on both the Cora-ML and Citeseer datasets. Among the other augmentations, JacAug (yellow) shows better clean accuracy in most cases, whereas FAEAug (green) achieves higher certified accuracy.
When certifying edge deletion robustness (Figure~\ref{fig:Trade-off_rd}), 
all AuditVotes variants consistently outperform SparseSmooth (with vanilla GCN as base classifier, represented by gray dots). This advantage is observed across all three datasets. 
These results highlight the robustness and versatility of AuditVotes in improving the trade-off between accuracy and robustness.


\subsection{More Advanced Empirical Accuracy (Q3)}
It is worth noting that randomized smoothing models can also be used as empirical defense modeling~\cite{lai2024node,weber2023rab}. In this section, we evaluate AuditVotes as a general empirical defense tool defending against actual edge modification evasion attacks. In Table~\ref{tab:empirical_nettack} and Table~\ref{tab:empirical_ig} (in Appendix~\ref{appendix:more_results}), we test the empirical robustness under Nettack and IG-attack with various attack power (i.e., the number of edge modification budgets per target node). The results show that AuditVotes achieves much higher robustness compared to regular robust GNN models. 
For example, under an attack power of 5 edges, the accuracy of AuditVotes remains nearly unchanged, whereas the accuracy of regular models declines sharply. On the 
Citeseer dataset, smoothed GCN with SimAug achieves an accuracy of $83.3\%$ while the accuracy of AirGNN is $18.0\%$ on target nodes. 
Furthermore, AuditVotes maintains clean accuracy comparable to regular models across all the datasets. In contrast, the accuracy of the baseline model is much lower, making it unable to serve as an applicable defense model. For instance, the clean accuracy of the SparseSmooth is $16.7\%$ on Citeseer, which is much worse than the Multilayer Perceptron (MLP) model. These results highlight the ability of AuditVotes to significantly improve both clean and empirical robust accuracy, effectively ``rescuing" smoothed GNNs for practical use as robust defense models.

\subsection{Wide Applicability of AuditVotes (Q4)}
In this section, we demonstrate the generality and wide applicability of AuditVotes. Specifically, we extend AuditVotes to other certifying schemes, including de-randomized smoothing frameworks for GNNs~\cite{xia2024gnncert} and randomized smoothing models for image classification tasks using Gaussian noise~\cite{cohen2019certified}.

\begin{figure}[!htb]
\centering
\subfigure[Citeseer ($r_a/r_d$)]{\includegraphics[width=0.235\textwidth,height=2.95cm]{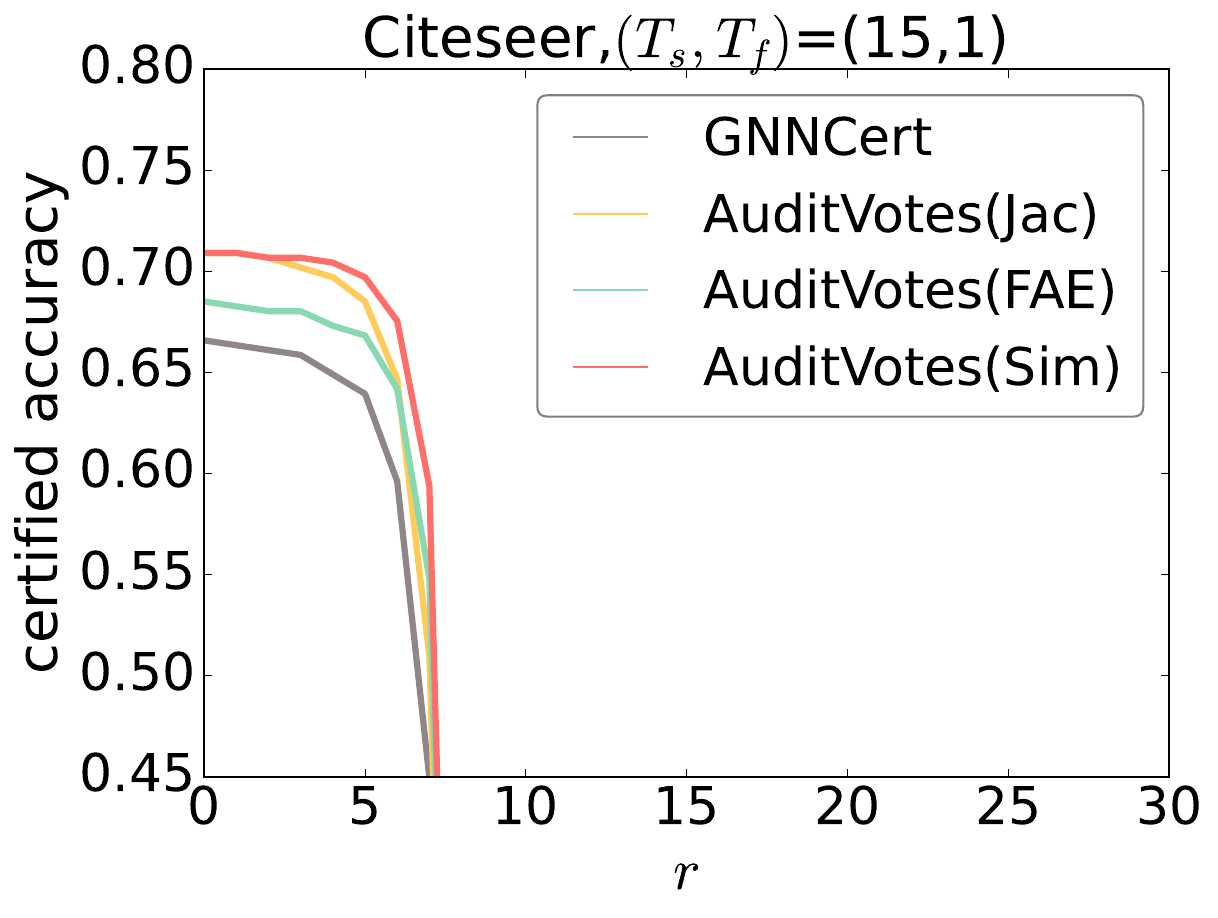}}
\subfigure[Cora-ML ($r_a/r_d$)]{\includegraphics[width=0.235\textwidth,height=2.95cm]{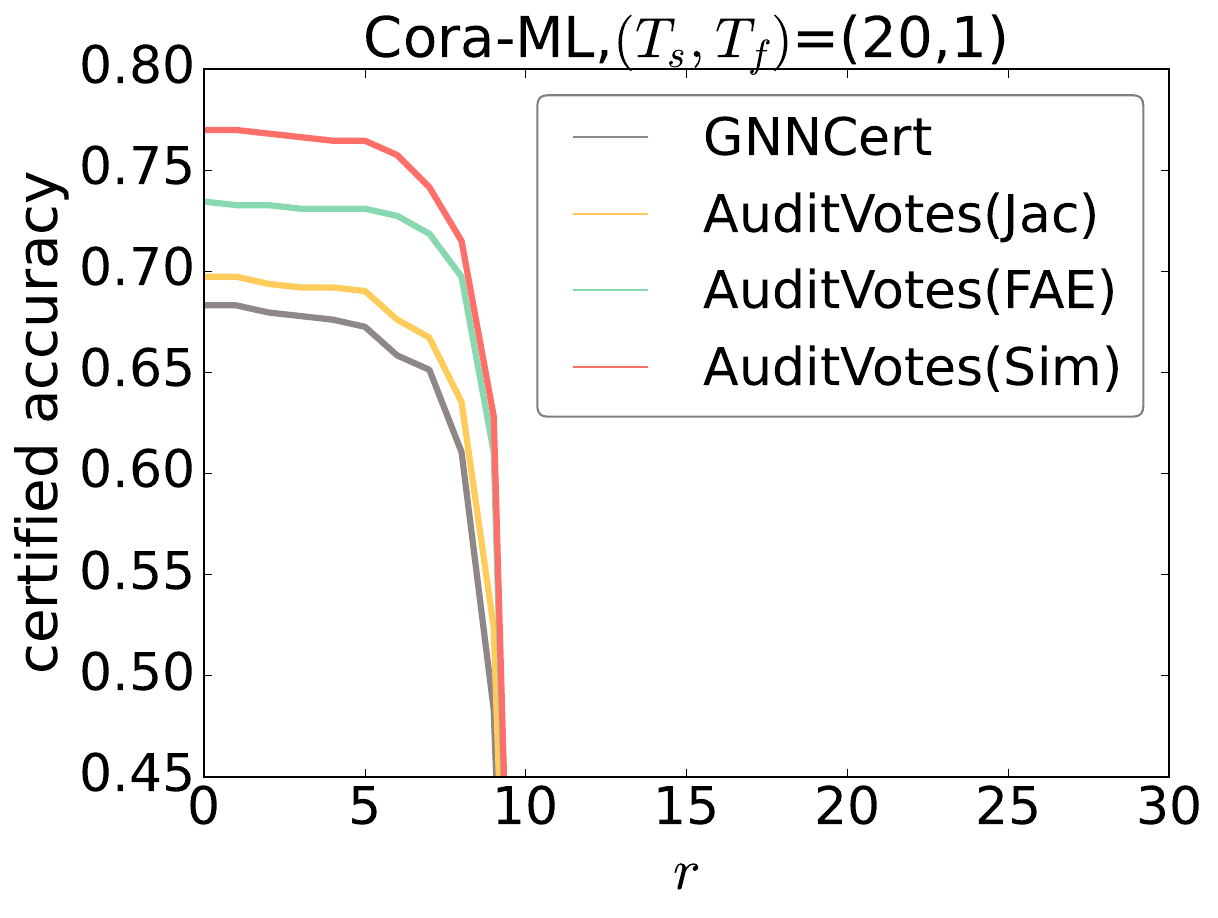}}
\vspace{-8pt}
\caption{Applying AuditVotes to GNNCert~\cite{xia2024gnncert}.}
\label{fig:gnncert}
\end{figure}

In Figure~\ref{fig:gnncert}, we apply the augmentation component of AuditVotes to the GNNCert model to address the challenge of overly sparse graphs. Both SimAug and FAEAug improve the clean accuracy and certified accuracy of GNNCert. For instance, on the Cora-ML dataset, SimAug increases the clean accuracy from $68.3\%$ to $77.0\%$ and boosts the certified accuracy (for arbitrary edge insertion/deletion with 7 edges) from $65.1\%$ to $74.2\%$. In Sec.~\ref {sec:runtime}, we also apply AuditVotes to AGNNCert to show its scalability. 

Furthermore, our conditional smoothing is highly applicable to randomized smoothing models in other domains. For example, we apply the confidence filter (Conf) with a threshold of $0.9$ to the Gaussian~\cite{cohen2019certified} randomized smoothing model used for image classification tasks. As shown in Table~\ref{tab:certify_cifar}, and Figure~\ref{fig:cer_image} (in Appendix~\ref{appendix:more_results}), Conf effectively improves certified performance. On the CIFAR-10 dataset, Conf increases clean accuracy by $9.8\%$ and certified accuracy by $104.2\%$ when the $l_2$-radius is $0.7$. Similarly, on the MNIST dataset, AuditVotes(Conf) enhances certified accuracy by $4.1\%$ at the same $l_2$-radius.

When compared to more advanced randomized smoothing training models, AuditVotes maintains superiority on the CIFAR-10 dataset (Figure~\ref{fig:cer_image}(a)) and achieves comparable performance on the MNIST dataset (Figure~\ref{fig:cer_image}(b)). Notably, Conf is computationally efficient, requiring minimal additional training time, whereas other models incur significant computational overhead. The limited improvement observed on the MNIST dataset is likely due to its already high baseline accuracy, which makes further enhancements more challenging.

\subsection{Efficiency and Scalability (Q5)}
\label{sec:runtime}

In this section, we evaluate the efficiency of AuditVotes in both node classification and image classification tasks. Table~\ref{tab:runtime_cora} presents the runtime performance when certifying edge addition perturbations (parameters: $p_+=0.2$, $p_-=0.6$) on the Cora-ML dataset.

It is worth noting that training our augmentation requires less than $20$ seconds, which is very efficient. During the testing (smoothing) phase, the edge intensity matrix, which is used for augmentation, is computed only once. Additionally, the augmentation techniques (FAEAug and SimAug) make the graph sparser (Figure~\ref{fig:e_sparsity}), resulting in faster testing. As a result, AuditVotes models FAEAug and SimAug exhibit lower total runtime compared to the vanilla SparseSmooth.

Moreover, our conditional smoothing model (Conf) also maintains high efficiency. It uses the same training procedure as the Gaussian smoothing model, requiring no additional computations during training. In the testing phase, Conf performs $N$ simple comparisons between confidence scores and a given threshold, introducing only a minor runtime increase. As shown in Table~\ref{tab:runtime_cora}, this results in negligible additional computational cost.

\begin{table}[!htb]
\centering
\caption{Certified accuracy comparison on CIFAR-10.}
\vspace{-8pt}
\begin{tabular}{lllll}
\toprule[0.9pt]
Certified $l_2$-radius: & 0.0 (clean) & 0.3 & 0.5 & 0.7 \\ \hline
Gaussian~\cite{cohen2019certified} & {\ul 0.758} & 0.565 & 0.410 & 0.262 \\
Stability~\cite{li2019certified} & 0.720 & 0.542 & 0.413 & 0.288 \\
CAT-RS~\cite{jeong2023confidence} & 0.711 & 0.627 & {\ul 0.564} & {\ul 0.489} \\
Diffusion~\cite{carlini2023certified} & \textbf{0.862} & {\ul 0.695} & 0.554 & 0.407\\
\textbf{+AuditVotes (Conf)} & {\ul 0.833} & \textbf{0.752} & \textbf{0.671} & \textbf{0.535} \\ \toprule[0.9pt]
\end{tabular}
\label{tab:certify_cifar}
\end{table}

In Table~\ref{tab:runtime_cifar}, we evaluate the runtime of various models on the CIFAR-10 dataset. The Conf model has the same training time as the Gaussian smoothing model, as it does not involve additional complexity during training. In contrast, baseline models such as Stability and CAT-RS require significantly higher training time. For example, CAT-RS requires approximately $\times11$ training time of the Gaussian model. During the testing (smoothing) phase, Conf only introduces less than $20$ minutes of additional runtime, demonstrating its computational efficiency.

\begin{table}[!htb]
\centering
\caption{Running time comparison on Cora-ML.}
\begin{tabular}{lrrr}
\toprule[0.9pt]
Smoothed models & Training & Testing & Total \\ \hline
\cellcolor{gray!10}SparseSmooth (GCN) & \cellcolor{gray!10}0:00:13 & \cellcolor{gray!10}0:09:33 & \cellcolor{gray!10}0:09:46 \\
+AuditVotes (JacAug) & 0:00:17 & 0:10:11 & 0:10:28 \\
+AuditVotes (FAEAug) & 0:00:27 & 0:08:51 & 0:09:18 \\
+AuditVotes (SimAug) & 0:00:31 & 0:08:33 & \textbf{0:09:04} \\
+AuditVotes (Conf) & 0:00:13 & 0:09:34 & 0:09:47 \\
+AuditVotes (JacAug+Conf) & 0:00:17 & 0:10:12 & 0:10:29 \\
+AuditVotes (FAEAug+Conf) & 0:00:27 & 0:08:52 & \underline{0:09:19} \\
+AuditVotes (SimAug+Conf) & 0:00:31 & 0:09:08 & 0:09:39 \\ \bottomrule[0.9pt]
\end{tabular}
\label{tab:runtime_cora}
\vspace{-4pt}
\end{table}

\begin{table}[!htb]
\centering
\caption{Running time comparison on CIFAR-10.}
\begin{tabular}{lrrr}
\toprule[0.9pt]
Models & Training & Testing & Total \\ \hline
Gaussian~\cite{cohen2019certified} & 00:52:59 & 01:59:41 & \textbf{02:52:40} \\
Stability~\cite{li2019certified} & 01:34:30 & 01:59:22 & 03:33:52 \\
CAT-RS~\cite{jeong2023confidence} & 11:26:47 & 01:59:37 & 13:26:24 \\
Diffusion~\cite{carlini2023certified} & - & 94:36:38 & 94:36:38 \\
\textbf{+AuditVotes (Conf)} & 00:52:59 & 02:12:24 & \underline{03:05:23} \\ \bottomrule[0.9pt]
\end{tabular}
\label{tab:runtime_cifar}
\end{table}


\begin{table*}[!ht]
\centering
\caption{Apply AuditVotes to AGNNCert~\cite{li2025agnncert} on large scale dataset (Amazon2M~\cite{chiang2019cluster}).}
\label{tab:amazon2m}
\begin{tabular}{rrrrrrrrr}
\hline
\multicolumn{1}{c}{Amazon2M dataset} & \multicolumn{5}{c}{Certified   accuracy (r)}                                       & \multicolumn{3}{c}{Runtime (s)} \\ \hline
\multicolumn{1}{c}{Models}           & 0 (clean)              & 5              & 10             & 20             & 30             & Training  & Testing  & Total    \\ \hline
Standard AGNNCert-E                           & 0.842          & 0.840          & 0.837          & 0.834          & 0.830          & 3369.842  & 148.297  & 3518.138 \\
\textbf{+AuditVotes (CosAug)}                 & \textbf{0.848} & \textbf{0.845} & \textbf{0.843} & \textbf{0.839} & \textbf{0.836} & 3482.895  & 150.298  & 3633.193 \\ \hline
Standard AGNNCert-N                           & 0.835          & 0.835          & 0.835          & 0.835          & 0.835          & 3428.633  & 145.512  & 3574.145 \\
\textbf{+AuditVotes (CosAug)}                 & \textbf{0.839} & \textbf{0.839} & \textbf{0.839} & \textbf{0.839} & \textbf{0.839} & 3622.470  & 154.939  & 3777.408 \\ \hline
\end{tabular}
\end{table*}

AGNNCert~\cite{li2025agnncert} is the first work that evaluates the scalable certified robustness on a large-scale dataset (Amazon2M~\cite{chiang2019cluster}). 
Following their work, we show our AuditVotes enhancement framework can be applied to AGNNCert and is also scalable to large graphs. Since the node features in the Amazon2M dataset are continuous instead of discrete, we employ cosine similarity instead of Jaccard similarity, termed as AuditVotes (\textbf{CosAug}). 

To pre-compute the similarity matrix for a large graph with $|V|$ nodes, its complexity is $O(|V|^2)$, which introduces a high computation workload. To reduce the complexity, we randomly sample $N$ nodes among $|V|$, and we only pre-compute the similarity matrix among the $N$ nodes. For example, in our paper, we set $N=\lfloor|V|/40\rfloor$ for the Amazon2M dataset. The experimental results are presented in Table \ref{tab:amazon2m}. It is observed that we have similar testing time compared to the standard AGNNCert-E or AGNNCert-N (less than 10s time increment). CosAug increases the training time by less than 200s (mainly in the pre-computation of the similarity matrix). 
With the minor time increment, our AuditVotes (CosAug) enhance both the clean accuracy (when $r=0$) and certified accuracy. These results show that our AuditVotes framework can also be applied to large-scale graphs efficiently.

\section{Related Works}
\label{sec:relatedW}
In this section, we discuss the related existing works in certified robustness and graph data augmentation.

\noindent\textbf{ - Certified Robustness.}
The mainstream approach to realize certified robustness is \textit{randomized smoothing}~\cite{cohen2019certified,lecuyer2019certified}, which is originally designed for the image classification domain. Then, it was adapted to graph domains with representative works ~\cite{lee2019tight,bojchevski2020efficient,jia2020certified,wang2021certified,lai2024node,lai2024collective}. These approaches add random noise to the input, and then thousands of samples are drawn from the noise to obtain a smoothed prediction. To avoid the huge sample size and obtain a deterministic guarantee, de-randomized smoothing schemes are proposed in image domain~\cite{levine2020randomized,levine2021improved} and graph domain~\cite{xia2024gnncert,yang2024distributed,li2025agnncert}. We mainly focus on the graph domain, improving both randomized and de-randomized smoothing schemes for node classification models. 

The mainstream of improving certified robustness consists of de-noised smoothing and more advanced training strategies. 
De-noised smoothing~\cite{salman2020denoised,carlini2023certified} employ diffusion model to de-noise the noisy input before the predictions, which can effectively improve the data quality. Nevertheless, these studies are in the image domains, and they can not be directly transferred to inductive graph learning tasks. We are the first to propose graph augmentations to de-noise the random graph. 

Another way to improve certified robustness is by improving training strategies~\cite{salman2019provably,li2019certified, Zhai2020MACER,jeong2020consistency,jeong2023confidence}. \cite{salman2019provably,li2019certified,gosch2024adversarial} employ adversarial training to improve the robustness of smoothed classifiers in classifying adversarial examples. \cite{Zhai2020MACER} propose attack-free robust training that maximizes the certified radius during training, which avoids the high computation in finding adversarial examples.  
\cite{jeong2020consistency} proposed consistency regularization term to reduce the variance of predictions. Similarly, \cite{jeong2023confidence} design a sample-wise and confidence-aware training objective. These models, employing adversarial training, designing prediction consistency regularization, and developing adaptive noise level training, are orthogonal to our work. Compared to these approaches, our conditional smoothing based on confidence score is a \textit{post-training} procedure that does not require extra computation during training. On the other hand, it is a general approach that can be applied to any base classifier using the training strategies above. 

\noindent\textbf{ - Graph Data Augmentation.}
Graph data augmentation~\cite{ding2022data,zhao2022graph} has received lots of research efforts, and it is widely used to enhance reliable graph learning. Graph augmentation can be used as an adversarial defense technique via cleaning the perturbed graph \cite{wu2019adversarial,zhang2020gnnguard,entezari2020all}. For instance, GCNJaccard~\cite {wu2019adversarial} uses the Jaccard similarity to quantify node feature similarity, removing the edges that connect dissimilar nodes to defend against add-edge adversarial attacks. Structural learning ~\cite{chen2020iterative,zhao2021data,jin2020graph} optimizes the edge connection as a parameter with the network parameter at the same time. For example, GAug~\cite{zhao2021data} uses graph auto-encoder to implement an edge predictor to improve the performance of GNNs on node classification tasks. Despite the advancement of graph data augmentation in empirical defense, it has not been developed in the context of inductive graph learning and enhancing randomized smoothing.

\section{Limitations}
\label{sec:limitation}
Despite the advantages of AuditVotes in enhancing clean accuracy and certified accuracy simultaneously, AuditVotes has limitations. The augmentation component of AuditVotes relies on the quality of node features; when node features are absent or extremely uninformative, augmentation contributes less. Nevertheless, our confidence filter is still meaningful in this case. It is worth noting that, when defending against node feature attacks, AuditVotes remains applicable via node-partitioning certification (e.g., AGNNCert-N~\cite{li2025agnncert}), because malicious node features are "masked" in the subgraphs.   

\section{Conclusion and Future Works}
\label{sec:conclu}

In this paper, we addressed critical limitations in the field of certifiably robust Graph Neural Networks (GNNs) by proposing AuditVotes, a general enhancement to the randomized smoothing framework. Our work, aiming to mitigate the trade-off between model accuracy and certified performance, is the first framework that can achieve both high model accuracy and certified accuracy for GNNs.
AuditVotes incorporates two novel components: augmentation and conditional smoothing, both of which are efficient, adaptive to noise, and suitable for inductive learning scenarios.

We instantiated AuditVotes with three augmentation strategies—JacAug, FAEAug, and SimAug—that are computationally efficient and generalizable to unseen nodes. These augmenters enhance the quality of randomized graphs, leading to significant improvements in clean and certified accuracy. Additionally, we introduced conditional smoothing based on prediction confidence to exclude low-quality votes and improve voting consistency, which further boosts certified accuracy with minimal computational overhead. We establish a theoretical guarantee for the AuditVotes to obtain certified robustness. Our extensive evaluations demonstrate that AuditVotes achieves not only substantial gains in clean, certified, and empirical accuracy but also exhibits broad applicability as a general framework for enhancing other smoothing schemes, such as GNNCert and Gaussian smoothing for image classification.


By providing higher model accuracy, stronger robustness guarantees, and wide applicability, AuditVotes paves the way for robust GNNs in security-sensitive applications where accuracy and robustness are paramount. We believe that our framework can inspire further research into exploring the more advanced designs of augmentation methods and filtering functions.


\section*{Acknowledgements}
This paper was proofread for grammar with OpenAI GPT-5. This work is partially supported by the Research Center for Culture \& Sci-Tech Integration Innovation, Key Research Base of Humanities and Social Sciences of Hubei Province (Project number: WK2026ZD01); the Fundamental Research Funds for the Central Universities (Project number: XJSJ26042); and the Hong Kong Research Grants Council General Research Fund (Project number: PolyU15230025).

\bibliographystyle{ACM-Reference-Format}
\balance
\bibliography{references}

\appendix

\section*{Ethical Considerations}

In this section, we discuss the ethical implications, societal impact, and steps taken to ensure the responsible conduct of this research. Our work aims to enhance the robustness and accuracy of Graph Neural Networks (GNNs) against adversarial attacks in security-sensitive applications such as fraud detection and social networks. 

\subsubsection*{Stakeholders and Impacts}

\begin{itemize}
    \item \textit{Users of GNN Systems:} Individuals and organizations relying on GNN-powered systems in domains such as financial fraud detection, social network moderation, and traffic monitoring benefit from improved robustness and accuracy, reducing the risk of adversarial attacks compromising their systems.
    \item \textit{Researchers and Practitioners:} Security researchers and machine learning practitioners gain access to a framework that advances certified robustness, improving the state of the art.
    \item \textit{Society at Large:} By improving the security of GNNs, our research indirectly contributes to safer and more trustworthy systems for end-users, promoting public confidence in AI technologies.
\end{itemize}

\subsubsection*{Ethical Principles}
We adhered to the principles outlined in \textit{The Menlo Report}:
\begin{itemize}
    \item \textit{Beneficence:} Our work aims to benefit society by improving the robustness of GNNs, ensuring their safe deployment in critical applications. The potential benefits outweigh any risks, as the research enhances the security of systems vulnerable to adversarial attacks.
    \item \textit{Respect for Persons:} We ensured the privacy and anonymity of datasets by exclusively using publicly available benchmark datasets (e.g., Cora-ML, Citeseer, PubMed) without collecting or using sensitive user data.
    \item \textit{Justice:} The research promotes fairness by providing defenses applicable to a wide range of GNN tasks, ensuring that advancements in robustness and accuracy are accessible to diverse applications.
    \item \textit{Respect for Law and Public Interest:} We followed all ethical and legal guidelines for the use of publicly available datasets, ensuring compliance with data usage agreements and privacy standards.
\end{itemize}

\subsubsection*{Potential Harms and Mitigations}

\begin{itemize}
    \item \textit{Misuse of Research:} While our framework is designed to enhance robustness, it could be analyzed by adversaries to understand defense mechanisms. Nevertheless, we provide certified defense without exposing specific vulnerabilities that could be exploited.
    \item \textit{Privacy Concerns:} Our research exclusively uses publicly available benchmark datasets. No private or sensitive user data was collected or processed, ensuring full respect for privacy rights.
\end{itemize}

\subsubsection*{Decision to Proceed and Publish}
The decision to conduct and publish this research was based on the following considerations:
\begin{itemize}
    \item \textit{Ethical Benefits:} The potential benefits of improving the robustness and accuracy of GNNs outweigh the limited risks.
    \item \textit{Transparency and Community Advancement:} Publishing this research contributes to the academic and practitioner communities, enabling further advancements in secure machine learning.
    \item \textit{Minimized Risks:} Through the use of publicly available datasets, compliance with ethical guidelines, and a focus on societal benefits, we mitigated potential harms to ensure a responsible and ethical research process.
\end{itemize}

\subsubsection*{Conclusion}
We believe this research aligns with ethical principles by prioritizing societal benefit, respecting privacy, and advancing security in critical machine learning applications. By addressing potential harms and ensuring transparency, we have taken the necessary steps to conduct and publish this research responsibly. We welcome feedback from the community to further improve the societal impact and ethical conduct of our work.

\section*{Open Science}
We implement this using Python (PyTorch 1.13.1) on an NVIDIA GeForce RTX 3090 Ti.  
All the code, data, settings, and usage are provided in our GitHub Repo:\\
https://github.com/Yuni-Lai/AuditVotes-Certified-Robustness.

\section*{Appendix}
\section{Proofs for Theorems}
(Note: We recommend that readers consult the proof in SparseSmooth~\cite{bojchevski2020efficient} and Gaussian~\cite{cohen2019certified} before reading the proofs below.)
\subsection{Proof for Theorem~\ref{theorem-condition}}
\label{AppendixA.1}
\begin{proof}
We employ a similar proof scheme as SparseSmooth~\cite{bojchevski2020efficient}, and the problem of certified robustness can be formulated as a linear programming problem in \eqref{opt:randomsmooth}. Specifically, given a node $v$, let $y_A$ and $y_B$ denote the top predicted class and the runner-up class, respectively. Let $\underline{p_A}$ and $\overline{p_B}$ denote the lower bound of $p_{v,y_A}$ and the upper bound of $p_{v,y_B}$, respectively. Let $p'_A$ and $p'_B$ denote the probability of predicting $y_A$ and $y_B$ given the perturbed graph $\mathcal{G}'$. By the definition of the smoothed classifier $g^c(\cdot)$ defined in Eq.~\eqref{eqn:consmooth_g}, we know that the prediction is $y_A$ if $p'_A>p'_B$. We employ the linear programming problem to find a worst-case classifier represented by vectors $\mathbf{s}$ and $\mathbf{t}$ such that the classification margin $\mu_{r_a,r_b}:=p'_A-p'_B$ under the perturbed graph is minimized. The vectors $\mathbf{s}\in[0,1]^{I}$ and $\mathbf{t}\in[0,1]^{I}$ encode the classifier that assigns class $y_A$ and class $y_B$ among the regions. More specifically, $s_i=\mathbb{P}(f(Z)=y_A|h(Z)=0,Z\in\mathcal{R}_i)$, and $t_i=\mathbb{P}(f(Z)=y_B|h(Z)=0,Z\in\mathcal{R}_i)$.
Given that the classifier $g^c(\cdot)$ predicts $y_A$ for the randomized clean graph $\phi(\mathcal{G})$ with probability at least $\underline{p_A}$, and predicts $y_B$ with probability at most $\overline{p_B}$, the worst-case classifier satisfies $\mathbf{s}^T\mathbf{r}=\underline{p_A}$, and $ \mathbf{t}^T\mathbf{r}=\overline{p_B}$. 
In the worst case, for $\phi(\mathcal{G}')$, $\mathbf{s}$ tends to assign the lowest probability of class $y_A$ and $\mathbf{t}$ tends to assign the highest probability of class $y_B$. Therefore, the worst-case classifier $\mathbf{s}$ assigns class $y_A$ in decreasing order of the constant likelihood regions until $\mathbf{s}^T\mathbf{r}=\underline{p_A}$, and $\mathbf{t}$ assigns class $y_B$ in increasing order of the constant likelihood regions until $\mathbf{t}^T\mathbf{r}=\overline{p_B}$. With this classifier represented by $\mathbf{s}$ and $\mathbf{t}$, the classification margin $\mu_{r_a,r_b}:=p'_A-p'_B=\mathbf{s}^T\mathbf{r}'-\mathbf{t}^T\mathbf{r}'$ is minimized. 

Next, we provide an example to further illustrate the process of decomposing the probability $p'_A-p'_B$ into vectors $\mathbf{s}^T\mathbf{r}'-\mathbf{t}^T\mathbf{r}'$ using the law of total probability. Assuming that there are two constant likelihood regions $\mathcal{R}_1$ and $\mathcal{R}_2$, then we can decompose the conditional probabilities $\underline{p_A}$ and $\overline{p_B}$ as follows:
\begin{align}
\underline{p_A}&=\mathbb{P}(f(\phi(\mathcal{G}))=y_A| h(\phi(\mathcal{G}))=0)\nonumber\\
&=\mathbb{P}(f(Z)=y_A|\phi(\mathcal{G})=Z \in \mathcal{R}_1, h(Z)=0)\nonumber\\
&\quad\times\mathbb{P}(\phi(\mathcal{G})=Z \in \mathcal{R}_1)\nonumber\\
&\quad+\mathbb{P}(f(Z)=y_A|\phi(\mathcal{G})=Z \in \mathcal{R}_2, h(Z)=0)\nonumber\\
&\quad\times\mathbb{P}(\phi(\mathcal{G})=Z \in \mathcal{R}_2)
=s_1 r_1+s_2r_2=\mathbf{s}^T\mathbf{r}.\nonumber
\end{align}
\begin{align}
\overline{p_B}&=\mathbb{P}(f(\phi(\mathcal{G}))=y_B| h(\phi(\mathcal{G}))=0)\nonumber\\
&=\mathbb{P}(f(Z)=y_B|\phi(\mathcal{G})=Z \in \mathcal{R}_1, h(Z)=0)\nonumber\\
&\quad\times\mathbb{P}(\phi(\mathcal{G})=Z \in \mathcal{R}_1)\nonumber\\
&\quad+\mathbb{P}(f(Z)=y_B|\phi(\mathcal{G})=Z \in \mathcal{R}_2, h(Z)=0)\nonumber\\
&\quad\times\mathbb{P}(\phi(\mathcal{G})=Z \in \mathcal{R}_2)
=t_1 r_1+t_2r_2=\mathbf{t}^T\mathbf{r}.\nonumber
\end{align}

Next, our goal is to estimate the prediction probabilities given arbitrary perturbed graph $\mathcal{G}'\in \mathcal{B}_{r_a,r_b}$: 
\begin{align}
p'_A&:=\mathbb{P}(f(\phi(\mathcal{G}'))=y_A| h(\phi(\mathcal{G}'))=0)\nonumber\\
&=\mathbb{P}(f(Z)=y_A|\phi(\mathcal{G}')=Z \in \mathcal{R}_1, h(Z)=0)\nonumber\\
&\quad\times\mathbb{P}(\phi(\mathcal{G}')=Z \in \mathcal{R}_1)\nonumber\\
&\quad+\mathbb{P}(f(Z)=y_A|\phi(\mathcal{G}')=Z \in \mathcal{R}_2, h(Z)=0)\nonumber\\
&\quad\times\mathbb{P}(\phi(\mathcal{G}')=Z \in \mathcal{R}_2)=s_1 r'_1+s_2r'_2=\mathbf{s}^T\mathbf{r}'.\nonumber
\end{align}
\begin{align}
p'_B&:= \mathbb{P}(f(\phi(\mathcal{G}'))=y_B| h(\phi(\mathcal{G}'))=0)\nonumber\\
&=\mathbb{P}(f(Z)=y_B|\phi(\mathcal{G}')=Z \in \mathcal{R}_1, h(Z)=0)\nonumber\\
&\quad\times\mathbb{P}(\phi(\mathcal{G}')=Z \in \mathcal{R}_1)\nonumber\\
&\quad+\mathbb{P}(f(Z)=y_B|\phi(\mathcal{G}')=Z \in \mathcal{R}_2, h(Z)=0)\nonumber\\
&\quad\times\mathbb{P}(\phi(\mathcal{G}')=Z \in \mathcal{R}_2)=t_1 r'_1+t_2r'_2=\mathbf{t}^T\mathbf{r}'.\nonumber
\end{align}
A similar decomposition can be obtained for more region numbers using the law of total probability. Finally, the prediction of sample $\mathcal{G}$ by the smoothed classifier defined in ~\eqref{eqn:consmooth_g} can be certified if:
\begin{align}
    \mu_{r_a,r_b} &:=p'_A-p'_B\nonumber\\
    &=\mathbb{P}(f(\phi(\mathcal{G}'))=y_A| h(\phi(\mathcal{G}'))=0)\nonumber\\
    &\quad-\mathbb{P}(f(\phi(\mathcal{G}'))=y_B| h(\phi(\mathcal{G}'))=0)\nonumber\\
    &=\mathbf{s}^T\mathbf{r}'-\mathbf{t}^T\mathbf{r}'
    >0.\nonumber
\end{align}

Next, the key to obtaining the certifying condition and solving the optimization problem is to find the consent likelihood $\frac{\mathbb{P}(\phi(\mathcal{G})=Z)}{\mathbb{P}(\phi(\mathcal{G}')=Z)}=c_i$, divide the regions $\mathbb{G}=\bigcup_{i=1}^I \mathcal{R}_i$, and get the probability $r$ and $r'$~\cite{bojchevski2020efficient,lee2019tight}. Let $\mathcal{G}'$ denote the perturbed graph among $\mathcal{B}_{r_a,r_b}(\mathcal{G})$, and $Z\in\mathbb{G}$ be any possible graph obtained by $\phi(\mathcal{G})$ or $\phi(\mathcal{G}')$. We now need to compute the likelihood ratio with condition $h(Z)=0$. \textbf{Notably, the probability $\mathbb{P}(\phi(\mathcal{G})=Z| h(Z)=0)$ only depends on the number of edge different between $Z$ and $\mathcal{G}$, and does not depend on the filter $h(Z)$ because the randomization $\phi(\cdot)$ is totally random and it has equal probability to delete the exiting edges, and equal probability to add non-exiting edges}. 
Then, we have the likelihood ratio $\Lambda(Z)$ as follows:
\begin{align}
\label{eqn:likelihood_ratio}
    \Lambda(Z)&=\frac{\mathbb{P}(\phi(\mathcal{G})=Z| h(Z)=0)}{\mathbb{P}(\phi(\mathcal{G}')=Z| h(Z)=0)}=\frac{\mathbb{P}(\phi(\mathcal{G})=Z)}{\mathbb{P}(\phi(\mathcal{G}')=Z)}.
\end{align}  
That is, our conditional smoothing model (with arbitrary $h(\cdot)$) does not affect the likelihood ratio, allowing us to utilize the same constant likelihood ratio region as in SparseSmooth~\cite{bojchevski2020efficient}. This underscores the seamless adaptability of the certifying condition to our conditional smoothing model. The only difference lies in the definition of $\underline{p_A}$ and $\overline{p_B}$. In SparseSmooth~\cite{bojchevski2020efficient}, the $\underline{p_A}$ and $\overline{p_B}$ are estimated among all the votes, while in our model, the probabilities $\underline{p_A}$ and $\overline{p_B}$ are estimated among the valid votes ($h(Z)=0$). Specifically, let $X$ denote the adjacency matrix of the clean graph $\mathcal{G}$, and $Y$ denote the adjacency matrix of a perturbed graph $\mathcal{G}'\in \mathcal{B}_{r_a,r_b}(\mathcal{G})$. We know that in the worst case, the attacker consumes all the attack budget so that $X$ and $Y$ have exactly $r_a+r_d$ different bits. Let $\mathcal{C}=\{(k,l)|X_{kl}\neq Y_{kl}\}$ denotes the location that $X$ and $Y$ are different, and $X_\mathcal{C}\in \{0,1\}^{|\mathcal{C}|}$ denote the elements of $X$ in location $\mathcal{C}$. We know that $X_\mathcal{C}$ must contains $r_a$ zeros and $r_d$ ones: $||\mathbf{1}-X_\mathcal{C}||_o=r_a, ||X_\mathcal{C}||_o=r_d$. 
We can divide the space $\mathbb{G}$ into $r_a+r_d+1$ disjoint regions: $\mathbb{G}=\bigcup_{i=0}^{r_a+r_d} \mathcal{R}_i$, where $\mathcal{R}_i$ contains all the adjacency matrix that can be obtained by flipping $i$ bits in $X_\mathcal{C}$ and have any combination of ones and zeros in the other location:
\begin{align}
    \mathcal{R}_i=\{Z\in\mathbb{G}:||X_\mathcal{C}-Z_\mathcal{C}||_o=i\}, i=0,1,\cdots,r_a+r_d.\nonumber
\end{align}
Equivalently, the region $\mathcal{R}_i$ contains all adjacency matrix that obtained by not flipping $i$ bits in $Y_\mathcal{C}$ because $Y_\mathcal{C}=1-X_\mathcal{C}$. 
Then we have the constant likelihood:
\begin{align}
    \Lambda(Z)_i&=\frac{\mathbb{P}(\phi(X)=Z| h(Z)=0)}{\mathbb{P}(\phi(Y)=Z| h(Z)=0)}=\frac{\mathbb{P}(\phi(X)=Z)}{\mathbb{P}(\phi(Y)=Z)}\nonumber\\
    &=[\frac{p_+}{1-p_-}]^{i-r_d}[\frac{p_-}{1-p_+}]^{i-r_a}. \nonumber
\end{align}
Given $r_a$ and $r_d$, this likelihood ratio $\Lambda(Z)_i$ is monotonically decreasing of $i$ if $p_-+p_+\leq1$, and monotonically decreasing, otherwise.

According to ~\cite{bojchevski2020efficient}, the probability $r_i=\mathbb{P}(\phi(\mathcal{G})\in\mathcal{R}_i)$ and $r'_i=\mathbb{P}(\phi(\mathcal{G}')\in\mathcal{R}_i)$ are Poisson-Binomial distributions:
\begin{align}
    r_i=&\mathbb{P}(\phi(\mathcal{G})\in\mathcal{R}_i)=PB([p_+,r_a],[p_-,r_b]),\nonumber\\
    r'_i=&\mathbb{P}(\phi(\mathcal{G}')\in\mathcal{R}_i)=PB([1-p_-,r_a],[1-p_+,r_b]),\nonumber
\end{align}
where $PB([p_+,r_a],[p_-,r_b])$ denote Poisson-Binomial distribution with parameter $p_+$ repeated for $r_a$ times, and $p_-$ repeated for $r_d$times: $PB(p_+,\cdots,p_+,p_-,\cdots,p_-)$. (These probabilities can be calculated following the same calculation procedure as in SparseSmooth~\cite{bojchevski2020efficient}. )

With consent likelihood $\frac{\mathbb{P}(\phi(\mathcal{G})=Z)}{\mathbb{P}(\phi(\mathcal{G}')=Z)}=c_i$, and the corresponding regions $\mathcal{R}_i$, and the probability $r$ and $r'$, we can solve the optimization problem~\eqref{opt:randomsmooth} and certify the robustness. In the worst case, for $\phi(\mathcal{G}')$, $\mathbf{s}$ tends to assign the lowest probability of class $y_A$ and $\mathbf{t}$ tends to assign the highest probability of class $y_B$. Therefore, the worst-case classifier $\mathbf{s}$ assigns class $y_A$ in decreasing order of the constant likelihood regions until $\mathbf{s}^T\mathbf{r}=\underline{p_A}$, and $\mathbf{t}$ assigns class $y_B$ in increasing order of the constant likelihood regions until $\mathbf{t}^T\mathbf{r}=\overline{p_B}$~\cite{bojchevski2020efficient}. With this classifier represented by $\mathbf{s}$ and $\mathbf{t}$, the classification margin $\mu_{r_a,r_b}:=p'_A-p'_B=\mathbf{s}^T\mathbf{r}'-\mathbf{t}^T\mathbf{r}'$ is minimized
\end{proof}

\subsection{Proof for Theorem~\ref{theorem-condition_image}}
\label{AppendixA.2}

\begin{proof}
     Given an image $x$, let $y_A$ and $y_B$ denote the top predicted class and the runner-up class, respectively. Let $\underline{p_A}$ and $\overline{p_B}$ denote the lower bound of $\mathbb{P}(f(x+\epsilon)=y_A|h(x+\epsilon)=0)$ and the upper bound of $\mathbb{P}(f(x+\epsilon)=y_B|h(x+\epsilon)=0)$, respectively.
     We have $g(x')=g(x)$ if $\mathbb{P}(f(x'+\epsilon)=y_A|h(x'+\epsilon)=0)>\max_{y_B} \mathbb{P}(f(x'+\epsilon)=y_B|h(x'+\epsilon)=0)$. Let's denote two random variables:
    \begin{align}
        &X:=x+\epsilon = \mathcal{N}(x,\sigma^2I),\nonumber\\
        &Y:=x+\delta+\epsilon = \mathcal{N}(x+\delta,\sigma^2I),\nonumber
    \end{align}
    where $\delta$ satisfies: $x+\delta=x'$. By the definition of $\underline{p_A}$ and $\overline{p_B}$, we know that: $\mathbb{P}(f(X)=y_A|h(X)=0)\geq\underline{p_A}$ and $\mathbb{P}(f(X)=y_B|h(X)=0)\leq\overline{p_B}$. According to Neymen-Pearson Lemma adapted by Gaussian~\cite{cohen2019certified} (Lemma 4), we know that:

    \textit{Let $s:\mathbb{R}^d\rightarrow \{0,1\}$ denote \textbf{any function} that outputs $0$ or $1$. If a half-space $A=\{Z\in \mathbb{R}^d:\delta^Tz\geq \beta\}$ for some $\beta$ and $\mathbb{P}(s(X)=1)\leq \mathbb{P}(X\in A)$, then $\mathbb{P}(s(Y)=1)\leq \mathbb{P}(Y\in A)$. Similarity, if a half-space $B=\{Z\in \mathbb{R}^d: \delta^Tz\leq \beta\}$ for some $\beta$ and $\mathbb{P}(s(X)=1)\geq \mathbb{P}(X\in A)$, then $\mathbb{P}(s(Y)=1)\geq \mathbb{P}(Y\in A)$.} 

    We define a half-space $A$ such that $\mathbb{P}(X\in A)=\underline{p_A}$, then:
    $\mathbb{P}(f(X)=y_A|h(X)=0)\geq \underline{p_A}=\mathbb{P}(X\in A).$
    Similarly, we define a half-space $B$ such that $\mathbb{P}(X\in B)=\overline{p_B}$, then:
    $\mathbb{P}(f(X)=y_B|h(X)=0)\leq \overline{p_B}=\mathbb{P}(X\in B).$
    Let $\mathbf{1}[\cdot]$ denote the indicator function, by applying the Neymen-Pearson Lemma above with $s(Z):=\mathbf{1}[f(Z)=c_A]\cdot\mathbf{1}[h(Z)=0]$ and $s(Z):=\mathbf{1}[f(Z)=c_B]\cdot\mathbf{1}[h(Z)=0]$, respectively, we have: 
    $$\mathbb{P}(f(Y)=y_A,h(Y)=0)\geq \mathbb{P}(Y\in A),$$
    $$\mathbb{P}(f(Y)=y_B,h(Y)=0)\leq\mathbb{P}(Y\in B).$$
    Assuming that $\mathbb{P}(h(Y)=0)\neq0$ (this assumption is valid because the classifier needs at least one vote for classification), by the definition of joint distribution, we have:
    $$\mathbb{P}(f(Y)=y_A|h(Y)=0)\geq \frac{\mathbb{P}(Y\in A)}{\mathbb{P}(h(Y)=0)},$$
    $$\mathbb{P}(f(Y)=y_B|h(Y)=0)\leq \frac{\mathbb{P}(Y\in B)}{\mathbb{P}(h(Y)=0)}.$$
    Then, we have $\mathbb{P}(f(Y)=y_A|h(Y)=0)\geq\mathbb{P}(f(Y)=y_B|h(Y)=0)$ if $\mathbb{P}(Y\in A)\geq\mathbb{P}(Y\in B)$. Next, we define the half-space following~\cite{cohen2019certified}:
    $A:=\{Z:\delta^T(Z-x)\leq\sigma||\delta||\Phi^{-1}(\underline{p_A})\},$ and
    $B:=\{Z:\delta^T(Z-x)\geq\sigma||\delta||\Phi^{-1}(1-\overline{p_B})\}.$
    We compute the probability that $\mathbb{P}(Y\in A)$and $\mathbb{P}(Y\in B)$:
    $$\mathbb{P}(Y\in A)=\Phi(\Phi^{-1}(\underline{p_A})-||\delta||/\sigma),$$
    $$\mathbb{P}(Y\in B)=\Phi(\Phi^{-1}(\overline{p_B})+||\delta||/\sigma).$$
    Finally, we have the certifying condition that: $g(x')=g(x)$ if 
    $||\delta||_2<\frac{\sigma}{2}(\Phi^{-1}(\underline{p_A})-\Phi^{-1}(\overline{p_B}))$.
\end{proof}

\section{More Implementation Details}
\label{Sec:implement_detail}

\subsection{Model Training}
\label{Sec:AppendixB.2}
\textbf{Base model.} For randomized smoothing models, 
 we follow the training setting in ~\cite{cohen2019certified,bojchevski2020efficient,lai2024node}. Specifically, during the training phase, we add the same random noise into the training data to improve the accuracy of the base model. For de-randomized smoothing models, we follow~\cite{xia2024gnncert} that includes all the subgraphs induced from the training graph for training. 

\textbf{Augmentation models.}
For augmentation training, we sample $90\%$ of existing edges in the training subgraph as the positive edges, and sample non-edges in a quantity $10$ times that of the positive edges as negative edges. These positive and negative edges are used to train the edge prediction models in FAEAug and SimAug with loss functions in Eq.~\eqref{eqn:loss_fae}. After the augmenter is trained, we apply it to the testing graph and pre-calculate the edge intensity matrix.

\subsection{Certificate Calculation}
\label{Sec:AppendixB.3}

\begin{algorithm}[!h]
\caption{AuditVotes for node classification.}  
\label{alg:certify_node}
\begin{algorithmic}[1]  
\REQUIRE Clean testing graph $\mathcal{G}$, smoothing distribution $\phi(G)$ with parameters $p_+$ and $p_-$, trained base classifier $f(\cdot)$, trained augmenter $\mathcal{A}(\cdot)$, conditional filtering function $h(\cdot)$, sample size $N$, confidence level $\alpha$, perturbation budget $r_a$ and $r_d$.
\STATE{Draw $N$ random graphs $\{\mathcal{G}_i|\sim \mathcal{G}_i \sim \phi(\mathcal{G})\}_{i=1}^N$.}
\STATE{Pre-process the random graphs with augmenter $\mathcal{A}(\mathcal{G}_i)$, for $i=1,\cdots, N$.}
\STATE{Count the votes for each class: $counts=|\{i: f(\mathcal{A}(\mathcal{G}_i))=y \cap h(\mathcal{A}(\mathcal{G}_i))=0 \}|$, for $y=1, \cdots, C$.}
\STATE{Total valid counts: $N_v=sum(counts)$.}
\STATE{$y_A,y_B=$ top two indices in $counts$.}
\STATE{$n_A,n_B=counts[y_A],counts[y_B]$.}
\STATE{$\underline{p_A},\overline{p_B}=\text{CP\_Bernoulli}(n_A,n_B,N_v,\alpha)$.}
\IF{Binomial($n_A+n_B,\frac{1}{2})>\alpha$}
    \RETURN ABSTAIN
\IF{$\mu_{r_a,r_b}>0$}
\RETURN Certified prediction $y_A$.
\ENDIF
\ENDIF
\end{algorithmic}  
\end{algorithm}
\textbf{Probability lower bound.} We employ Monte Carlo probability approximation following~\cite{bojchevski2020efficient,cohen2019certified,jia2020certified}. We draw $N$ samples to estimate the Clopper-Pearson Bernoulli confidence interval~\cite{clopper1934use} with adjusted confidence $\alpha/C$, where $\alpha$ is the confidence level and $C$ is the number of classes. For our proposed conditional smoothing, the sample size is adjusted to $N-N_e$, where $N_e$ is the number of exclusions.

\textbf{Certified accuracy.}
We describe the process of obtaining certified robustness in Algorithm~\ref{alg:certify_node}. We employ the same procedures to calculate the $\mu_{r_a,r_b}$ as in SparseSmooth~\cite{bojchevski2020efficient}, and the intuition of the calculation is included in the Proof of Theorem~\ref{theorem-condition} (Appendix~\ref{AppendixA.1}). The certified accuracy can be obtained by the ratio of nodes that are both correctly classified by the smoothed classifier (i.e., $y_A=y_{true}$), and the prediction is certified to be robust.

\subsection{More Baseline Description}
\label{appendix:baselines}

\textbf{Image Classification Task:} For the image classification task, we compare our proposed conditional smoothing (\textbf{Conf}) to four baselines:
\begin{itemize}
    \item[1)] Gaussian~\cite{cohen2019certified}: The initial randomized smoothing model. Gaussian noise is added to the input image, and then the smoothed classification is obtained by majority voting regarding the noisy input. Noisy samples are used for augmented training samples to improve the generalization. 
     \item[2)] Stability~\cite{li2019certified}: It introduced a stability regularization term in the loss function to improve the prediction consistency. Specifically, the stability regularization forces the predicted probability vector of the original data and the data with Gaussian noise to be close. 
    \item[3)]CAT-RS~\cite{jeong2023confidence}: It proposed a confidence-aware training strategy that uses prediction confidence to regularize the loss function, preventing the samples with low confidence from being highly (certified) robust while forcing high-confidence samples to be highly (certified) robust.
    \item[4)]Diffusion~\cite{carlini2023certified}: It employs an off-the-shelf denoise diffusion model to denoise the input, which improves the model's clean accuracy significantly. 
    \item[5)]\textbf{AuditVotes (Conf)}: We propose a conditional smoothing framework and use conference scores to filter the votes. Our model is a post-training process and does not need extra computation workload during training.   
\end{itemize}

\textbf{Adversarial Attack Defense:} Moreover, the smoothed models can also serve as empirical robust models defending against actual adversarial attacks. We also evaluate the effectiveness of our proposed AuditVotes as an empirical robust model under Nettack~\cite{zugner2018adversarial} and IG-attack~\cite{wu2019adversarial} (widely used structure attacks for graph data). We compare our models with regular (non-smoothed) robust GNNs:
\begin{itemize}
    \item[1)] GCN~\cite{kipf2016semi}: The most classical GCN model. 
    \item[2)] GAT~\cite{veličković2018graph}: Graph attention network (GAT) employs attention layers to learn the weights for neighbor nodes.
    \item[3)] MedianGCN~\cite{chen2021understanding}: It substitutes the weighted mean aggregation in GCN by median aggregation. 
    \item[4)] AirGNN~\cite{liu2021graph}:  Node-wise adaptive residual was added to the GNN model by adaptive message passing. 
\end{itemize}

\section{Novelty of This Paper}
\subsection{Novelty of introducing graph augmentation}
While diffusion denoising augmentations (e.g. Carlini 2023~\cite{carlini2023certified}) are well-studied for images, directly extending diffusion/score-based methods to graphs is impractical: they are computationally heavy (Table~\ref{tab:runtime_cifar}), non-scalable, and tailored to Gaussian noise in continuous spaces. Graph node classification often operates on very large graphs (e.g., Amazon2M with 2.45M nodes), and SparseSmooth (or GNNCert) does not employ Gaussian noise, making such pipelines impractical.

Existing graph augmentations~\cite{wu2019adversarial,zhao2021data,chen2020iterative} optimize a single, static graph to aid robust graph learning, and the training-based augmentations~\cite{zhao2021data,chen2020iterative} co-train edges with the task network, which is not applicable to smoothing scenarios because it is computationally heavy to train the augmentation for each random graph. In contrast, our augmentations are designed to process a large number of randomized graphs efficiently and generalize to unseen nodes/graphs. We precompute a single edge-intensity matrix and reuse it during testing, keeping runtime/memory overheads low and scaling to graphs with 2.45M nodes on a single 24GB 3090 GPU (Table~\ref{tab:amazon2m}). Existing approaches consider one-way pruning and are not calibrated to certification noise, while ours are bidirectional rewiring and noise-adaptive. We apply adaptive thresholds tied to ($p_+,p_-$) of SparseSmooth or $T_s$ of GNNCert, to systematically restore sparsity or homophily degraded by smoothing noise. Our augmentations are designed to be certificate-compatible, lightweight, inductive, noise-adaptive, and effective.

In spirit, this parallels Carlini 2023~\cite{carlini2023certified}’s employing existing diffusion to the image smoothing pipeline; here, we adapt graph augmentation to the inductive and large-scale graph setting. Our design makes it compatible with certification and practical at scale, yielding substantial certified-accuracy gains with only a negligible runtime increase.

\subsection{Novelty of introducing conditional smoothing}
Both for image and graph, it is the first time that we prove that a filter (the filter can be a confidence score or other filtering models) can be directly applied to the randomized smoothing model to enhance the performance, and the final certified criteria are nearly the same. Most importantly, this almost computation-free filtering (Conf) outperforms other baselines that are designed to improve the certified performance (Table~\ref{tab:certify_cifar} and Table~\ref{tab:runtime_cifar}). The baselines require extra training, while our filtering method does not, and Conf still achieves better certified accuracy.

\begin{table*}[!ht]
\centering
\caption{Certified accuracy ($r_d$) comparison across GNNCert and SparseSmooth w/wo AuditVotes applied (Cora-ML).}
\setlength{\tabcolsep}{1.5pt}
\begin{tabular}{lrrrrrrrrr}
\hline
 & \multicolumn{1}{c}{clean acc} & \multicolumn{4}{c}{certified acc ($r_d$)} & \multicolumn{1}{c}{maximum} & \multicolumn{1}{c}{runtime (s)} & \multicolumn{1}{c}{runtime (s)} & \multicolumn{1}{c}{Parameters} \\ \cline{1-6} \cline{8-9}
Models & \multicolumn{1}{c}{r=0} & \multicolumn{1}{c}{r=3} & \multicolumn{1}{c}{r=5} & \multicolumn{1}{c}{r=7} & \multicolumn{1}{c}{r=10} & \multicolumn{1}{c}{certifiable radius} & \multicolumn{1}{c}{training} & \multicolumn{1}{c}{inference} & \multicolumn{1}{c}{(smoothing or subgraphs)} \\ \hline
\rowcolor[HTML]{EFEFEF} 
Standard GNNCert & 0.683 & 0.678 & 0.673 & 0.651 & 0.057 & 10 & 6.104 & 0.019 & $T_s=20$ \\
+AuditVotes (JacAug) & 0.697 & 0.692 & 0.690 & 0.667 & 0.060 & 10 & 10.823 & 0.067 & $T_s=20$ \\
+AuditVotes (FAEAug) & 0.735 & 0.731 & 0.731 & 0.719 & 0.074 & 10 & 13.611 & 0.043 & $T_s=20$ \\
+AuditVotes (SimAug) & 0.770 & 0.766 & \textbf{0.765} & \textbf{0.742} & 0.069 & 10 & 27.478 & 0.068 & $T_s=20$ \\ \hline
\rowcolor[HTML]{EFEFEF} 
Standard SparseSmooth & 0.781 & 0.713 & 0.651 & 0.596 & 0.549 & \textbf{30} & 2.802 & 21.118 & $p_-=0.8,N=10,000$ \\
+AuditVotes   (JacAug) & 0.772 & 0.736 & 0.710 & 0.692 & 0.676 & \textbf{30} & 3.655 & 58.209 & $p_-=0.8,N=10,000$ \\
+AuditVotes   (FAEAug) & 0.758 & 0.720 & 0.671 & 0.662 & 0.623 & \textbf{30} & 12.672 & 26.023 & $p_-=0.8,N=10,000$ \\
+AuditVotes   (SimAug) & 0.793 & 0.738 & 0.704 & 0.685 & 0.660 & \textbf{30} & 19.799 & 58.394 & $p_-=0.8,N=10,000$ \\
+AuditVotes   (SimAug+Conf) & \textbf{0.812} & \textbf{0.772} & 0.756 & \textbf{0.742} & \textbf{0.715} & \textbf{30} & 19.799 & 59.346 & $p_-=0.8,N=10,000$ \\ \hline
\end{tabular}
\label{tab:cross_rd}
\end{table*}

\subsection{Novelty of Theoretical Proofs}

For SparseSmooth+Filter (Theorem~\ref{theorem-condition}), which is designed for discrete data (graphs), we establish the theoretical guarantee by re-decomposing the conditional probability such that the conditional probability under attack can be solved by a similar procedure as in SparseSmooth. The proof of the original SparseSmooth~\cite{bojchevski2020efficient} does not involve the conditional probability because they do not employ the filter, and it has not been studied in previous research. We can be solved the certification problem by a similar procedure as in SparseSmooth, because we find that the likelihood ratio is not affected by the conditional filter, since the likelihood ratio of a sample $Z$ being randomly drawn from the clean graph $\mathcal{G}$ and the perturbed graph $\mathcal{G}'$ is only depend on the number of edges they are different, and it is independent of the property of $Z$.

For Gaussian+Filter (Theorem~\ref{theorem-condition_image}), which is designed for image data, we establish the theoretical guarantee using the Neyman-Pearson Lemma. Our proof is novel in how to redefine the half-space and redefine the classifier $s(\cdot)$ in the Neyman-Pearson Lemma. Then the term $P(h(Y)=0)$ on both sides can be offset. The original proof in Gaussian~\cite{cohen2019certified} is designed for a marginal probability classifier instead of a conditional classifier.

\section{Additional Experimental Results}
\label{appendix:more_results}

\subsection{Cross Scheme Comparison}
\label{sec:appenx_cross}
In this section, we compare AuditVotes with two existing certification schemes: SparseSmooth and GNNCert. Due to the trade-off between clean accuracy and certified accuracy, we select the parameters of GNNCert and SparseSmooth that yield close clean accuracy, and then we compare their certified accuracy and maximum certifiable radius. Tables~\ref {tab:cross_ra} and~\ref {tab:cross_rd} present the cross-comparison results regarding various edge-addition and edge-deletion attack budgets. 
We also present the maximum certifiable radii and training and inference times for the different models. For GNNCert, the certified accuracy remains the same regardless of edge deletion or addition, whereas for SparseSmooth, we can select different optimal parameters for deletion or addition, respectively. 
It is worth noting that GNNCert+AuditVotes and SparseSmooth+AuditVotes have their own advantages. The former one is significantly more efficient during inference, while the latter one generally supports significantly larger certifiable radii. Notably, it is consistently observed that the model with our AuditVotes applied always achieves the best clean and certified accuracy.

\begin{table*}[!ht]
\centering
\caption{Certified accuracy ($r_a$) comparison across GNNCert and SparseSmooth w/wo AuditVotes applied (Cora-ML).}
\setlength{\tabcolsep}{1.5pt}
\begin{tabular}{lrrrrrrrrr}
\hline
 & \multicolumn{1}{c}{clean acc} & \multicolumn{4}{c}{certified acc ($r_a$)} & \multicolumn{1}{c}{maximum} & \multicolumn{1}{c}{runtime (s)} & \multicolumn{1}{c}{runtime (s)} & \multicolumn{1}{c}{Parameters} \\ \cline{1-6} \cline{8-9}
Models & \multicolumn{1}{c}{r=0} & \multicolumn{1}{c}{r=3} & \multicolumn{1}{c}{r=5} & \multicolumn{1}{c}{r=7} & \multicolumn{1}{c}{r=10} & \multicolumn{1}{c}{certifiable radius} & \multicolumn{1}{c}{training} & \multicolumn{1}{c}{inference} & \multicolumn{1}{c}{(smoothing or subgraphs)} \\ \hline
\rowcolor[HTML]{EFEFEF} 
Standard GNNCert & 0.683 & 0.678 & 0.673 & 0.651 & 0.057 & 10 & 6.104 & 0.019 & $T_s=20$ \\
+AuditVotes (JacAug) & 0.697 & 0.692 & 0.690 & 0.667 & 0.060 & 10 & 10.823 & 0.067 & $T_s=20$ \\
+AuditVotes (FAEAug) & 0.735 & 0.731 & 0.731 & 0.719 & 0.074 & 10 & 13.611 & 0.043 & $T_s=20$ \\
+AuditVotes (SimAug) & \textbf{0.770} & \textbf{0.766} & \textbf{0.765} & \textbf{0.742} & 0.069 & 10 & 27.478 & 0.068 & $T_s=20$ \\ \hline
\rowcolor[HTML]{EFEFEF} 
Standard SparseSmooth & 0.140 & 0.140 & 0.140 & 0.140 & 0.140 & \textbf{48} & 13.390 & 572.745 & $p_+=0.2,p_-=0.6,N=10,000$ \\
+AuditVotes   (JacAug) & 0.738 & 0.715 & 0.690 & 0.683 & 0.681 & \textbf{48} & 16.756 & 611.468 & $p_+=0.2,p_-=0.6,N=10,000$ \\
+AuditVotes   (FAEAug) & 0.719 & 0.713 & 0.706 & 0.704 & 0.703 & \textbf{48} & 26.895 & 531.285 & $p_+=0.2,p_-=0.6,N=10,000$ \\
+AuditVotes   (SimAug) & 0.754 & 0.733 & 0.715 & 0.710 & 0.708 & \textbf{48} & 30.612 & 512.534 & $p_+=0.2,p_-=0.6,N=10,000$ \\
+AuditVotes   (SimAug+Conf) & 0.761 & 0.745 & 0.733 & 0.726 & \textbf{0.720} & \textbf{48} & 30.574 & 547.963 & $p_+=0.2,p_-=0.6,N=10,000$ \\ \hline
\end{tabular}
\label{tab:cross_ra}
\end{table*}

\subsection{Other Base Classifier}
We put other experimental results here due to the space limit. We evaluate APPNP~\cite{gasteiger2019appnp} as the base node classifier in Figure~\ref{fig:appnp}. 
We observe that the APPNP+SimAug significantly outperforms the vanilla APPNP smoothing model regarding both clean accuracy and certified accuracy. When certifying edge deletion ($r_d$), conditional smoothing with the confidence filter (Conf) can further improve the performance. 

\begin{figure}[htb]
\centering
\subfigure[Citeseer]{\includegraphics[width=0.235\textwidth,height=3.2cm]{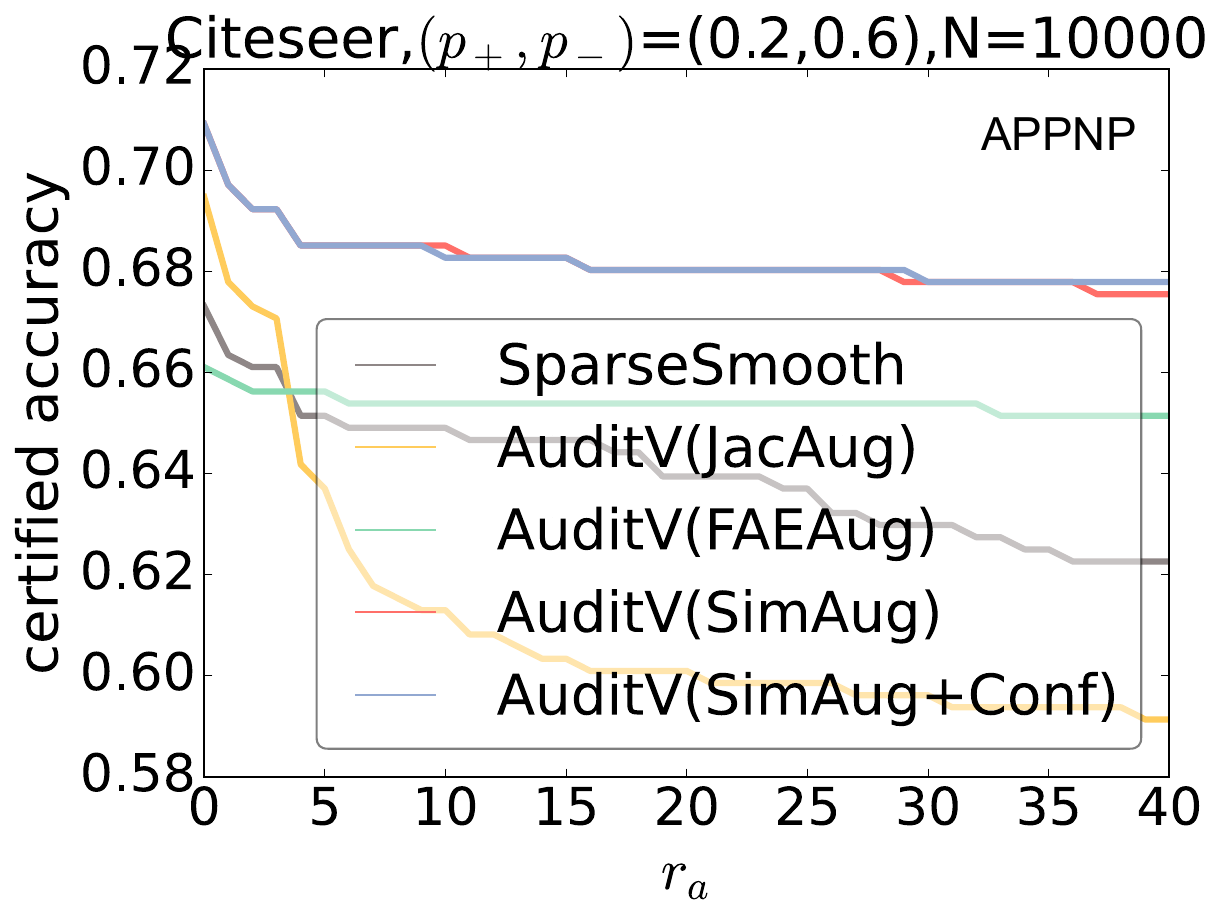}}
\subfigure[Citeseer]{\includegraphics[width=0.235\textwidth,height=3.2cm]{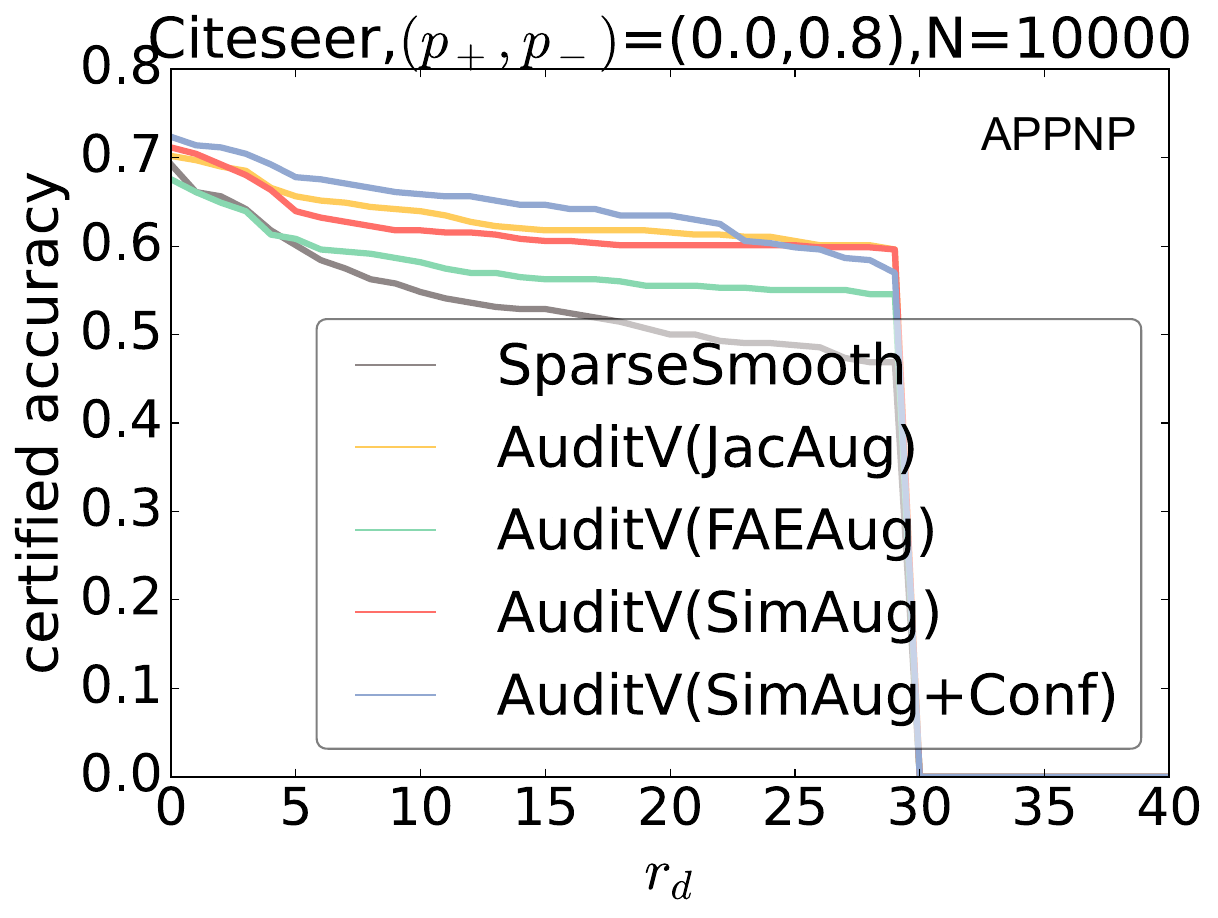}}
\vspace{-7pt}
\caption{Using APPNP as the base GNN model for randomized smoothing.}
\label{fig:appnp}
\end{figure}


\subsection{Other Filtering Functions}
\label{sec:other_filter} 
Previous studies~\cite{zhu2020beyond} have revealed that homophily is an important factor for node classification performance. Empirically, the graph neural networks perform better for nodes with higher homophily. However, the homophily requires the label information, which is not available during the testing. Instead, we use the pseudo label (predicted by the base classifier) to calculate the node homophily. Given a node $v$, the node homophily $homo(\cdot)$ is calculated as follows: 
\begin{equation}
\label{eqn:homo_defi}
    homo(\mathcal{G})_v= \frac{|\{u\in\mathcal{N}(v):\hat{y}_v=\hat{y}_u\}|}{|\mathcal{N}(v)|},
\end{equation}
where $\mathcal{V}$ is the set of nodes, $\mathcal{N}(v)$ is the neighbor of node $v$, and $\hat{y}_v$ is the predicted class of node $v$. If all the neighbors of the node $v$ have the same (different) predicted label, then the homophily is $homo=1$ ($0$). In particular, if a node is an isolated node, we set its homophily $homo=0$.


\begin{table}[!ht]
\centering
\caption{Evaluation of MLP filtering function on CiteSeer dataset.}
\begin{tabular}{cllll}
\hline
 & \multicolumn{4}{c}{certified accuracy   ($r_d$)} \\ \cline{2-5} 
\multirow{-2}{*}{Models} & \multicolumn{1}{c}{\textbf{0}} & \multicolumn{1}{c}{\textbf{5}} & \multicolumn{1}{c}{\textbf{7}} & \multicolumn{1}{c}{\textbf{10}} \\ \hline
\rowcolor[HTML]{EFEFEF} 
SparseSmooth & 0.700 & 0.596 & 0.570 & 0.550 \\
AuditV(JacAug) & 0.695 & 0.620 & 0.613 & 0.599 \\
AuditV(JacAug+Conf) & 0.663 & 0.630 & 0.627 & 0.623 \\
AuditV(JacAug+\textbf{MLP}) & \textbf{0.712} & 0.647 & 0.644 & 0.635 \\
AuditV(FAEAug) & 0.680 & 0.620 & 0.601 & 0.582 \\
AuditV(FAEAug+Conf) & 0.675 & 0.644 & 0.639 & \textbf{0.630} \\
AuditV(FAEAug+\textbf{MLP}) & 0.671 & \textbf{0.656} & \textbf{0.651} & 0.625 \\ \hline
\end{tabular}
\label{tab:MLP_filter}
\end{table}

JSDivergence~\cite{zhang2021detection} measures the discrepancy of a probability vector, aiming to detect victim nodes targeted by an attacker. The intuition is that the attacker prefers adding edges between nodes with different classes, and this might enlarge the probability discrepancy between neighbors. Specifically, for a node $v$, the $JSD(\cdot)$ is calculated as follows:
$$JSD(\mathcal{G})_v=H(\frac{1}{|\mathcal{N}(v)|}\sum_{u\in\mathcal{N}(v)}p_u)-\frac{1}{|\mathcal{N}(v)|}\sum_{u\in\mathcal{N}(v)} H(p_u),
$$
where $p_u$ is the softmax probability vector of node $u$ predicted by base classifier $f$, and $H(\cdot)$ denotes the Shannon entropy. 

To jointly utilize these metrics, we employ an MLP model with three combined features as its input: $h(\phi(\mathcal{G}))=MLP(Conf(\phi(\mathcal{G}))$, $Homo(\phi(\mathcal{G})),JSD(\phi(\mathcal{G})))$. Then, we can sample some noise graphs $\phi(\mathcal{G})$ and train the MLP model on the training nodes with labels of whether they are correctly classified by the base classifier $f$. Nevertheless, we occasionally observe slight improvement when using combined metrics (Table~\ref{tab:MLP_filter}). Compared to using confidence scores alone, it requires more computation and hyperparameter tuning during training of the MLP model. As a result, we recommend the confidence filtering (Conf), which is efficient and effective.


\begin{table*}[!ht]
\centering
\caption{Empirical robust accuracy among regular robust GNNs and smoothed GNNs. Attacker: Nettack~\cite{zugner2018adversarial} (evasion setting), with 30 targeted nodes. The attack power is the edge number the attacker can manipulate for each target node.}
\vspace{-8pt}
\setlength{\tabcolsep}{1.5pt}
\begin{tabular}{c|c|cccc|cccc}
\toprule[0.9pt]
\multirow{2}{*}{Datasets} & Defense models & \multicolumn{4}{c|}{Regular} & \multicolumn{4}{c}{Smoothed ($p_+=0.2$,   $p_-=0.3$, $N=1000$)} \\ \cline{2-10} 
 & Attack power & GCN & GAT & MedianGCN & AirGNN & SparseSmooth & \textbf{AuditV(JacAug)} & \textbf{AuditV(FAEAug)} & \textbf{AuditV(SimAug)} \\ \toprule[0.9pt]
\multirow{6}{*}{Citeseer} & 0 & 0.789 & 0.790 & 0.722 & {\ul 0.800} & 0.167 & {\ul 0.800} & 0.733 & \textbf{0.833} \\
 & 1 & 0.527 & 0.453 & 0.647 & 0.673 & 0.167 & {\ul 0.800} & 0.700 & \textbf{0.833} \\
 & 2 & 0.240 & 0.273 & 0.340 & 0.527 & 0.167 & {\ul 0.767} & 0.700 & \textbf{0.833} \\
 & 3 & 0.147 & 0.207 & 0.220 & 0.327 & 0.167 & {\ul 0.767} & 0.700 & \textbf{0.833} \\
 & 4 & 0.153 & 0.153 & 0.107 & 0.240 & 0.167 & {\ul 0.767} & 0.700 & \textbf{0.833} \\
 & 5 & 0.127 & 0.133 & 0.100 & 0.180 & 0.167 & {\ul 0.767} & 0.700 & \textbf{0.833} \\ \hline
\multirow{6}{*}{PubMed} & 0 & 0.793 & 0.800 & 0.767 & 0.760 & OOM & {\ul 0.833} & {\ul 0.833} & \textbf{0.867} \\
 & 1 & 0.447 & 0.533 & 0.587 & 0.720 & OOM & {\ul 0.833} & {\ul 0.833} & \textbf{0.867} \\
 & 2 & 0.313 & 0.380 & 0.267 & 0.653 & OOM & {\ul 0.833} & {\ul 0.833} & \textbf{0.867} \\
 & 3 & 0.200 & 0.200 & 0.200 & 0.560 & OOM & {\ul 0.833} & {\ul 0.833} & \textbf{0.867} \\
 & 4 & 0.167 & 0.153 & 0.167 & 0.447 & OOM & {\ul 0.833} & {\ul 0.833} & \textbf{0.867} \\
 & 5 & 0.167 & 0.147 & 0.140 & 0.420 & OOM & {\ul 0.833} & {\ul 0.833} & \textbf{0.867} \\ \bottomrule[0.9pt]
\end{tabular}
\label{tab:empirical_nettack}
\end{table*}

\begin{table*}[!ht]
\centering
\caption{Empirical robust accuracy comparison among regular robust GNNs and smoothed GNNs. Attacker: IG-attack~\cite{wu2019adversarial} (evasion setting), with 30 targeted nodes. (The attack encounters OOM on the PubMed dataset.)}
\vspace{-7pt}
\setlength{\tabcolsep}{1.5pt}
\begin{tabular}{c|c|cccc|cccc}
\toprule[0.9pt]
\multirow{2}{*}{Datasets} & Defense models & \multicolumn{4}{c|}{Regular} & \multicolumn{4}{c}{Smoothed ($p_+=0.2$, $p_-=0.3$, $N=1000$)} \\ \cline{2-10} 
 & Attack power & GCN & GAT & MedianGCN & AirGNN & SparseSmooth & \textbf{AuditV(JacAug)} & \textbf{AuditV(FAEAug)} & \textbf{AuditV(SimAug)} \\ \toprule[0.9pt]
\multirow{6}{*}{Cora-ML} & 0 & {\ul 0.989} & 0.933 & \textbf{1.000} & 0.978 & 0.667 & 0.933 & 0.967 & 0.967 \\
 & 1 & 0.922 & 0.856 & 0.900 & {\ul 0.944} & 0.667 & 0.933 & \textbf{0.967} & {\ul 0.933} \\
 & 2 & 0.822 & 0.767 & 0.800 & 0.856 & 0.667 & {\ul 0.933} & \textbf{0.967} & {\ul 0.933} \\
 & 3 & 0.756 & 0.700 & 0.589 & 0.744 & 0.667 & {\ul 0.933} & \textbf{0.967} & {\ul 0.933} \\
 & 4 & 0.567 & 0.667 & 0.344 & 0.556 & 0.667 & {\ul 0.933} & \textbf{0.967} & {\ul 0.933} \\
 & 5 & 0.400 & 0.444 & 0.244 & 0.467 & 0.667 & {\ul 0.933} & \textbf{0.967} & {\ul 0.933} \\ \hline
\multirow{6}{*}{Citeseer} & 0 & 0.789 & 0.789 & 0.722 & {\ul 0.800} & 0.167 & {\ul 0.800} & 0.733 & \textbf{0.833} \\
 & 1 & 0.767 & 0.622 & 0.700 & 0.767 & 0.167 & {\ul 0.800} & 0.733 & \textbf{0.833} \\
 & 2 & 0.489 & 0.456 & 0.378 & 0.522 & 0.167 & {\ul 0.800} & 0.733 & \textbf{0.833} \\
 & 3 & 0.333 & 0.322 & 0.300 & 0.378 & 0.167 & {\ul 0.800} & 0.733 & \textbf{0.833} \\
 & 4 & 0.333 & 0.289 & 0.189 & 0.333 & 0.167 & {\ul 0.800} & 0.733 & \textbf{0.833} \\
 & 5 & 0.311 & 0.278 & 0.178 & 0.300 & 0.167 & {\ul 0.800} & 0.733 & \textbf{0.833} \\ \bottomrule[0.9pt]
\end{tabular}
\label{tab:empirical_ig}
\end{table*}

\subsection{Empirical robustness evaluation }
We visualize the number of edges changed by Nettack and IG-attack in Figure~\ref{fig:edge_change}. We provide the results of clean accuracy and certified accuracy trade-off for certifying edge-deletion ($r_d$) in Figure~\ref{fig:Trade-off_rd}. Empirical robustness comparison against Nettack and IG-attack is provided in Table~\ref{tab:empirical_nettack} and Table~\ref{tab:empirical_ig}. Figure~\ref{fig:e_sparsity} visualizes the edge sparsity of the original graph, the smoothed graph before and after augmentation. 

\subsection{Image classification evaluation}
Certified accuracy evaluations on CIFAR-10 and MNIST are provided in Figure~\ref{fig:cer_image}. Results discussion is provided in the main paper.

\begin{figure}[htb]
\centering
\subfigure[CIFAR-10]{\includegraphics[width=0.235\textwidth,height=2.95cm]{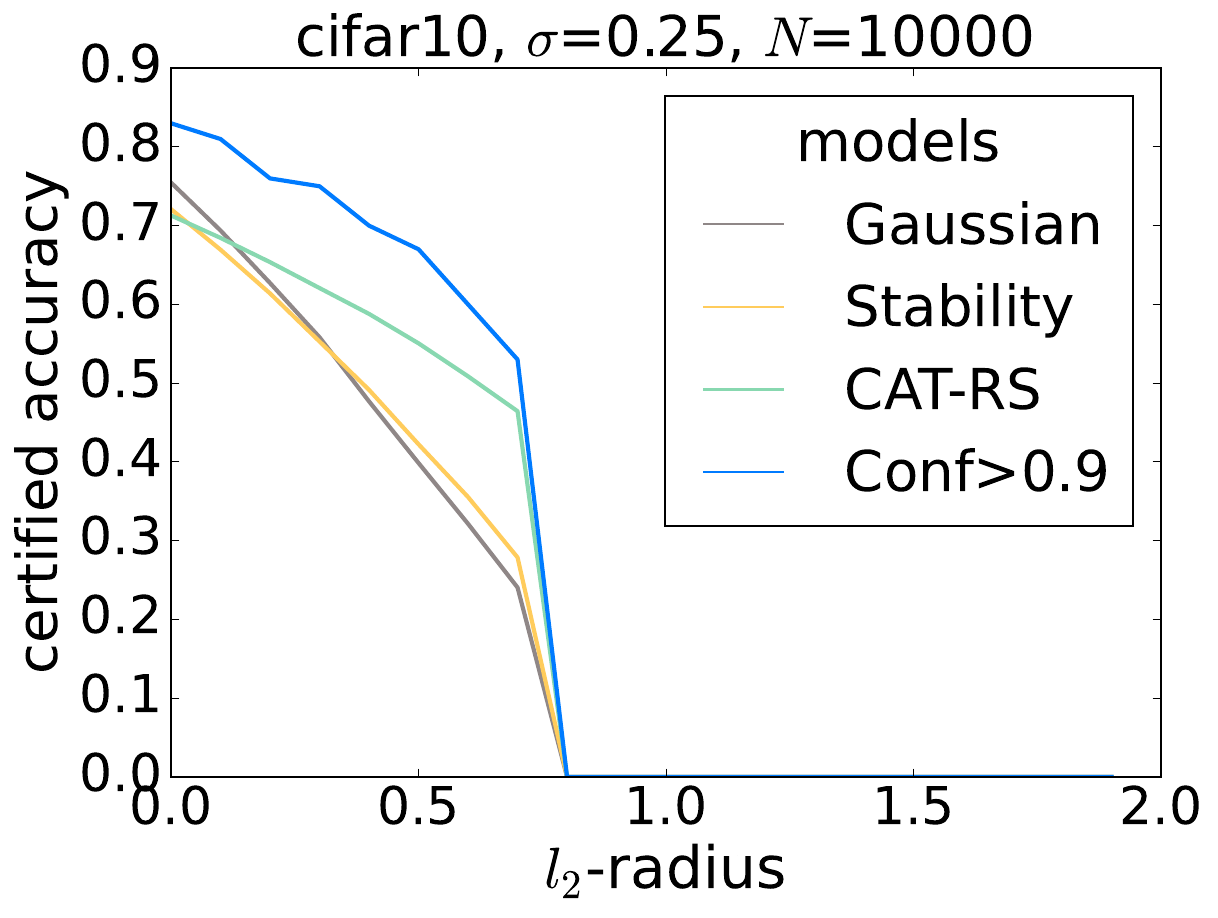}}
\subfigure[MNIST]{\includegraphics[width=0.235\textwidth,height=2.95cm]{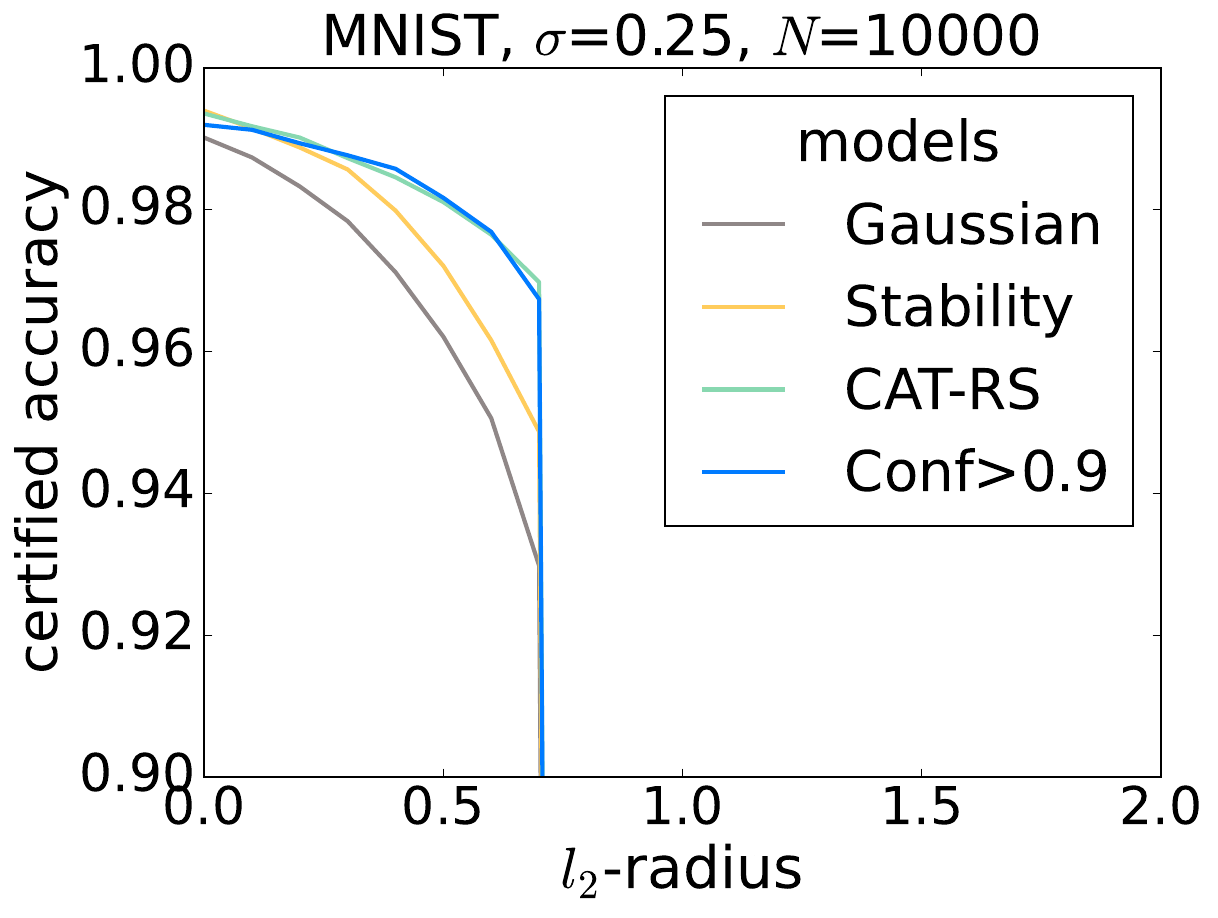}}
\vspace{-8pt}
\caption{Applying AuditVotes (Conf) to Gaussian~\cite{cohen2019certified} model for the image classification task.}
\label{fig:cer_image}
\end{figure}

\begin{figure}[!htb]
\centering
\subfigure[MNIST]{\includegraphics[width=0.235\textwidth,height=3.2cm]{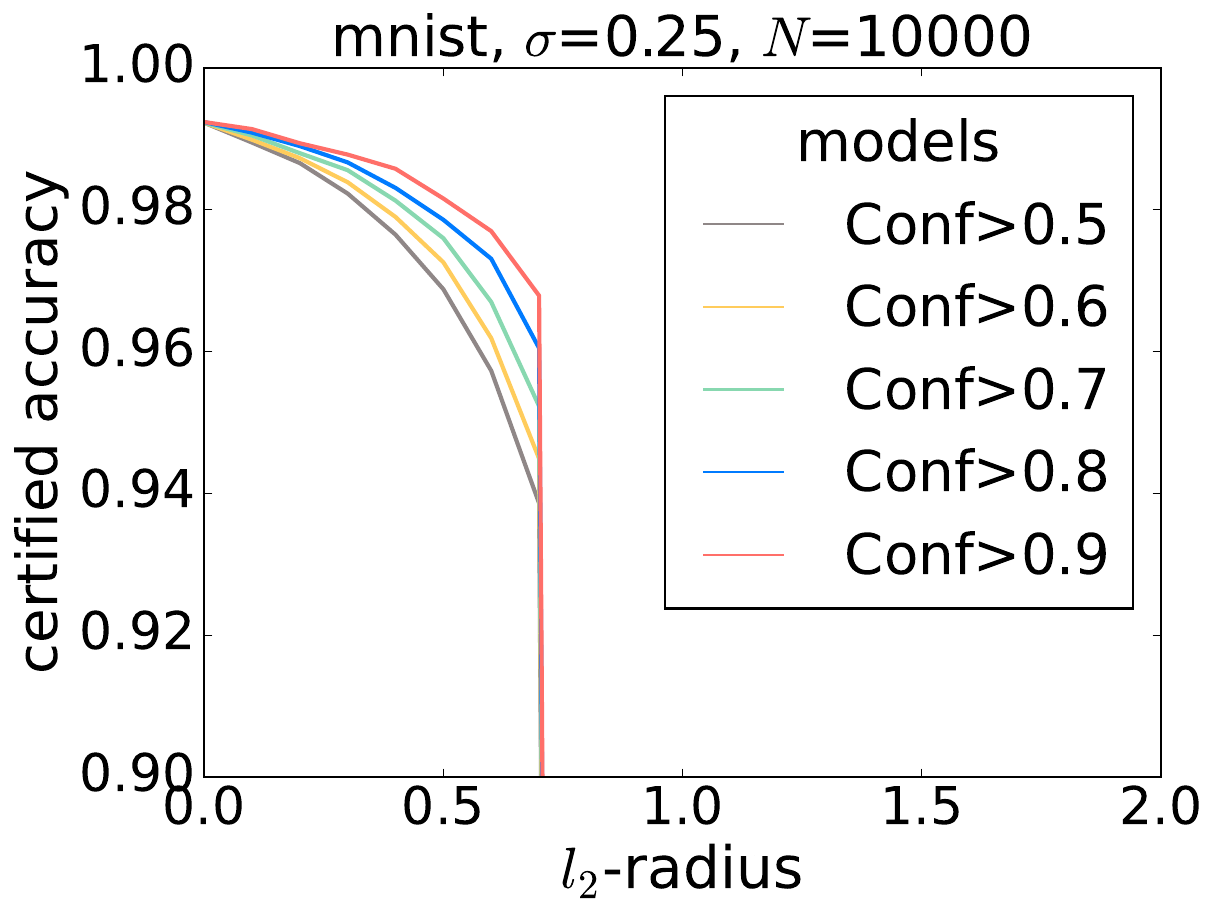}}
\subfigure[CIFAR-10]{\includegraphics[width=0.235\textwidth,height=3.2cm]{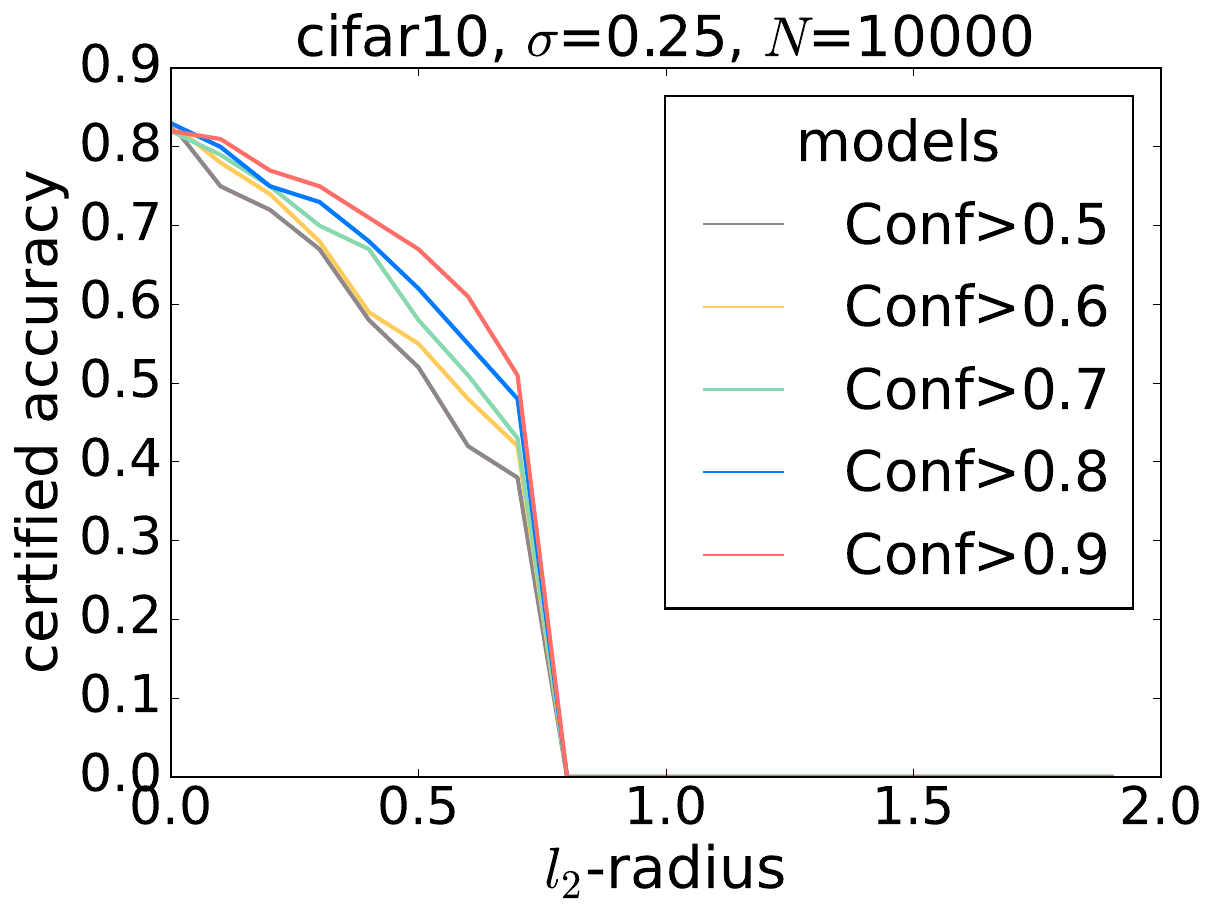}}
\caption{Certified accuracy of conditional smoothing with various confidence thresholds on image datasets.}
\label{fig:conf_thre_image}
\end{figure}

\begin{figure}[!ht]
\centering
\subfigure[Cora-ML]{\includegraphics[width=0.185\textwidth,height=2.65cm]{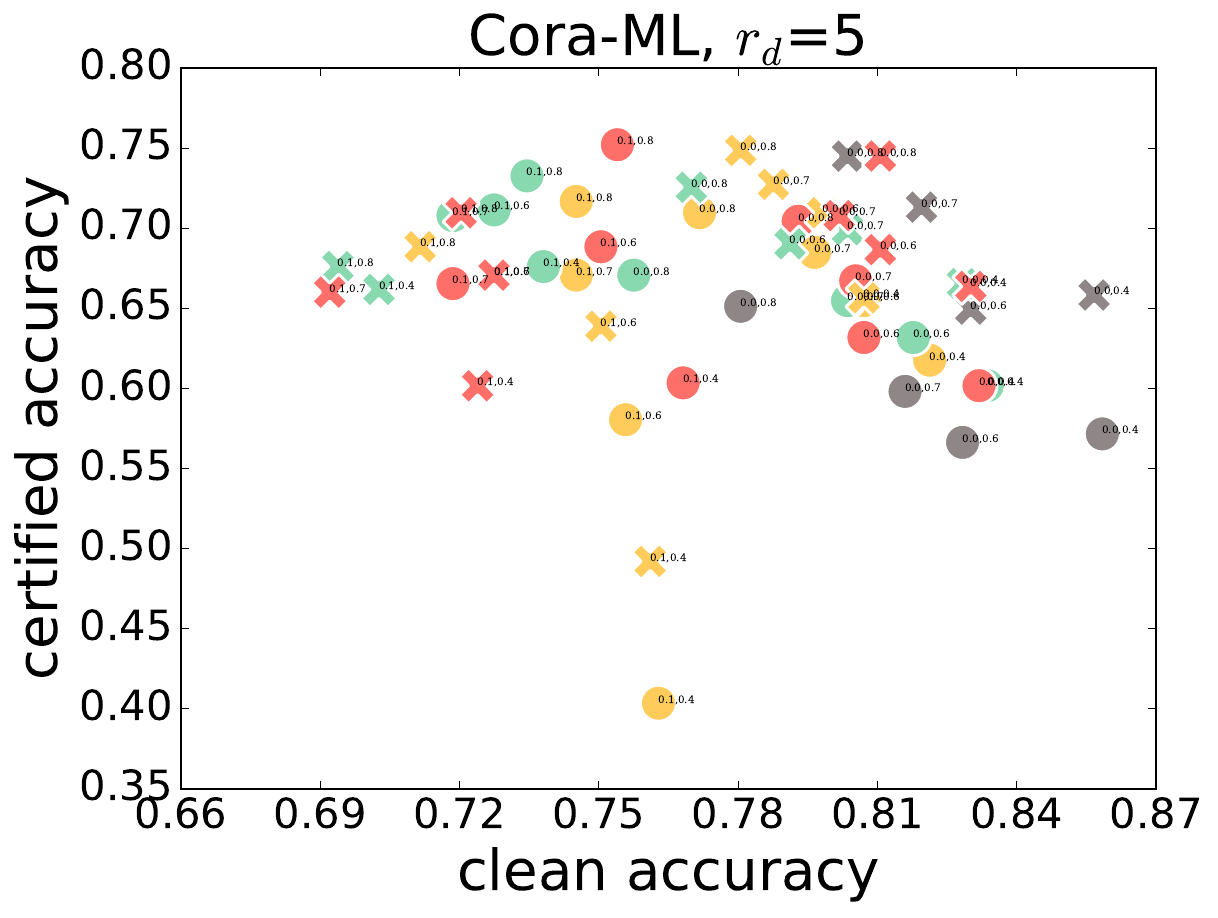}}
\subfigure[Citeseer]{\includegraphics[width=0.285\textwidth,height=2.65cm]{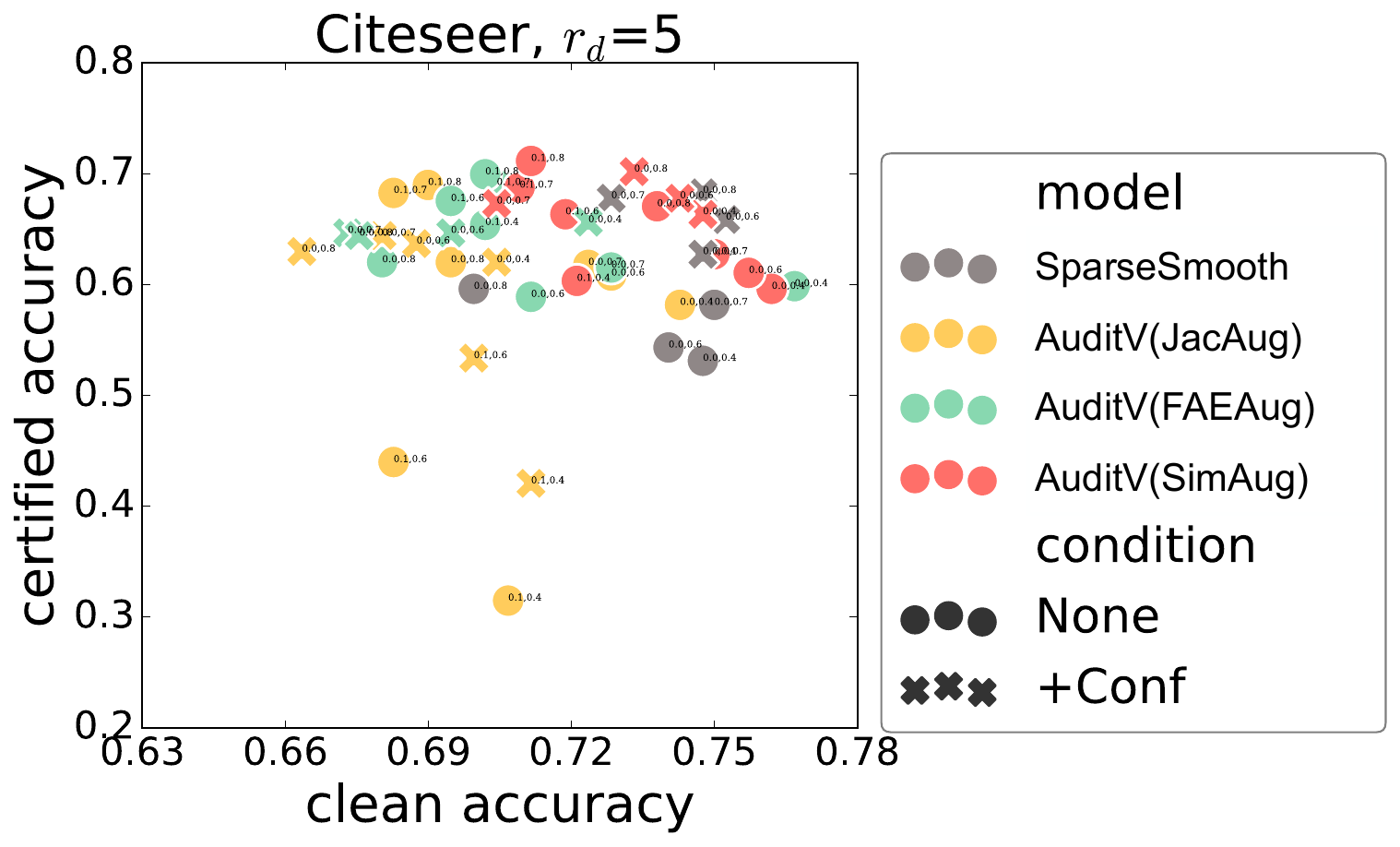}}
\subfigure[PubMed (smaller $r_d$)]{\includegraphics[width=0.185\textwidth,height=2.65cm]{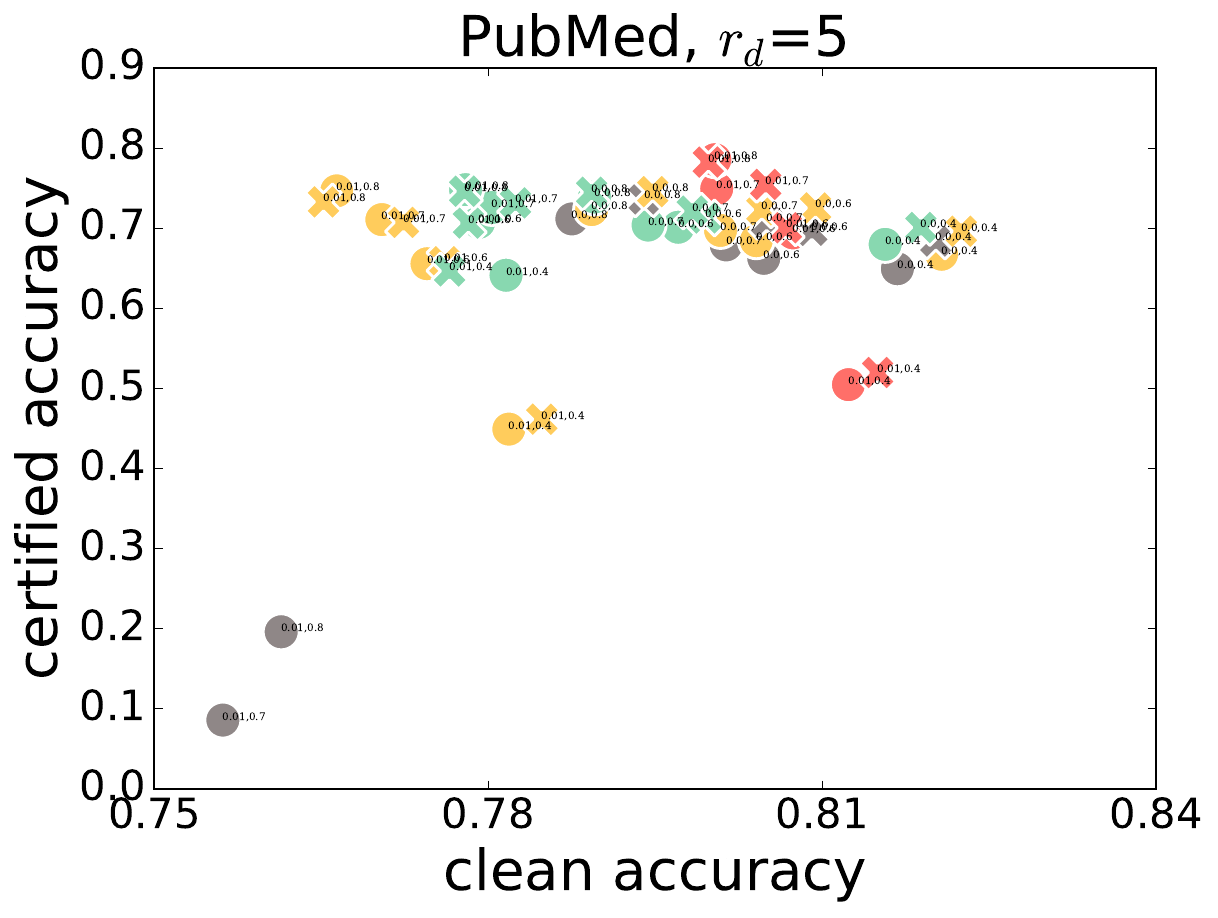}}
\subfigure[PubMed (larger $r_d$)]{\includegraphics[width=0.285\textwidth,height=2.65cm]{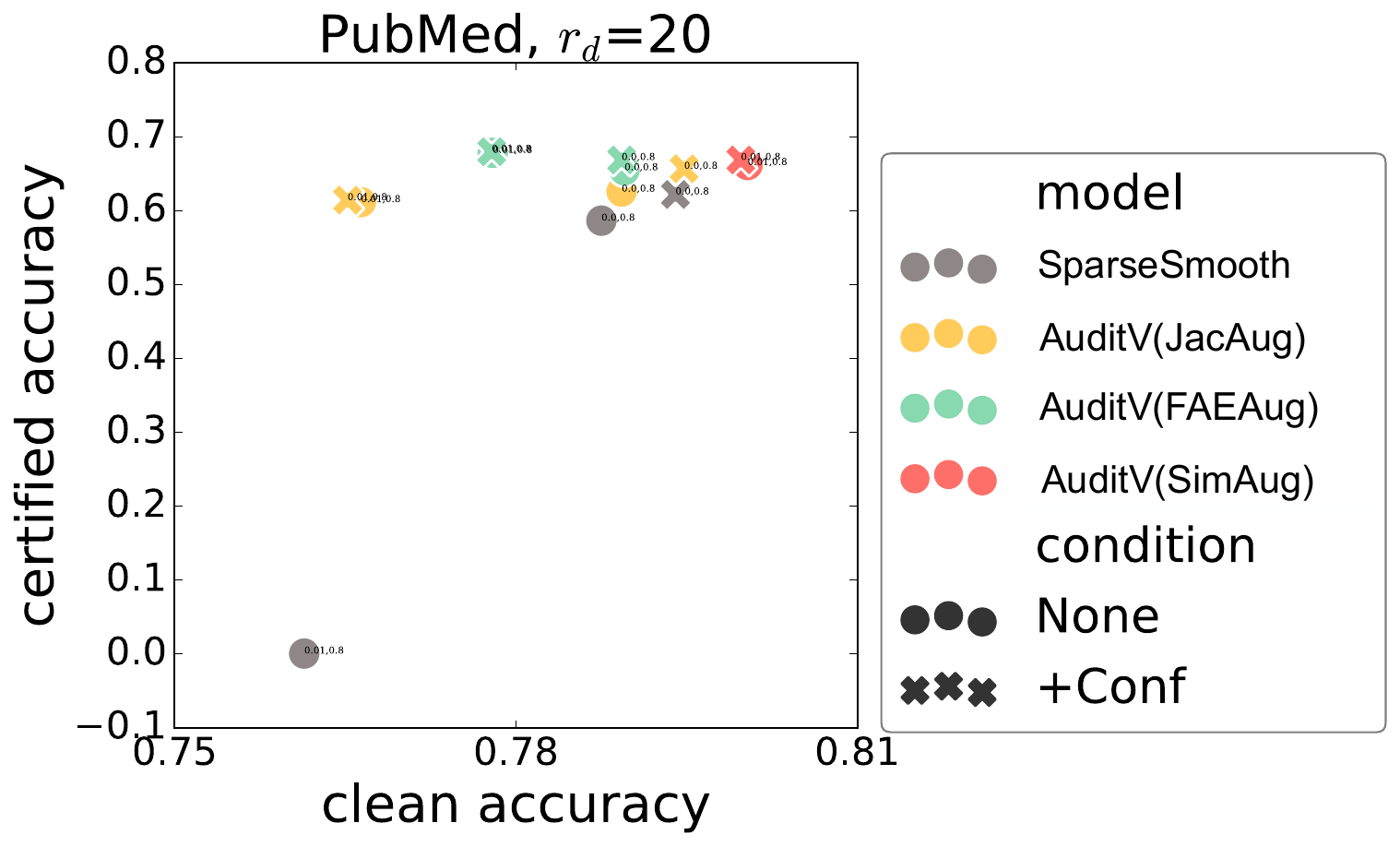}}
\vspace{-8pt}
\caption{Clean accuracy and certified accuracy trade-off.}
\label{fig:Trade-off_rd}
\end{figure}

\begin{figure}[!ht]
    \centering
    \includegraphics[width=0.40\textwidth,height=2.6cm]{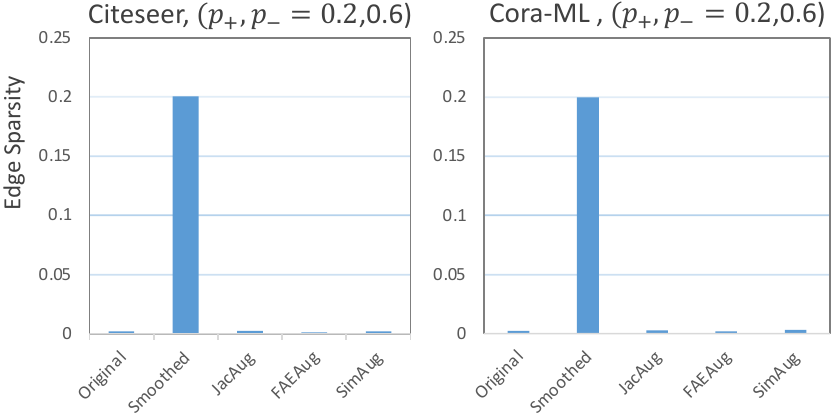}
    \vspace{-5pt}
    \caption{Edge sparsity of original graph, randomized graph with and without augmentations.}
    \label{fig:e_sparsity}
    \vspace{-8pt}
\end{figure}

\begin{figure}[htb]
\centering
\includegraphics[width=0.43\textwidth,height=2.5cm]{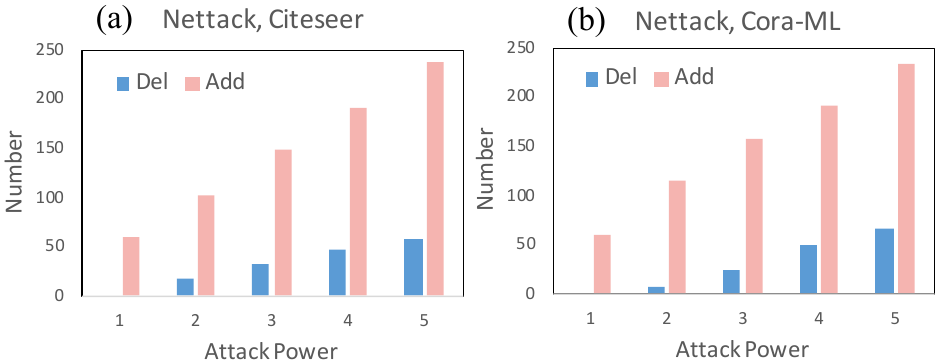}
\includegraphics[width=0.43\textwidth,height=2.5cm]{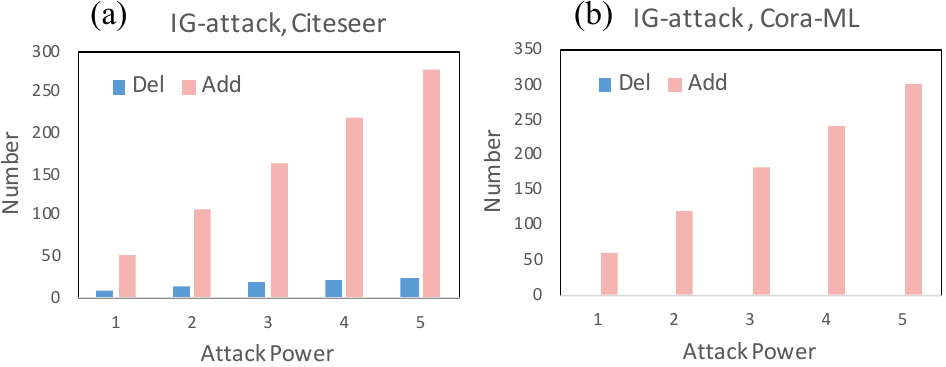}
\vspace{-5pt}
\caption{Attackers prefer adding edges rather than deleting edges.}
\label{fig:edge_change}
\end{figure}

\subsection{Hyper-parameter Analysis}
\label{sec:Hyperparameter}
In this section, we conduct a hyper-parameter analysis regarding the threshold for the confidence filter. In Figure~\ref{fig:conf_thre_graph}, we present the edge deletion certified accuracy of conditional smoothing with various confidence thresholds on graph datasets. In Figure~\ref{fig:conf_thre_image}, we present the $l_2$-norm certified accuracy of conditional smoothing with various confidence thresholds on the CIFAR-10 dataset. We note that in different datasets, the optimal confidence thresholds might vary. For convenience, we simply set the same confidence threshold for all datasets. In practice, this confidence threshold can be adjusted according to the performance of the validation nodes. 

\begin{figure}[!htb]
\centering
\subfigure[Citeseer]{\includegraphics[width=0.235\textwidth,height=3.2cm]{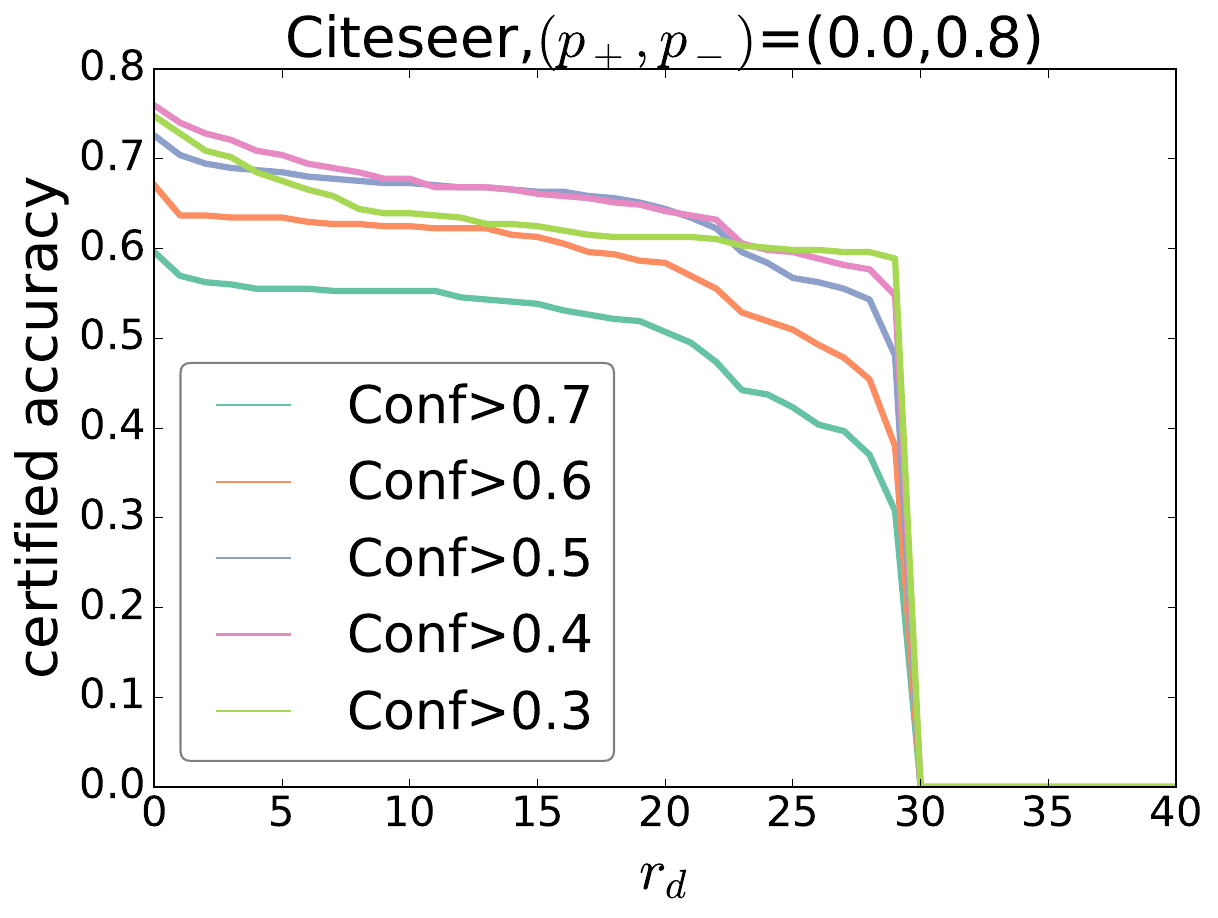}}
\subfigure[Cora-ML]{\includegraphics[width=0.235\textwidth,height=3.2cm]{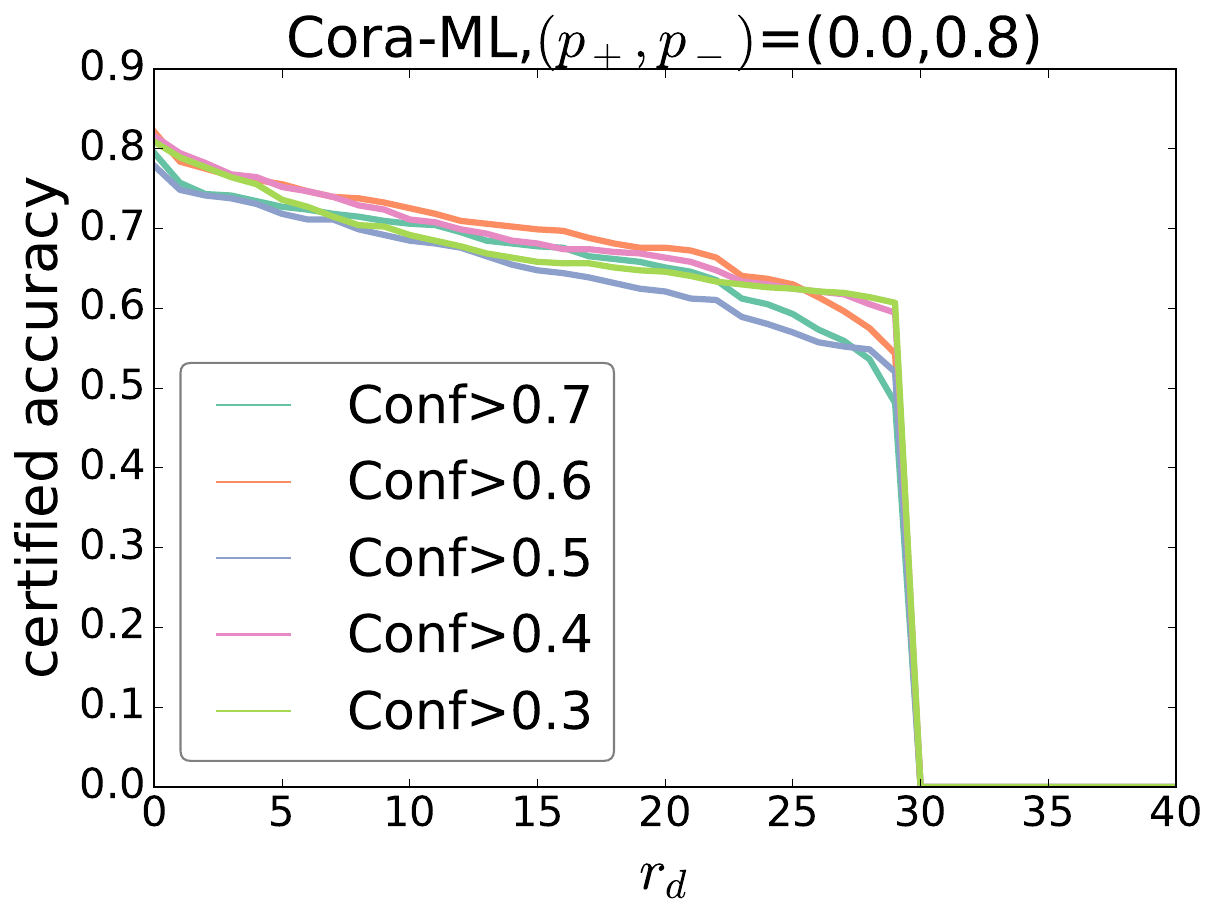}}
\caption{Certified accuracy of conditional smoothing with various confidence thresholds on graph datasets.}
\label{fig:conf_thre_graph}
\end{figure}

\subsection{Potential Real-world Security-critical Applications}
In this section, we present an example of potential real-world security-critical applications. Real-world financial networks are naturally modeled as graphs in which each node represents a user account, and each edge represents a monetary transaction. GNNs are widely deployed as anomaly detectors to flag suspicious or malicious accounts~\cite{cheng2025graph}. However, motivated adversaries can attempt to evade detectors by manipulating the transaction graph~\cite{wu2024safeguarding}, e.g., by adding carefully crafted transaction edges. A GNN with a provable defense certifies that the model’s prediction for a node cannot be flipped even if an attacker modifies up to \(r_a\) incident transaction edges. We report certified accuracy, i.e., the fraction of nodes that are both correctly classified and provably robust to \(r_a\) edge manipulations.

To illustrate security-critical usage, we evaluate two Anti–Money Laundering (AML) datasets commonly used in financial crime research: (i) T-Finance~\cite{tang2022rethinking}, a real-world transaction network collected for anomaly detection on graphs; and (ii) AMLSim, a synthetic, multi-agent simulator that generates banking transactions together with known money-laundering patterns~\cite{suzumura2021anti}. In both datasets, nodes are accounts and edges are transactions; labels indicate whether an account is illicit or benign. 

Across both datasets, integrating our AuditVotes into a certified GNN substantially improves both clean and certified accuracy. On T-Finance, for example, certified accuracy at \(r_a\in\{0,3,5,7,10,20,30\}\) increases from as low as 0.155 under SparseSmooth to 0.868 with our AuditVotes-augmented model, representing a large gain in certifiable robustness without sacrificing accuracy. On AMLSim, which stresses detectors with diverse laundering motifs, AuditVotes likewise raises certified accuracy across budgets while preserving clean performance. Tables~\ref{tab:cert_tfinance} and \ref{tab:cert_amlsim} provide the full results.

\begin{table}[!ht]
\centering
\caption{Certified Accuracy on Anti–Money Laundering Datasets (T-Finance \cite{tang2022rethinking}) with $p_+=0.2$ and $p_-=0.6$. }
\label{tab:cert_tfinance}
\setlength{\tabcolsep}{1.5pt}
\begin{tabular}{lrrrrrrr}
\hline
 & \multicolumn{7}{c}{Certified Accuracy ($r_a$)} \\ \hline
Models & 0 (Clean) & 3 & 5 & 7 & 10 & 20 & 30 \\ \hline
SparseSmooth & 0.368 & 0.239 & 0.216 & 0.207 & 0.201 & 0.169 & 0.155 \\
+AuditV(F) & \textbf{0.868} & \textbf{0.868} & \textbf{0.868} & \textbf{0.868} & \textbf{0.868} & \textbf{0.868} & \textbf{0.868} \\ \hline
\end{tabular}
\end{table}

\begin{table}[!ht]
\centering
\caption{Certified Accuracy on Anti–Money Laundering Datasets (AMLSim \cite{suzumura2021anti}) with $p_+=0.2$ and $p_-=0.6$.}
\label{tab:cert_amlsim}
\setlength{\tabcolsep}{1.5pt}
\begin{tabular}{lrrrrrrr}
\hline
 & \multicolumn{7}{c}{Certified Accuracy ($r_a$)} \\ \hline
Models & 0 (Clean) & 3 & 5 & 7 & 10 & 20 & 30 \\ \hline
SparseSmooth & 0.509 & 0.509 & 0.509 & 0.509 & 0.509 & 0.509 & 0.509 \\
+AuditV(F) & \textbf{0.755} & \textbf{0.746} & \textbf{0.746} & \textbf{0.746} & \textbf{0.746} & \textbf{0.744} & \textbf{0.742} \\ \hline
\end{tabular}
\end{table}

Overall, these results indicate that provably robust GNNs can provide actionable, certification-backed decisions in financial fraud detection—precisely the type of security-critical environment where resilience to adaptive, edge-manipulating adversaries is essential.

\subsection{Evaluation On Heterophilic Graph}
While our main paper focuses on homophily graphs, our augmentation function FAEAug does not rely on any homophily assumption. FAEAug learns to reconstruct edge patterns from node features—capturing any feature-edge correlation present in the data, whether homophilous or heterophilous. In Table~\ref{tab:cert_actor}, we evaluate FAEAug on a heterophilic graph dataset (Actor~\cite{tang2009social}) and a heterophilic GNN (H2GCN~\cite{zhu2020beyond}) to show its compatibility. The results show that the certified accuracy consistently increases, reaching 27.56\% when $r_d=20$.

\begin{table}[ht!]
\centering
\caption{Certified Accuracy on heterophilic graph dataset (Actor~\cite{tang2009social}) and a heterophilic GNN (H2GCN~\cite{zhu2020beyond}) with $p_+=0.0$ and $p_-=0.8$.}
\label{tab:cert_actor}
\setlength{\tabcolsep}{1.5pt}
\begin{tabular}{lrrrrrrr}
\hline
 & \multicolumn{7}{c}{Certified Accuracy ($r_d$)} \\ \hline
Models & 0 (Clean) & 3 & 5 & 7 & 10 & 20 & 30 \\ \hline
SparseSmooth & 0.296 & 0.259 & 0.240 & 0.229 & 0.216 & 0.185 & 0.000 \\
+AuditV(F) & \textbf{0.301} & \textbf{0.286} & \textbf{0.269} & \textbf{0.262} & \textbf{0.252} & \textbf{0.236} & 0.000 \\ \hline
\end{tabular}
\end{table}

\end{document}